\documentclass[final,3p,times]{elsarticle}
\journal{arXiv}

\usepackage[utf8]{inputenc} 
\usepackage[T1]{fontenc}
\usepackage{amsmath,amssymb,amsthm}
\usepackage{graphicx}
\usepackage{booktabs}
\usepackage[table]{xcolor}
\usepackage{comment}
\usepackage{float}      
\usepackage{needspace}  
\usepackage{dsfont}
\usepackage{subcaption} 
\usepackage{listings}
\usepackage{makecell} 

\graphicspath{{figures/}}

\usepackage{algorithm}
\usepackage{algpseudocode}
\makeatletter
\algrenewcommand\alglinenumber[1]{\hbox to 2.2em{\footnotesize #1:\hss}}
\algrenewcommand\algorithmiccomment[1]{\hfill$\triangleright$~#1}
\makeatother 

\usepackage[hidelinks]{hyperref}
\newcommand{\doi}[1]{\href{https://doi.org/#1}{\texttt{#1}}}
\newcommand{\repoKbeta}{\href{https://github.com/sck-at-ucy/kbeta}{\texttt{kbeta}}}
\newcommand{\repoPINN}{\href{https://github.com/sck-at-ucy/kbeta-pinn3d}{\texttt{kbeta-pinn3d}}}
\newcommand{\repoTrans}{\href{https://github.com/sck-at-ucy/kbeta-transformer2d}{\texttt{kbeta-transformer2d}}}

\newcommand{\doiKbeta}{\doi{10.5281/zenodo.16902740}}
\newcommand{\doiPINN}{\doi{10.5281/zenodo.16915163}}     
\newcommand{\doiTrans}{\doi{10.5281/zenodo.16911538}}

\DeclareMathOperator{\expmone}{expm1}

\newtheorem{theorem}{Theorem}[section]
\newtheorem{lemma}[theorem]{Lemma}

\theoremstyle{definition}
\newtheorem{assumption}[theorem]{Assumption}

\theoremstyle{remark}

\makeatletter
\algdef{SE}[INDENT]{Indent}{EndIndent}{}{} 
\algtext*{Indent}\algtext*{EndIndent}      
\makeatother

\makeatletter
\AtBeginEnvironment{appendices}{%
  \renewcommand\section{\@startsection{section}{1}{\z@}%
    {2.3ex plus 1ex minus .2ex}{1.2ex plus .2ex}{\Large\bfseries}}%
  \renewcommand\subsection{\@startsection{subsection}{2}{\z@}%
    {2.3ex plus 1ex minus .2ex}{1.0ex plus .2ex}{\large\bfseries}}%
}
\makeatother

\begin{document}
\begin{frontmatter}

\title{\textbf{Kourkoutas-\(\beta\): A Sunspike-Driven Adam Optimizer with Desert Flair}}

\author[inst1]{Stavros Kassinos\corref{cor1}}
\cortext[cor1]{Corresponding author, kassinos@ucy.ac.cy}
\ead{kassinos@ucy.ac.cy}

\affiliation[inst1]{organization={Computational Sciences Laboratory, Department of Mechanical Engineering, University of Cyprus},
            addressline={1 University Avenue}, 
            city={Aglantzia},
            postcode={2109}, 
            state={Nicosia},
            country={Cyprus}
            }

\begin{abstract}
Transformer neural networks are increasingly used for physics-based problems.
In data-driven training of PDE surrogates, training samples are often generated
by solving the governing equations under heterogeneous boundary and initial
conditions. Even without stochasticity, these sample-to-sample shifts can induce
erratic changes in the loss landscape and \emph{spiky} (bursty) gradients; in Physics-Informed
Neural Networks (PINNs), stiff composite losses can amplify the effect.

We introduce \textbf{Kourkoutas-$\beta$}, an Adam-style optimizer that replaces the
fixed second-moment discount $\beta_2$ with a \emph{layer-wise} dynamic value driven by a
bounded “sunspike’’ ratio: the current pooled gradient norm divided by an EMA (with
coefficient $\alpha$) of past norms, squashed to $[0,1)$. Large spikes lower $\beta_2$ toward
$\beta_{2,\min}$ to react quickly; calm phases keep it near $\beta_{2,\max}$ to smooth updates.
Kourkoutas-$\beta$ optionally includes leaky-AMSGrad (\texttt{decay}\,$\in(0,1]$), a trust-region
clip (\texttt{max\_ratio}), and an adaptive tiny term in the denominator; bias-correction modes
are \texttt{"none"}, \texttt{"beta2max"}, and \texttt{"exact"}. When dynamic $\beta_2$ and all options
are disabled and \texttt{bias\_correction}=\texttt{"none"}, the method is \emph{exactly} Adam.

We evaluate on four testbeds that stress second‑moment adaptivity: (i) a data‑driven
Transformer PDE surrogate (Heat2D), (ii) a 3D cylindrical PINN for heat conduction (Heat3D),
(iii) a lightweight MLX synthetic task with length jitter and a rare trigger that creates deterministic gradient bursts,
and (iv) a character‑level Transformer on a 30\,MB slice of \texttt{enwik8} (\texttt{small‑enwik8}).
Across these, Kourkoutas‑$\beta$ improves stability and final loss versus Adam with fixed~$\beta_2$.
On \texttt{small‑enwik8} it lowers final bits‑per‑character by $\sim38\%$ vs.\ Adam ($\beta_2{=}0.95$) 
and $\sim58\%$ vs.\ Adam ($\beta_2{=}0.999$) over 10 matched seeds, with dramatically smaller across‑seed variance.
Kourkoutas‑$\beta$ retains drop‑in simplicity and parity‑level overhead (on par with Adam when diagnostics are off in
testbeds A-C, within  single‑digit percent in testbed D).
We sketch why, with $\beta_2 \in [\beta_{2,\min},
\beta_{2,\max}]\subset(0,1)$, the method preserves Adam-style convergence properties (e.g., sublinear
regret) yet offers practical robustness under spiky gradients.
\end{abstract}

\begin{graphicalabstract}
\centering
\includegraphics[width=0.35\textwidth]{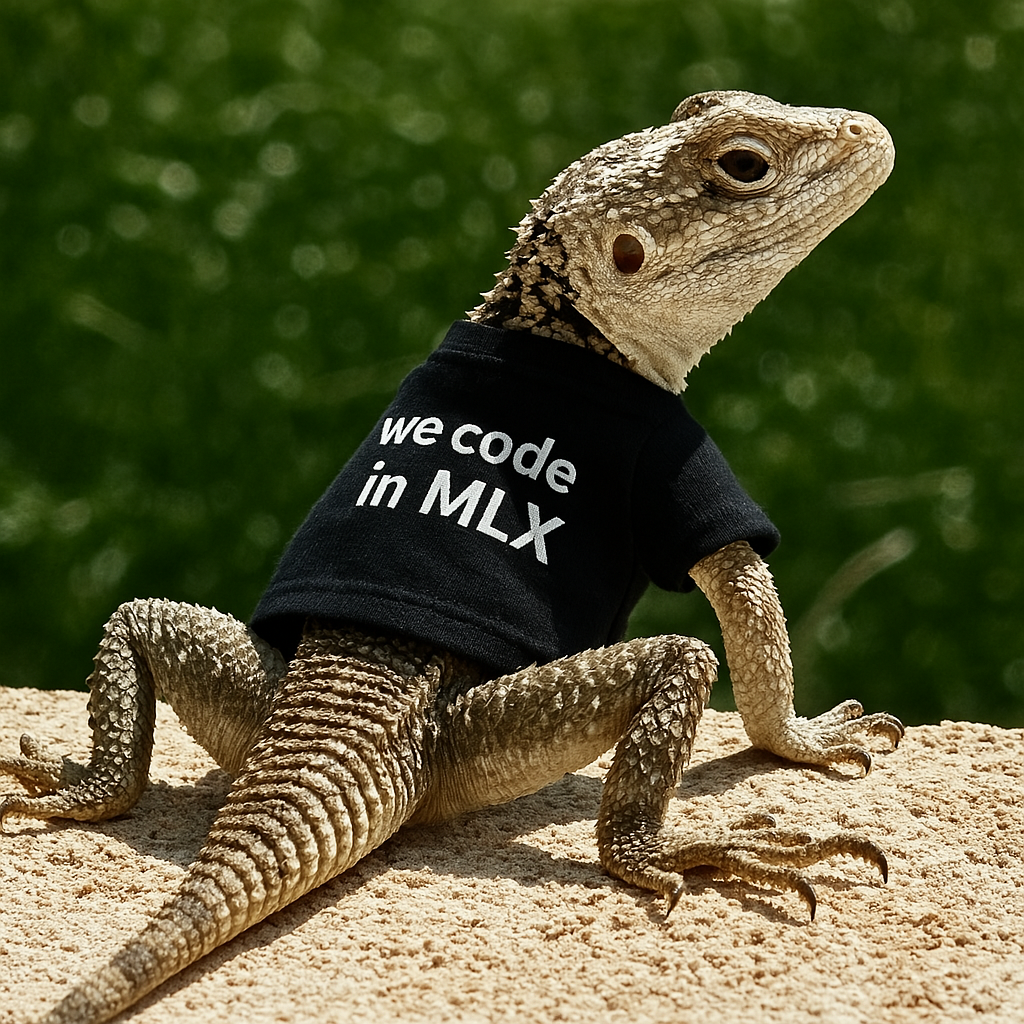}

\medskip
\textit{``Kourkoutas,'' the whimsical desert lizard of Cyprus, is symbolic of how optimizer Kourkoutas–$\beta$ explores the parameter space. Kourkoutas–$\beta$ is a variant of Adam with layer‑wise dynamic $\beta_{2}$ designed for PDE workloads and sequence models (e.g., attention) that exhibit deterministic yet heterogeneous conditions, length‑jitter, or rare‑trigger bursts.}
\end{graphicalabstract}

\begin{highlights}
\item Layer-wise dynamic $\beta_2$ adapts to gradient spikes via a bounded sunspike ratio.
\item Drop-in, Adam-faithful; exact Adam recovered when options are off.
\item Improves stability and final loss on Transformer PDE surrogates and PINNs, enabling more aggressive schedules.
\item Cuts binary cross-entropy by $\sim$10\% vs.\ Adam on bursty sequence signals (length-jitter + rare trigger).
\item Reduces final BPC by $\sim$38\% vs.\ Adam–95 and $\sim$58\% vs.\ Adam–999, with universal per-seed wins on a compact character-level LM task.
\item Retains Adam-style convergence guarantees under standard assumptions.
\end{highlights}

\end{frontmatter}

\section{Introduction}
The optimizer proposed in this work takes its name and metaphor from observing a family of Kourkoutas lizards that had adopted my patio grill. 
In the cool morning shade they lay sluggish and hesitant, but once the first spikes of sunlight broke through the cloud cover they came alive, darting about with energy. 
At the same time, we were wrestling with Adam’s slow adaptation when training a Transformer PDE surrogate: long horizons of sluggish updates on smaller 
datasets forced us to generate much larger ones just to reach the target loss. The parallel was irresistible. 
If the lizard needs warmth to explore, so too might an optimizer need warmth from spiky gradients. 
Thus Kourkoutas-$\beta$ was born: an optimizer that reacts dynamically to sunspikes in gradient scale, 
lowering the second‑moment discount  ($\beta_2$) to explore more freely when energized, and restoring high-memory stability when conditions are calm.

Gradient-based optimization underpins modern deep learning. While Adam
\cite{kingma2015adam} is broadly effective, certain regimes produce \emph{large,
intermittent spikes} in gradient norms that slow convergence or require brittle
hyperparameter tuning. In PDE-centric workloads, heterogeneity across samples
(e.g., varying boundary/initial conditions) shifts solution features and yields
abrupt loss changes. Even without randomness in the governing physics, these
deterministic shifts manifest as “spiky’’ gradients. In Physics-Informed Neural
Networks (PINNs) the effect can be amplified by stiff composite losses.

\paragraph{Our setting and observation}
In our recent work on MLX-based Transformer surrogates for heat conduction
\cite{kassinos2025Beyond}, small or heterogeneous datasets produced notably spiky
gradients and—in very small-data regimes—training with standard
Adam often failed without brittle tuning. We therefore adapt Adam’s memory to observed spikiness instead of fixing a single global $\beta_2$.

\paragraph{Kourkoutas-$\beta$ in one line}
We propose \textbf{Kourkoutas-$\beta$}, which replaces Adam’s fixed second-moment
discount $\beta_2$ with a \emph{layer-wise} dynamic value computed from a bounded
“sunspike’’ ratio
\[
  \mathrm{raw}=\frac{\|g\|}{r+\varepsilon_{\mathrm{spike}}},\qquad
  \mathrm{sun}=\frac{\mathrm{raw}}{1+\mathrm{raw}}\in[0,1),
\]
where $r$ is an Exponential Moving Average (EMA) of pooled layer norms (coefficient $\alpha$) . The update
$\beta_{2,t}^{(\ell)}=\beta_{2,\max}-(\beta_{2,\max}-\beta_{2,\min})\,\mathrm{\texttt{sun}}$
lowers $\beta_2$ on spikes (faster reaction) and raises it toward $\beta_{2,\max}$
when calm (stronger smoothing). A short warmup optionally fixes
$\beta_{2,t}^{(\ell)}=\tfrac12(\beta_{2,\min}+\beta_{2,\max})$. Following the code, we store the (squashed) sunspike 
ratio in the variable \texttt{sun}, while \texttt{raw} denotes the unsquashed ratio. Unless stated otherwise, “sunspike” refers to \texttt{sun}.

\paragraph{Faithful to Adam—with controlled extensions}
Beyond dynamic $\beta_2$, the implementation mirrors Adam and adds: (i) optional
AMSGrad or \emph{leaky}-AMSGrad (\texttt{decay}\,$\in(0,1]$), (ii) an optional
trust-region clip \texttt{max\_ratio} providing \emph{element-wise} capping $|\Delta\theta| \le \rho\cdot\texttt{max\_ratio}$,
(iii) an optional adaptive tiny term in the denominator, and (iv) bias-correction
modes \texttt{"none"}, \texttt{"beta2max"}, and \texttt{"exact"} (the latter tracks the
per-layer product of varying $\beta_2$’s). If dynamic $\beta_2$ and all options
are off and \texttt{bias\_correction}=\texttt{"none"}, Kourkoutas-$\beta$ is
\emph{exactly} Adam. With bias-correction set to  \texttt{"beta2max"}, it becomes Adam with bias correction.

\paragraph{When and why it helps}
Kourkoutas-$\beta$ targets regimes with bursty gradients:
\begin{itemize}
  \item \textbf{Data-driven PDE surrogates (Transformers).} Heterogeneous boundary/initial
        conditions across samples create erratic gradient swings. A fixed $\beta_2$ may
        over-smooth or under-react; a dynamic $\beta_2$ reacts swiftly to spikes.
  \item \textbf{Physics-Informed Neural Networks (PINNs).} Stiff residual/BC terms can yield
        persistent bursts. Lowering $\beta_2$ on spikes improves robustness and enables
        more aggressive learning-rate schedules.
  \item \textbf{Quantization-Aware Training (QAT).} Low-bit effects increase gradient variance;
        dynamic $\beta_2$ mitigates bursts where vanilla Adam may stall.
  \item \textbf{Attention-based models and small-batch training.} Sudden attention shifts or tiny
        batches cause large norm variability; modulating $\beta_2$ stabilizes updates.
\end{itemize}

\paragraph{Empirical highlights}
We evaluate four testbeds that stress second–moment adaptation: (i) a data‑driven Transformer surrogate for 2D heat conduction (Heat2D), (ii) a 3D cylindrical PINN (Heat3D), (iii) a non‑PDE MLX synthetic task with length jitter and a 1\% rare trigger (Testbed~C), and (iv) a character‑level Transformer on a 30\,MB slice of \texttt{enwik8} (\texttt{small‑enwik8}) (Testbed~D).
On \textbf{Heat2D} (30 matched seeds; epoch~100), Kourkoutas‑$\beta$ lowers training MSE by \textbf{12.8\%} versus Adam–95 and \textbf{39.4\%} versus Adam–999, winning on \textbf{24/30} and \textbf{30/30} seeds, respectively (paired $t$: $t(29){=}4.205$ and $11.160$; see \S\ref{sec:transformer}, Tables~\ref{tab:transformer-abs}–\ref{tab:transformer-stats-3tests}). 
On the \textbf{Heat3D PINN} (100k epochs; 10 seeds), Kourkoutas‑$\beta$ \textbf{converges in 10/10 runs} (95\% CI [69.2\%,\,100\%]) with median final loss $1.66{\times}10^{-6}$, whereas Adam–95 succeeds in \textbf{1/10} and Adam–999 in \textbf{0/10} (McNemar exact $p{=}0.0039$ and $0.00195$; Table~\ref{tab:pinn3d_summary}, Table~\ref{tab:pinn3d-mcnemar}). 
On \textbf{Testbed~C} (length‑jitter + 1\% rare trigger; 10 seeds), Kourkoutas‑$\beta$ attains lower BCE on \textbf{8/10} seeds and a significant advantage on $\log$‑loss: geometric‑mean loss ratios \textbf{$1.105{\times}$} vs.\ Adam–95 and \textbf{$4.13{\times}$} vs.\ Adam–999 (paired $t(9){=}4.29$, $p{=}0.0022$; Wilcoxon $p{=}0.0098$; \S\ref{sec:toy-rare-trigger}, Tables~\ref{tab:toyrare-seeds}–\ref{tab:toyrare-stats}). 
To test generality beyond physics, \textbf{Testbed~D} trains a compact character‑level Transformer on a 30\,MB \texttt{enwik8} slice with variable lengths (16–512): Kourkoutas‑$\beta$ reaches \textbf{$1.639\pm0.027$ BPC} vs.\ $2.637\pm0.681$ (Adam–95) and $3.906\pm0.087$ (Adam–999), winning \textbf{10/10} seeds; paired $t$ raw $p$‑values are $1.337{\times}10^{-3}$ (Adam–95) and $1.129{\times}10^{-14}$ (Adam–999), Holm‑adjusted to $1.33{\times}10^{-3}$ and $2.40{\times}10^{-14}$ (Tables~\ref{tab:enwik8-abs}–\ref{tab:enwik8-stats}).
Collectively, these results support the central claim: shortening the second‑moment memory during spiky phases and restoring it in calm phases yields consistent stability and loss gains across both PDE and non‑PDE settings.

\paragraph{Theory in brief}
Because $\beta_{2,t}^{(\ell)}\in[\beta_{2,\min},\beta_{2,\max}]\subset(0,1)$ and the
second-moment denominator uses (soft-)monotone $\widehat v_t$, the method retains
Adam-style guarantees (e.g., sublinear regret) under standard assumptions, while
providing practical robustness under spiky gradients.

\section*{Code and data availability}

All code is released under the MIT license. Public repositories and their
archived, version-specific releases are:

\begin{itemize}
  \item Optimizer: \repoKbeta{} (archived release: \doiKbeta) \cite{kbeta_software}.
  \item 3-D PINN benchmark: \repoPINN{} (archived release: \doiPINN) \cite{pinn3d_software}.
  \item 2-D Transformer benchmark: \repoTrans{} (archived release: \doiTrans) \cite{transformer2d_software}.
\end{itemize}

Each repository includes a fully reproducible training script and an environment
file pinned to \texttt{mlx==0.26.3}. The Testbed-C script
\texttt{rare\_trigger\_toy.py} is included verbatim in \ref{app:trigger-code}. For Testbed-D,
\ref{app:codeTestbedD} lists the exact commands; the training code
\texttt{testbed\_d.py} is in \repoKbeta{} under
“Examples”.

\paragraph{\bf Reproducibility note on MLX version}
All Adam baselines reported in this paper were run with MLX~v0.26.3 using the native MLX Adam implementation. We observed minor behavior changes in Adam across nearby MLX releases (e.g., 0.26.3~$\rightarrow$~0.26.5), likely due to kernel and/or default updates. Kourkoutas‑$\beta$ is implemented in our codebase and was unaffected. To reproduce the exact Adam numbers, please use MLX~v0.26.3 (see \ref{app:env} for the environment file). The released repositories are pinned to MLX~v0.26.3, but the code is also compatible with the latest MLX release as of publication.


\section{Kourkoutas-\(\beta\) Overview}
\label{sec:overview}

Kourkoutas-\(\beta\) replaces Adam’s fixed second–moment discount \(\beta_2\) with a 
\emph{layer‑wise} dynamic value driven by a “sunspike” ratio that compares the
current gradient norm to an exponential moving average (EMA). The options discussed in this 
section are summarized in Algorithm~\ref{alg:Algorithms}.

\subsection{Per‑layer pooling and EMA.}
\label{sec:per-layer-pool}
For each layer \(\ell\), define the pooled gradient norm and its EMA
\[
  \|g_t^{(\ell)}\| \;=\; \sqrt{\sum_{p\in\ell}\sum_i g_{t,i}(p)^2},
  \qquad
  r_t^{(\ell)} \;=\; \alpha\,r_{t-1}^{(\ell)} + (1-\alpha)\,\|g_t^{(\ell)}\|,
  \quad r_0^{(\ell)}=0,
\]
where \(\alpha\in(0,1)\) is the EMA coefficient.

\paragraph{Bucketization and notation}
The index \(\ell\) denotes a \emph{parameter bucket}, not necessarily an architectural layer.
Buckets are induced by a user-supplied mapping \texttt{layer\_key\_fn} that assigns each parameter tensor \(p\) to a discrete key; tensors with the same key share the pooled statistic \(r^{(\ell)}_t\) and the resulting per-bucket discount \(\beta^{(\ell)}_{2,t}\).
Typical choices are a stable module path (Transformer), the tensor shape \(p.\mathrm{shape}\) (PINN), or a single global bucket via the constant map $p \mapsto 0$ (i.e., \texttt{layer\_key\_fn=lambda \_: 0}).
\emph{Importantly, when \(\beta_2\) is fixed} (\(\beta_{2,\min}=\beta_{2,\max}=\beta_2\)), the update is \emph{invariant} to \texttt{layer\_key\_fn}: pooled statistics are unused and the method reduces exactly to Adam under the settings in §2.7 and Table~\ref{tab:kbeta-adam-cases}.
See \ref{appendix:KbetaASAdam} for per-step equivalence checks against MLX Adam (with and without bias correction).

\subsection{Sunspike and dynamic \(\beta_2\)}
With a small \(\varepsilon_{\mathrm{spike}}>0\) (code: \texttt{tiny\_spike}), define
\[
  \mathrm{raw}^{(\ell)} \;=\; \frac{\|g_t^{(\ell)}\|}{\,r_t^{(\ell)}+\varepsilon_{\mathrm{spike}}\,},
  \qquad
  \mathrm{sun}^{(\ell)} \;=\; \frac{\mathrm{raw}^{(\ell)}}{1+\mathrm{raw}^{(\ell)}} \in [0,1).
\]
The squashing in \texttt{sun} guards against runaway behavior. 
The per‑layer discount is
\[
  \beta_{2,t}^{(\ell)} \;=\;
  \begin{cases}
    \tfrac{1}{2}\bigl(\beta_{2,\min}+\beta_{2,\max}\bigr), & t \le \texttt{warmup\_steps},\\[4pt]
    \beta_{2,\max} - \bigl(\beta_{2,\max}-\beta_{2,\min}\bigr)\,\mathrm{sun}^{(\ell)}, & t > \texttt{warmup\_steps},
  \end{cases}
\]
so large spikes drive \(\beta_{2,t}^{(\ell)}\!\to\!\beta_{2,\min}\) while calm steps keep
\(\beta_{2,t}^{(\ell)}\!\approx\!\beta_{2,\max}\).

\subsection{Moments and \(v\) variants.}
For each parameter tensor \(p\) in layer \(\ell\),
\[
m_t(p) \;=\; \beta_1\,m_{t-1}(p) + (1-\beta_1)\,g_t(p),
\qquad
v_t(p) \;=\; \beta^{(\ell)}_{2,t}\,v_{t-1}(p) + \bigl(1-\beta^{(\ell)}_{2,t}\bigr)\,g_t(p)^{\odot 2},
\]
where \(g_t(p)^{\odot 2}\) denotes the \emph{elementwise} square. The second moment used in the denominator is
\[
\widehat v_t(p) =
\left\{
\begin{array}{@{}l@{\quad}l@{}}
  v_t(p) 
    & \text{no }v_{\max} \text {(no AMSGrad)}, \\[4pt]
  \max\ \!\bigl(v_t(p),\,v^{\max}_{t-1}(p)\bigr),\ 
    \texttt{decay}=1\ \text{or}\ 
    (\texttt{decay}=\texttt{None} \ \text{\&} \ v_{\max} \ \text{active via} \ \texttt{max\_ratio}) 
    & \textbf{AMSGrad}, \\[6pt]
  \max\ \!\bigl(v_t(p),\,\delta\,v^{\max}_{t-1}(p)\bigr),\ 
    \delta=\texttt{decay}\in(0,1) 
    & \textbf{leaky-AMSGrad}.
\end{array}
\right.
\]

\paragraph{Activation logic and coupling to trust‑region}
In our implementation the \(v_{\max}\) buffer is maintained iff either the AMSGrad knob is enabled 
\(\bigl(\delta=\texttt{decay}\in(0,1]\bigr)\) \emph{or} the trust‑region clip is enabled \(\bigl(\texttt{max\_ratio}\neq\texttt{None}\bigr)\).
Thus:
\begin{itemize}
  \item \(\delta=\texttt{None}\) and \(\texttt{max\_ratio}=\texttt{None}\) \(\Rightarrow\) no \(v_{\max}\) (plain \(v\)).
  \item \(\delta=1.0\) \(\Rightarrow\) hard AMSGrad (non‑decreasing \(v_{\max}\)).
  \item \(\delta\in(0,1)\) \(\Rightarrow\) leaky AMSGrad:
        \(\;v^{\max}_t = \max\bigl(\delta\,v^{\max}_{t-1},\,v_t\bigr)\).
  \item \(\texttt{max\_ratio}\neq\texttt{None}\) \(\Rightarrow\) \(v_{\max}\) is active. If \(\delta=\texttt{None}\) or \(\delta=1.0\),
        the update uses hard AMSGrad together with clipping; if \(\delta\in(0,1)\) it uses leaky‑AMSGrad with clipping.
\end{itemize}

\noindent With the current code, enabling the clip alone implicitly uses AMSGrad; there is no “clip‑only without \(v_{\max}\)” path.
To obtain hard AMSGrad without the clip, set \(\delta=1.0\) and leave \(\texttt{max\_ratio}=\texttt{None}\).
(Alg.~\ref{alg:Algorithms} initializes \(v_{\max}\) iff \texttt{decay} or \texttt{max\_ratio} is set; the per‑step rule then selects the hard vs.\ soft‑max branch.)

\begin{table}[h]
\centering
\caption{Denominator / clip behaviour by settings (implementation semantics).}
\label{tab:denom-clip-matrix}
\begin{tabular}{l l l l l}
\toprule
\textbf{decay} & \textbf{max\_ratio} & \textbf{$v_{\max}$ active?} & \textbf{AMSGrad mode} & \textbf{Clip?} \\
\midrule
None   & None      & No                   & ---                 & No \\
1.0    & None      & Yes                  & Hard (non-decreasing) & No \\
(0,1)  & None      & Yes                  & Soft-max (leaky)    & No \\
None   & $>0$      & Yes                  & Hard                & Yes \\
1.0    & $>0$      & Yes                  & Hard                & Yes \\
(0,1)  & $>0$      & Yes                  & Soft-max (leaky)    & Yes \\
\bottomrule
\end{tabular}

\vspace{0.5ex}
\footnotesize\emph{Notes.} (i) In the current code the trust-region clip implies $v_{\max}$ allocation; there is no clip-only path. (ii) To get hard AMSGrad without the clip, use \texttt{decay}=1.0 and \texttt{max\_ratio}=None. (iii) With \texttt{decay}=None and \texttt{max\_ratio}=None the method reduces to the plain $v$ denominator. 
\end{table}

\subsection{Bias correction options}
Let
\[
a_{1,t} \;=\;
\begin{cases}
  1, & \text{\texttt{bias\_correction} = ``none''},\\[3pt]
  1-\beta_1^t, & \text{\texttt{bias\_correction} \(\in\) \{``beta2max'', ``exact''\}},
\end{cases}
\]
\[
b_{2,t}^{(\ell)} \;=\;
\begin{cases}
  1, & \text{\texttt{bias\_correction} = ``none''},\\[3pt]
  1 - \beta_{2,\max}^{\,t}, & \text{\texttt{bias\_correction} = ``beta2max''},\\[6pt]
  1 - \displaystyle\prod_{i=1}^{t} \beta_{2,i}^{(\ell)}, & \text{\texttt{bias\_correction} = ``exact''}.
\end{cases}
\]
For quick reference:
\[
\text{``beta2max'':}\quad b_{2,t} = 1 - \beta_{2,\max}^{\,t},
\qquad
\text{``exact'':}\quad b_{2,t}^{(\ell)} = 1 - \prod_{i=1}^{t} \beta_{2,i}^{(\ell)}.
\]

\paragraph{Numerical notes on  \texttt{"beta2max"} and \texttt{"exact"}}%
\texttt{“beta2max”} lower‑bounds\footnote{\[
\prod_{i=1}^t \beta_{2,i}\;\le\;\beta_{2,\max}^{\,t}
\quad\Longrightarrow\quad
1-\textstyle\prod_{i=1}^t \beta_{2,i}\;\ge\;1-\beta_{2,\max}^{\,t}.
\]
Thus using \(b_{2,t}=1-\beta_{2,\max}^{\,t}\) makes the denominator slightly larger (a conservative choice) compared to dividing by the
exact factor \(1-\prod_{i=1}^t\beta_{2,i}\).} the true denominator correction when $\beta_{2,t}\le \beta_{2,\max}$.
In \texttt{“exact”} we track the per‑layer product $\prod_{i=1}^t \beta^{(\ell)}_{2,i}$. 
The reference implementation keeps a per‑parameter slot (equivalently, a per‑layer scalar) that accumulates the product $\prod_{i=1}^{t}\beta^{(\ell)}_{2,i}$; 
since all parameters in a bucket share the same $\beta^{(\ell)}_{2,t}$, storing it per parameter or per layer is mathematically identical and it is broadcast within the bucket at use time.
At very long horizons the product can underflow
to zero, making $b_{2,t}^{(\ell)} \approx 1$, which is benign once
$b_{2,t}^{(\ell)}$ is near one, but log‑space accumulation is numerically preferable: keep $s^{(\ell)}_{2,t}=\sum_{i=1}^t \log \beta^{(\ell)}_{2,i}$ and 
compute $b^{(\ell)}_{2,t} = -\expmone\!\big(s^{(\ell)}_{2,t}\big)$.

\subsection{Denominator and update.}
Let \(\delta_{\mathrm{tiny}}\in\{0,1\}\) denote the adaptive‑tiny switch
(\(\delta_{\mathrm{tiny}}=1\) when \texttt{adaptive\_tiny}=\texttt{True}, else \(0\)).
With \(\varepsilon\) matching the code’s \texttt{eps} and
\(\langle|p|\rangle\) is the elementwise mean of \(|p|\),
\[
\mathrm{denom}_t(p)
\;=\;
\sqrt{\frac{\widehat v_t(p)}{\,b_{2,t}^{(\ell)}\,}}
\;+\;
\varepsilon
\;+\;
\delta_{\mathrm{tiny}}\,
\text{\texttt{tiny\_denom}}\,
\max\ \!\bigl(\langle |p|\rangle,\,1\bigr).
\]
The step and parameter update are
\[
\Delta\theta_t(p)\;=\;\frac{\rho}{a_{1,t}}\;\frac{m_t(p)}{\mathrm{denom}_t(p)},
\qquad
\theta_{t+1}(p)\;=\;\theta_t(p)\;-\;
\mathrm{clip}\!\left(\Delta\theta_t(p),\,\pm\,\rho\cdot\texttt{max\_ratio}\right),
\]
where the clip applies only when \texttt{max\_ratio} is set; otherwise
\(\mathrm{clip}(x,\pm L)=x\).

\subsubsection{Trust-region clip (\texttt{max\_ratio}).}
We optionally cap the elementwise update magnitude after bias correction:
\[
\Delta\theta_t(p) \leftarrow \operatorname{clip}\ \!\bigl(\Delta\theta_t(p),\,\pm\rho\cdot\texttt{max\_ratio}\bigr),
\]
which enforces $\lvert\Delta\theta_t(p)\rvert \le \rho\cdot\texttt{max\_ratio}$ coordinate-wise 
(Alg.~\ref{alg:Algorithms}, trust‑region block). 
This guard stabilizes rare bursts without altering the gradient statistics or the $\beta_2$ logic.

\paragraph{Interaction with $v_{\max}$}
For implementation economy, the optimizer allocates and maintains a $v_{\max}$ buffer whenever the clip is enabled; consequently, the denominator uses AMSGrad with a hard bound when 
$\texttt{max\_ratio}\neq\texttt{None}$ and $\delta=\texttt{None}$. 
If you prefer clip-only (no AMSGrad), the code would need a minor refactor to suppress $v_{\max}$ allocation in that case. 
(See \S2.3 ``Activation logic''.)

\subsection{Warmup.}
For steps \(t \le \texttt{warmup\_steps}\) the sunspike is held at zero and
\(\beta_{2,t}^{(\ell)}=\tfrac12\bigl(\beta_{2,\min}+\beta_{2,\max}\bigr)\). Warmup does not affect the learning rate.

\subsection{Plain Adam as a special case (exact equivalence)}
\label{sec:plain-adam-as-sc}
To make Kourkoutas-\(\beta\) behave exactly like plain Adam with \texttt{bias\_correction}=\texttt{"none"}, set:
\begin{itemize}
  \item \textbf{No dynamic \(\beta_2\):} \(\beta_{2,\min}=\beta_{2,\max}=\beta_2\) (e.g.\ \(0.999\)).
  \item \textbf{Bias correction (BC) off:} \texttt{bias\_correction}=\texttt{"none"} \(\Rightarrow a_{1,t}=b_{2,t}^{(\ell)}=1\).
  \item \textbf{No AMSGrad / soft‑max leak:} \texttt{decay}=\texttt{None}.
  \item \textbf{No trust‑region clip:} \texttt{max\_ratio}=\texttt{None}.
  \item \textbf{No adaptive tiny:} \texttt{adaptive\_tiny}=\texttt{False}.
  \item \textbf{No warmup:} \texttt{warmup\_steps}=0.
  \item Diagnostics may be on or off; they do not change the math.
\end{itemize}
With those settings, for each parameter tensor \(p\),
\[
\Delta\theta_t(p)\;=\;\rho\,\frac{m_t(p)}{\sqrt{v_t(p)}+\varepsilon},
\quad
m_t(p)=\beta_1 m_{t-1}(p)+(1-\beta_1)g_t(p),\quad
v_t(p)=\beta_2 v_{t-1}(p)+(1-\beta_2)\,g_t(p)^{\odot 2},
\]
which is standard Adam with bias correction disabled. We verified per‑step numerical equivalence between MLX Adam (bias‑correction off) and Kourkoutas‑\(\beta\) configured as Adam, within O(1e‑6) absolute error in FP32 (\ref{appendix:ablations}), with residual differences attributable to benign differences in arithmetic ordering and kernel fusion.

\paragraph{Clarifications}
\begin{itemize}
  \item The \texttt{"beta2max"} and \texttt{"exact"} bias‑correction options are specific to Kourkoutas-\(\beta\) (useful when \(\beta_2\) is dynamic). If \(\beta_2\) is fixed, both reduce to the usual \(1-\beta_2^t\) correction.
  \item If you \emph{keep} dynamic \(\beta_2\) but set \texttt{bias\_correction}=\texttt{"none"}, you intentionally deviate from standard Adam, which assumes a fixed \(\beta_2\). 
  \item  {\bf Note:} When $\beta_2$ is fixed, the choice of \texttt{layer\_key\_fn} (single global bucket $\lambda\!:\!0$, coarser pooling \texttt{p.shape}, or a fine module path/id) has no effect on the update.

\end{itemize}
In summary, Table~\ref{tab:kbeta-adam-cases} shows when Kourkoutas–$\beta$ matches or deviates from Adam.

\begin{table}[h]
\centering
\caption{Cases where Kourkoutas--$\beta$ matches or deviates from Adam.}
\label{tab:kbeta-adam-cases} 
\setlength{\tabcolsep}{5pt}
\begin{tabular}{@{}llll@{}}
\toprule
\textbf{Dynamic $\beta_2$} & \textbf{Other opts} & \textbf{BC mode} & \textbf{Behavior} \\
\midrule
Off & All off & \texttt{none} & Exact Adam (no BC) \\
Off & All off & \texttt{beta2max}, $\beta_{2,\max}=\beta_2$ & Exact Adam (with BC) \\
On  & ---     & \texttt{none} & Deviates (intentional) \\
On  & ---     & \texttt{beta2max} & Not standard Adam (BC uses capped $\beta_{2,t}$) \\
\bottomrule
\end{tabular}
\end{table}

\subsection{High-Level Intuition}

Recall the metaphor introduced in the opening section: Kourkoutas-$\beta$ behaves like a lizard whose activity depends on the warmth of the sun. 
When the sun spikes—the gradient is large relative to its history—the optimizer lowers its memory parameter ($\beta_2$) and moves more freely; when the gradients are calm, 
it restores $\beta_2$ toward its maximum, maintaining long-memory stability. This simple picture motivates the formal dynamic update rule we now describe.

\subsection{Paper Contributions}
\begin{itemize}
    \item \emph{Dynamic \(\beta_2\) advantage:} We show how Kourkoutas-\(\beta\) helps avoid 
        or reduce the “spiky gradient slowdowns” seen in PDE surrogates, 
        QAT, or attention-based tasks.
    \item \emph{Practical formula and code:} 
        Just a few lines differ from standard Adam. 
        The approach is drop-in for existing frameworks.
    \item \emph{Convergence analysis:} 
        We provide a theoretical sketch showing that even though \(\beta_2\) changes
        per iteration, the method inherits Adam’s sublinear \emph{regret} or 
        diminishing gradient norm properties, under usual assumptions.
\end{itemize}


\section{The Kourkoutas-\(\beta\) Optimizer in Practice}
\label{sec:practice}

For the concrete MLX implementation see public GitHub repo~\repoKbeta. This section records only the
\emph{code-level toggles and defaults}. All equations and definitions (pooled norms and EMA, sunspike, dynamic
\(\beta_2\), second-moment variants \(\widehat v_t\), bias-correction factors, denominator, update rule, and warmup)
appear  in \S\ref{sec:overview}.
\medskip

\noindent\textbf{Implementation notes (pointers to \S\ref{sec:overview}).}
\begin{itemize}
  \item \emph{Layer grouping.} As defined in \ref{sec:per-layer-pool}, layers are formed by a user-supplied \texttt{layer\_key\_fn} (default: tensor
        identity/name). The pooled norm is the L2 over all parameters in a layer.
  \item \emph{Warmup.} As in \S\ref{sec:overview}, during $t\le\texttt{warmup\_steps}$ the sunspike is held at zero and
        \(\beta_{2,t}^{(\ell)}=\tfrac12(\beta_{2,\min}+\beta_{2,\max})\).
   \item \emph{Bias correction.}  Modes are \texttt{"none"}, \texttt{"beta2max"}, and \texttt{"exact"}; definitions are in §2. 
For \texttt{"exact"}, the reference implementation stores a \emph{per-parameter} cumulative product \texttt{st["beta2\_cumprod"]} and broadcasts within each bucket (equivalently, one scalar per layer). 
See Table~\ref{tab:denom-clip-matrix} for the relation to (non)standard Adam.       
 \item \emph{Denominator variants (AMSGrad / leaky-AMSGrad).} Implemented exactly as described in
        \S\ref{sec:overview}: the $v_{\max}$ buffer is active whenever 
either decay is set (leaky-AMSGrad; $\delta\in(0,1]$) or \texttt{max\_ratio} is set (trust-region). With \texttt{decay} = \texttt{None}
 and \texttt{max\_ratio=None} we use the plain $v$ denominator. With \texttt{decay}=\texttt{1.0} we recover hard AMSGrad (non‑decreasing $v_{\max}$).  
            
  \item \emph{Adaptive tiny.} Optional additive term \(\texttt{tiny\_denom}\cdot\max(\langle|p|\rangle,1)\) in the
        denominator; see \S\ref{sec:overview} for the full denominator expression, where $\langle|p|\rangle$ is the scalar mean of the absolute values in tensor $p$.
  \item \emph{Trust‑region clip.} Optional elementwise clip \(|\Delta\theta_t|\le \rho\cdot\texttt{max\_ratio}\) applied after bias correction (see the
        update in \S\ref{sec:overview} and Alg.~\ref{alg:Algorithms} lines 25-27).  Note: in this implementation, enabling the clip also enables the $v_{\max}$ buffer (AMSGrad) by construction; to have clip‑only you would need a minor code refactor.  
        
  \item \emph{Plain Adam equivalence.} The exact conditions are stated once in \S\ref{sec:overview} (``Plain Adam as a
        special case'') and verified per-step in~\ref{appendix:ablations}.
\end{itemize}

\subsection*{Why another Adam variant?}
Kourkoutas-\(\beta\) differs from prior Adam‑style methods by adapting \(\beta_2\) \emph{layer‑wise}
in response to a bounded sunspike signal \((\in[0,1))\): large spikes lower \(\beta_2\) to react faster,
calm phases keep \(\beta_2\) near \(\beta_{2,\max}\) for smoothing. Optional leaky-AMSGrad, trust‑region clipping, adaptive tiny, and
bias‑correction modes round out a drop‑in, code‑faithful optimizer that is robust under spiky gradients
(e.g.\ in PDE and attention workloads). In contrast, AdaBelief modifies the
\emph{variance model} (belief residuals), RAdam rectifies early‑phase variance without dynamic momentum,
Yogi stabilizes the second‑moment growth via additive updates, and RMSProp keeps a fixed decay with no
mechanism to respond to spikes. 
\medskip

\noindent\textit{On other Adam variants.} We cite AdaBelief, RAdam, Yogi, and RMSProp but did not re-run these baselines in our regimes. Our goal here is to isolate the effect of \emph{layer‑wise dynamic} $\beta_2$, so the most relevant control is Adam with a \emph{fixed} $\beta_2$ under identical schedules and $\beta_1,\varepsilon$ -- which our implementation exactly recovers when dynamic $\beta_2$ and all extras are disabled (see \S\ref{sec:plain-adam-as-sc}, "Plain Adam as a special case").  By contrast, these methods target different failure modes: RAdam rectifies early‑phase variance, Yogi controls second‑moment growth via additive updates, AdaBelief changes the variance model, and RMSProp uses a fixed decay with no mechanism to respond to spikes; these are orthogonal to our lever (dynamic $\beta_2$) and would require problem‑specific tuning for a fair comparison. We therefore focus on fixed‑$\beta_2$ Adam as the matched control and leave a full head‑to‑head with those variants to future work (code will be released to facilitate drop‑in comparisons).

\makeatletter
\algrenewcommand\alglinenumber[1]{\hbox to 2.0em{\scriptsize #1:\hss}}
\makeatother
\algrenewcommand\algorithmicindent{1.2em}
\algrenewcommand{\algorithmiccomment}[1]{\hspace{2.5em}{\scriptsize$\triangleright$}~\scriptsize #1}

\Needspace{40\baselineskip}
\begin{figure}[H]  

\captionof{algorithm}{\bf Adam (bias correction off)}\label{alg:adam}
{\footnotesize
\begin{algorithmic}[1]
  \Require step size $\rho>0$, $\beta_1,\beta_2\in(0,1)$, $\varepsilon>0$
  \State Initialize $m_0(p)\gets 0$, $v_0(p)\gets 0$ for all parameters $p$
  \For{$t=1,2,\dots$}
    \State Compute gradients $g_t(p)=\nabla_\theta f(\theta_{t-1})(p)$
    \State $m_t(p)\gets \beta_1 m_{t-1}(p) + (1-\beta_1)\,g_t(p)$
    \State $v_t(p)\gets \beta_2 v_{t-1}(p) + (1-\beta_2)\,g_t(p)^{\odot 2}$ \Comment{elementwise square}
    \State $\Delta\theta_t(p)\gets \displaystyle \rho\,\frac{m_t(p)}{\sqrt{v_t(p)}+\varepsilon}$
    \State $\theta_t(p)\gets \theta_{t-1}(p)-\Delta\theta_t(p)$
  \EndFor
\end{algorithmic}
}

\captionsetup{type=algorithm}
\captionof{algorithm}{\bf Kourkoutas-\texorpdfstring{$\beta$}{beta} (softmax-flex; dynamic $\beta_2$ with optional (leaky) AMSGrad and trust-region)}\label{alg:Algorithms}

{\footnotesize
\begin{algorithmic}[1]
  \Require step size $\rho>0$, $\beta_1\in(0,1)$, $\beta_{2,\min}\le \beta_{2,\max}$, EMA $\alpha\in(0,1)$ (optionally scheduled), $\varepsilon>0$
  \Require options: \texttt{tiny\_spike}, \texttt{tiny\_denom},
           \texttt{decay} $\in (0,1] \cup \{0, \text{\texttt{None}}\}$,
           \texttt{max\_ratio} $\in \mathbb{R}_{+} \cup \{\text{\texttt{None}}\}$,
           \texttt{adaptive\_tiny} $\in \{0,1\}$,
           \texttt{bias\_correction} $\in \{\mathtt{none},\mathtt{beta2max},\mathtt{exact}\}$,
           \texttt{warmup\_steps}, \texttt{layer\_key\_fn}

  \State Initialize $m_0(p)\gets 0$, $v_0(p)\gets 0$ for all $p$
  \If{\texttt{decay}$\neq$\texttt{None} \textbf{or} \texttt{max\_ratio}$\neq$\texttt{None}}
    \State set $v^{\max}_0(p)\gets 0$ \Comment{$v^{\max}$ exists iff \texttt{decay} or \texttt{max\_ratio} is set (trust-region alone activates it)}
  \EndIf
  \For{\textbf{each bucket (a.k.a.\ layer)} $\ell$} \State $r^{(\ell)}_0\gets 0$ \EndFor

\Statex \textit{Semantics (AMSGrad / clip interplay; code-true):}
\Statex \quad\textbullet\; \texttt{decay}=\texttt{None} \& \texttt{max\_ratio}=\texttt{None} $\Rightarrow$ \textbf{no $v_{\max}$} (plain $v$).
\Statex \quad\textbullet\; \texttt{decay}=1.0 $\Rightarrow$ \textbf{hard AMSGrad} (non-decreasing $v_{\max}$):
$v^{\max}_{t} \gets \max(v^{\max}_{t-1},\, v_t)$; clipping applies if \texttt{max\_ratio} is set.
\Statex \quad\textbullet\; \texttt{decay}$\,\in\,(0,1)$ $\Rightarrow$ \textbf{leaky-AMSGrad}:
$v^{\max}_{t} \gets \max(\texttt{decay}\cdot v^{\max}_{t-1},\, v_t)$; clipping applies if set.
\Statex \quad\textbullet\; \texttt{decay}=0 $\Rightarrow$ \textbf{degenerate AMSGrad}:
$v^{\max}_t=\max(0\cdot v^{\max}_{t-1}, v_t)=v_t$, so $\widehat v_t=v_t$; clipping applies if set.
\Statex \quad\textbullet\; \texttt{max\_ratio}\,$\neq$\,\texttt{None} $\Rightarrow$ \textbf{clipping} always applies (post bias-correction) \emph{and} a $v_{\max}$ buffer is allocated; the denominator mode is:
\Statex \quad\hspace{2.2em}\textemdash\; \texttt{decay}=\texttt{None} or $1.0$: \textbf{hard AMSGrad} $($with clipping$)$;
\Statex \quad\hspace{2.2em}\textemdash\; \texttt{decay}\,$\in$\,($0,1$): \textbf{leaky-AMSGrad} $($with clipping$)$;
\Statex \quad\hspace{2.2em}\textemdash\; \texttt{decay}=0: \textbf{degenerate} ($\widehat v_t=v_t$) $($with clipping$)$.

  \For{$t=1,2,\dots$}
    \State Compute $g_t(p)$ and group parameters into layers $\ell$ via \texttt{layer\_key\_fn}
    \For{each layer $\ell$}
      \State $\|g_t^{(\ell)}\|\gets \sqrt{\sum_{p\in\ell}\sum_i g_{t,i}(p)^2}$
      \State $r_t^{(\ell)}\gets \alpha\,r_{t-1}^{(\ell)} + (1-\alpha)\,\|g_t^{(\ell)}\|$ \Comment{EMA of pooled grad norm}
      \State $\text{raw}\gets \dfrac{\|g_t^{(\ell)}\|}{\,r_t^{(\ell)}+\texttt{tiny\_spike}\,}$,\quad
             $\text{sun}\gets \begin{cases}
               0, & t \le \texttt{warmup\_steps}\\[2pt]
               \text{raw}/(1+\text{raw}), & \text{otherwise}
             \end{cases}$
      \State $\beta^{(\ell)}_{2,t}\gets \begin{cases}
         \tfrac12(\beta_{2,\min}+\beta_{2,\max}), & t \le \texttt{warmup\_steps}\\[2pt]
         \beta_{2,\max}-(\beta_{2,\max}-\beta_{2,\min})\,\text{sun}, & \text{otherwise}
      \end{cases}$

      \For{each $p\in\ell$}
        \State $m_t(p)\gets \beta_1 m_{t-1}(p) + (1-\beta_1)\,g_t(p)$
        \State $v_t(p)\gets \beta^{(\ell)}_{2,t} v_{t-1}(p) + \bigl(1-\beta^{(\ell)}_{2,t}\bigr)\,g_t(p)^{\odot 2}$
        \State \textbf{choose }$\widehat v_t(p)$ \textbf{(mutually exclusive)} \Comment{If \texttt{max\_ratio} is set and \texttt{decay} is None, the hard AMSGrad line applies.}
        \If{\texttt{decay}$\in(0,1)$} \Comment{leaky-AMSGrad (soft-max bound)}
          \State $v^{\max}_t \gets \max(\texttt{decay}\cdot v^{\max}_{t-1},\, v_t)$; \quad $\widehat v_t \gets v^{\max}_t$
        \ElsIf{\texttt{decay}$=1.0$ \textbf{ or } (\texttt{decay}$=\texttt{None}$ \textbf{ and } \texttt{max\_ratio}$\neq\texttt{None}$)} \Comment{hard AMSGrad}
          \State $v^{\max}_t \gets \max(v^{\max}_{t-1},\, v_t)$; \quad $\widehat v_t \gets v^{\max}_t$
        \Else \Comment{plain Adam-style denominator (no $v^{\max}$)}
          \State $\widehat v_t \gets v_t$
        \EndIf

        \State $a_{1,t}\gets \begin{cases}
          1, & \texttt{bias\_correction}=\texttt{"none"}\\
          1-\beta_1^t, & \texttt{"beta2max"}\ \text{or}\ \texttt{"exact"}
        \end{cases}$

        \State $b_{2,t}^{(\ell)}\gets \begin{cases}
          1, & \texttt{bias\_correction}=\texttt{"none"}\\
          1-\beta_{2,\max}^{\,t}, & \texttt{"beta2max"}\\
          1-\prod_{i=1}^{t}\beta^{(\ell)}_{2,i}, & \texttt{"exact"}
        \end{cases}$

        \State Let $\delta_{\mathrm{tiny}} \gets \mathbf{1}\{\texttt{adaptive\_tiny}\}$
        \State $\mathrm{denom}_t(p)\gets \sqrt{\dfrac{\widehat v_t(p)}{\,b_{2,t}^{(\ell)}\,}} + \varepsilon
               + \delta_{\mathrm{tiny}}\cdot\texttt{tiny\_denom}\cdot \max\bigl(\langle|p|\rangle,1\bigr)$
        \State $\Delta\theta_t(p)\gets \displaystyle \frac{\rho}{a_{1,t}}\,\frac{m_t(p)}{\mathrm{denom}_t(p)}$

        \If{\texttt{max\_ratio}$\neq$\texttt{None}} \Comment{elementwise trust region}
          \State $\Delta\theta_t(p)\gets \mathrm{clip}\!\bigl(\Delta\theta_t(p),\,\pm\,\rho\cdot\texttt{max\_ratio}\bigr)$
          \Statex \hspace{2.8em}\Comment{\footnotesize Clip does not set \texttt{decay}, but in this implementation \emph{does} imply $v^{\max}$ is maintained (hard AMSGrad if \texttt{decay} is None).}
        \EndIf

        \State $\theta_t(p)\gets \theta_{t-1}(p)-\Delta\theta_t(p)$
      \EndFor
    \EndFor
  \EndFor
\end{algorithmic}
}
\end{figure}

\noindent\textbf{Bucketization.}
In Algorithm~\ref{alg:Algorithms}, the index~$\ell$ denotes a parameter \emph{bucket} returned by \lstinline|layer_key_fn(param)|.
For the PINN we use \lstinline|p.shape| (coarser pooling); for the Transformer we use a stable parameter–path string (finer pooling).
If \lstinline|layer_key_fn| is omitted, the code falls back to \lstinline|p.name| (if present) or \lstinline|id(p)|.
Coarse buckets (e.g., \lstinline|p.shape|) amortize statistics and slightly reduce overhead, yielding a smoother, more data‑efficient $\beta_2$ signal; fine buckets (e.g., module path) localize the adaptation.
Setting \lstinline|layer_key_fn=lambda _: 0| collapses everything into a single global bucket, recovering a single $\beta_2$ schedule. {\bf Note:} When $\beta_2$ is fixed, the choice of \texttt{layer\_key\_fn} (single global bucket $\lambda\!:\!0$, coarser pooling \texttt{p.shape}, or a fine module path/id) has no effect on the update.



\section{Reproducibility of Methods}
\label{sec:repro}

\paragraph{Version pin for the MLX Adam baseline}
Across recent MLX releases we observed small but noticeable differences in the \emph{early} trajectory of the built‑in Adam optimizer (e.g., between v0.26.3 and v0.28.0). In our PINN–3D runs (seed=0), both versions ultimately settled on the same plateau loss, but the initial paths differed. To ensure strict comparability, \textbf{all Adam baselines reported in this paper were run with MLX v0.26.3}. We provide an environment file to reproduce exactly the numbers and figures reported here (~\ref{app:env}).\footnote{We attribute these small changes to low‑level kernel/arithmetic ordering rather than to algorithmic differences. Our Kourkoutas–$\beta$ implementation lives in our codebase and was unaffected by MLX updates.}

\paragraph{Control: Kourkoutas–$\beta$ configured as plain Adam}
To verify that performance differences in §\ref{sec:experiments} arise from \emph{dynamic} $\beta_2$ rather than implementation details, we also ran a control where Kourkoutas–$\beta$ is configured to recover plain Adam: $\beta_{2,\min}{=}\beta_{2,\max}$, \texttt{decay=None}, \texttt{max\_ratio=None}, \texttt{adaptive\_tiny=False}, \texttt{warmup\_steps=0}. 
With \texttt{bias\_correction="none"} this matches “Adam without bias correction” exactly (Overview, “Plain Adam as a special case”); with \texttt{bias\_correction="beta2max"} and fixed $\beta_2$ it reproduces the standard correction factor $1-\beta_2^{t}$. 
On controlled problems we observe step‑by‑step agreement within numerical tolerance (FP32: $\lesssim 2\times 10^{-6}$ absolute). In the full PINN, early trajectories can still diverge slightly (benign arithmetic effects), but final losses at 100k steps are comparable. This control is \emph{not} a competing method; it exists solely to validate equivalence of the update rule when $\beta_2$ is fixed.

\paragraph{Bucketization invariance for the control}
As already noted, when $\beta_2$ is fixed, the choice of \texttt{layer\_key\_fn} (single global bucket $\lambda\!:\!0$, coarser pooling \texttt{p.shape}, or a fine module path/id) has no effect on the update; we verified that trajectories are indistinguishable within floating‑point tolerance in this setting. For the main dynamic‑$\beta_2$ runs we use the per‑testbed bucketing stated in §\ref{sec:practice} and §\ref{sec:experiments} because it governs how the sunspike signal is pooled.

\paragraph{Compute environment}
Unless noted otherwise, all experiments (including the toy checks) were run on a single Apple Studio M2 Ultra with 198\,GB unified memory. Timings are wall‑clock.

\section{Sanity Checks (Toy Problems with bias correction off)}
\label{sec:sanity}

We include three small, deterministic “toy” problems to confirm that Kourkoutas‑\(\beta\) behaves exactly as described in Secs.~\S\ref{sec:overview}–\S\ref{sec:practice} and in our reference implementation.  Hence, to make the comparison faithful to the code paths in Secs.~\S\ref{sec:overview}–\S\ref{sec:practice}, in all three sanity checks, we use a single global bucket (\lstinline|layer_key_fn=lambda _: 0|), \texttt{decay}=\texttt{None}, \texttt{max\_ratio}=\texttt{None}, \texttt{adaptive\_tiny}=\texttt{False}, and \texttt{bias\_correction}=\texttt{"none"} (matching MLX Adam’s default). Each method runs with identical initializations per repeat and a short untimed warm‑up to stabilize JIT/kernels; we report medians over five repeats. Because these are micro‑benchmarks, we treat wall‑clock time as ancillary and focus on the optimization behavior; timing for the full PDE workloads appears in the Results section. The code for all three sanity checks is given verbatim in~\ref{app:synthetic-code}.

\subsection{Sanity‑1: least‑squares regression (convex).}
This test places all optimizers on a smooth quadratic where gradients are well behaved and the sunspike signal is typically small. We minimize a standard squared loss for a linear model over \(10{,}000\) steps with step size \(\rho=10^{-3}\). As expected, Kourkoutas‑\(\beta\) with dynamic \(\beta_2\) converges to the same minimum as its fixed‑\(\beta_2\) configuration and MLX Adam (bias correction off); the final losses are numerically indistinguishable. On this tiny setup, MLX Adam executes slightly faster per step, which is consistent with its lean code path and the absence of per‑layer bookkeeping.

\begin{table}[H]
\centering
\small
\begin{tabular}{lcc}
\toprule
Optimizer & Final MSE (median) & Time (median) \\
\midrule
K‑\(\beta\) (dynamic \(\beta_2\), BC off)  & \(4.699453\times 10^{-5}\) & 4.218\,s \\
K‑\(\beta\) (fixed \(\beta_2\), BC off)    & \(4.692386\times 10^{-5}\) & 4.217\,s \\
MLX Adam (BC off)                          & \(4.692433\times 10^{-5}\) & 3.155\,s \\
\bottomrule
\end{tabular}
\caption{All three converge to the same optimum; MLX Adam is faster on this tiny toy.}
\label{tab:sanity1}
\end{table}

\subsection{Sanity‑2: logistic regression on separable data (nonconvex but well‑behaved).}
Here we probe the regime where gradient norms spike early: a linearly separable classification problem trained with logistic loss for \(20{,}000\) steps at \(\rho=10^{-2}\). All methods reach \(100\%\) accuracy, but the dynamic \(\beta_2\) in Kourkoutas‑\(\beta\) lowers the loss markedly at a fixed step budget—by reacting to transient spikes with a smaller \(\beta_2\) and then drifting back toward \(\beta_{2,\max}\) once the boundary stabilizes. In our runs this shows up as a substantially smaller terminal logistic loss for Kourkoutas‑\(\beta\) (dynamic \(\beta_2\)) than for fixed‑\(\beta_2\) or MLX Adam, while the latter two are essentially tied—as they should be, since Kourkoutas‑\(\beta\) with \(\beta_{2,\min}=\beta_{2,\max}\) reproduces the Adam update when bias correction is off.

\begin{table}[H]
\centering
\small
\begin{tabular}{lccc}
\toprule
Optimizer & Final loss (median) & Acc (median) & Time (median) \\
\midrule
K‑\(\beta\) (dynamic \(\beta_2\), BC off)  & \(3.501587\times 10^{-9}\) & 1.000 & 8.767\,s \\
K‑\(\beta\) (fixed \(\beta_2\), BC off)    & \(9.259177\times 10^{-7}\) & 1.000 & 8.764\,s \\
MLX Adam (BC off)                          & \(9.258275\times 10^{-7}\) & 1.000 & 6.676\,s \\
\bottomrule
\end{tabular}
\caption{All reach perfect accuracy; dynamic \(\beta_2\) attains a substantially lower logistic loss at the same step budget.}
\label{tab:sanity2}
\end{table}

\subsection{Sanity‑3: utility maximization (concave).}
Finally, we flip the perspective and \emph{maximize} a concave logistic utility (equivalently, minimize its negation) for \(50{,}000\) steps with \(\rho=5\times10^{-2}\). This check ensures the method remains stable and unbiased in the opposite curvature regime. All optimizers converge to the same solution up to machine precision, with identical classification accuracy and vanishing loss/utility differences. As in Sanity‑1, MLX Adam tends to be faster per step on this tiny problem; the differences are not meaningful for our conclusions.

\begin{table}[H]
\centering
\small
\begin{tabular}{lcccc}
\toprule
Optimizer & Utility\(\uparrow\) (median) & Loss (median) & Acc (median) & Time (median) \\
\midrule
K‑\(\beta\) (dynamic \(\beta_2\), BC off)  & \(-5.868160\times 10^{-9}\) & \(5.868160\times 10^{-9}\) & 1.000 & 23.363\,s \\
K‑\(\beta\) (fixed \(\beta_2\), BC off)    & \(-5.834350\times 10^{-9}\) & \(5.834350\times 10^{-9}\) & 1.000 & 23.077\,s \\
MLX Adam (BC off)                          & \(-5.823792\times 10^{-9}\) & \(5.823792\times 10^{-9}\) & 1.000 & 17.659\,s \\
\bottomrule
\end{tabular}
\caption{All methods converge to the machine‑precision optimum (differences are negligible). Timing on these toy problems is secondary to the main results section.}
\label{tab:sanity3}
\end{table}

\medskip
\noindent\emph{Summary.} Across convex, classification, and concave toys, Kourkoutas‑\(\beta\) behaves as intended: it collapses to Adam when configured with a fixed \(\beta_2\) and no extras, and it offers a clear advantage in the presence of early gradient spikes by adaptively reducing \(\beta_2\). The empirical patterns here align with the layer‑wise sunspike logic laid out in the Overview; we defer speed comparisons to the PDE experiments, where compute and data movement dominate over micro‑kernel effects.


\section{Experiments}
\label{sec:experiments}

We evaluate Kourkoutas‑$\beta$ on four testbeds designed to stress optimizers under large, bursty gradients: (i) a data‑driven Transformer PDE surrogate (Heat2D), (ii) a 3D cylindrical PINN (Heat3D), (iii) a synthetic “Length‑Jitter + Rare Trigger” MLX toy (Testbed C) that induces intermittent spikes via a 1\% trigger under variable sequence lengths, and (iv) a character-level Transformer on a 30 MB
slice of \texttt{enwik8} (\texttt{small–enwik8}). The PDE implementations are released as reusable MLX reference code \cite{kbeta_software}, \cite{pinn3d_software}, \cite{transformer2d_software} (see companion GitHub repositories \repoPINN{} and \repoTrans{}), while the MLX toy script of Testbed~C is included verbatim in \ref{app:trigger-code}. The character-level language Transformer is given on the \repoKbeta{} repo under ``Examples''.
\smallskip
 
Across all testbeds we freeze the optimizer semantics of §\ref{sec:overview}–§\ref{sec:practice}, meaning a layer‑wise sunspike signal 
drives a dynamic $\beta_2 \in \left[ \beta_{2,\min}, \beta_{2,\max} \right]$.  
The denominator optionally uses (leaky) AMSGrad. Trust‑region clipping and adaptive tiny are off unless explicitly stated and from here on we use bias correction on for all methods unless otherwise stated.

\paragraph{Setup and protocol}
Unless noted otherwise, we keep the following \emph{identical} across optimizers:
the step size $\rho$ (same schedule), $\beta_1$, batch size, number of steps/epochs,
random seeds, and $\varepsilon=10^{-8}$. Our primary baselines are \textbf{MLX Adam}
(with standard bias correction on) at $\beta_2=0.999$ (\textbf{Adam--999}) and
$\beta_2=0.95$ (\textbf{Adam--95}). While $\beta_2=0.999$ is the common default,
$\beta_2\approx0.95$ is sometimes preferred in PDE‑centric workloads, so we report both.
In all main experiments, \textbf{Kourkoutas--$\beta$} uses
\texttt{bias\_correction="beta2max"} (i.e., the denominator uses $b_{2,t}=1-\beta_{2,\max}^{\,t}$),
with other knobs (\texttt{decay}, \texttt{max\_ratio}, \texttt{adaptive\_tiny},
\texttt{warmup\_steps}) fixed per testbed and stated where used.

\paragraph{Statistical Tools and Analysis}
For the sake of transparency and reproducibility, the details of the statistical analysis applied to each of the four Testbeds is detailed in \ref{app:stat}.
 
\paragraph{Compute environment}
\label{sec:setup}
Unless otherwise noted, all experiments, including the sanity checks of \S\ref{sec:sanity}, were run on a single Apple Studio M2 Ultra machine with 198\,GB unified memory (``M2U''). Timings reported are wall‑clock.

\paragraph{Bucketization choices (exactly as implemented)}
Layers are formed by a user‑supplied \texttt{layer\_key\_fn}. We use the following
mapping per testbed, with a consistent naming convention throughout:
\textbf{Testbed~A (Heat2D)} — stable parameter‑path key (fine, per‑parameter buckets);
\textbf{Testbed~B (Heat3D)} — \lstinline|lambda p: p.shape| (coarse, shape‑based buckets);
\textbf{Testbed~C (Length--Jitter + Rare Trigger)} — single global bucket \lstinline|lambda _: 0|;
\textbf{Testbed~D (small‑enwik8)} — stable parameter‑path key (same as Testbed~A).
If \lstinline|layer_key_fn| is \lstinline|None|, the implementation falls back to a stable name
(\lstinline|p.name|, if present) or to identity \lstinline|id(p)|. Coarser buckets yield a smoother,
more data‑efficient $\beta_2$ signal; finer buckets localize the adaptation.
Table~\ref{tab:testbed-settings} summarizes the Kourkoutas-$\beta$ settings used in each of the four testbeds. 

\begin{table}[H]
\centering
\caption{Key per‑testbed settings. Items not shown (e.g., $\rho$, $\beta_1$, batch size) are identical across methods. “Bucketization” denotes the mapping used to pool parameters into buckets $\ell$ for the sunspike and $\beta_2$ statistics.}
\label{tab:testbed-settings}
\setlength{\tabcolsep}{6pt}
\begin{tabular}{lcccc}
\toprule
 & \makecell{\textbf{Testbed~A}\\\textbf{Transformer}\\\textbf{(Heat2D)}} & \makecell{\textbf{Testbed~B}\\\textbf{PINN}\\\textbf{(Heat3D)}} & \makecell{\textbf{Testbed~C}\\\textbf{Length--Jitter}\\\textbf{+ Rare Trigger}} 
 &  \makecell{\textbf{Testbed~D}\\\textbf{LM Transformer}\\\textbf{small‑enwik8}}\\
\midrule
\makecell[l]{Bucketization\\\lstinline|layer_key_fn=|}
& \makecell{per‑parameter (fine)\\ \lstinline|param_path|} 
& \makecell{shape‑based (coarse)\\ \lstinline|lambda p: p.shape|} 
& \makecell{single global\\ \lstinline|lambda _: 0|} 
& \makecell{per‑parameter (fine)\\ \lstinline|param_path|} \\
Warmup steps 
& $\sim$350 & 0 & 50 & 250 \\
Adaptive tiny 
& off & on/off as specified & off & off \\
(leaky)AMSGrad 
& off (default) & on (e.g.\ \texttt{decay}$=0.98$) & off & off \\
Trust‑region clip 
& off (default) & on (\texttt{max\_ratio}$=3$) & off & off \\
Bias correction 
& \texttt{"beta2max"} & \texttt{"beta2max"} & \texttt{"beta2max"} & \texttt{"beta2max"} \\
\bottomrule
\end{tabular}

\vspace{0.2ex}
\footnotesize\emph{Alias mapping (CLI \(\rightarrow\) paper).} \texttt{--layer\_bucket per-array} \(\equiv\) \emph{per‑parameter (stable path)}; 
\texttt{shape} \(\equiv\) shape‑based; \texttt{global} \(\equiv\) single global bucket; 
\texttt{id}/\texttt{auto} \(\equiv\) per‑object identity (not recommended across runs because it is not stable). 
When $\beta_2$ is fixed, the bucketization choice does \emph{not} affect the update (it reduces exactly to Adam).
\end{table}


\subsection{{\bf Testbed A:} Data-Driven Transformer PDE Surrogate (Heat2D): 30-seed study}
\label{sec:transformer}

\subsubsection{Problem physics}
We follow the Heat2D benchmark of Kassinos \& Alexiadis \cite{kassinos2025Beyond}, where a compact Transformer is trained to predict the temperature field \(T(x,y,t)\) on a \(26\times 26\) grid over \(401\) time steps. Each sample in the training set has a randomly selected uniform initial temperature distribution in the interior and four distinct and randomly selected Dirichlet boundary conditions. While the heat diffusivity varies across cases in the dataset, it remains constant during the evolution of the temperature field. The non-dimensional thermal diffusivity  \( \beta \) is the ratio of the physical diffusivity to a reference diffusivity. The non-dimensionalized two-dimensional heat conduction equation is 

\begin{equation}
\frac{\partial \theta}{\partial \tau} = \beta \left( \frac{\partial^2 \theta}{\partial \xi^2} + \frac{\partial^2 \theta}{\partial \eta^2} \right)
\end{equation}

\noindent where \( \theta = \frac{T - T_{\text{ref}}}{T_0} \) is the non-dimensional temperature, \( \tau \) is the non-dimensional time, and \( \xi \), \( \eta \) are the non-dimensional spatial coordinates corresponding to \( x \) and \( y \).
The spatial domain, a rectangular plate, is discretized into a uniform grid with \( N_{\xi} = 26\) and \( N_{\eta}=26 \) nodes along the \( \xi \) and \( \eta \) directions, respectively. The grid spacing is \( \Delta \xi = \frac{1}{N_{\xi}-1} \) and \( \Delta \eta = \frac{1}{N_{\eta}-1} \).

At training time the model receives the first five frames together with initial and boundary–condition tokens and is asked to predict the remaining frames, either in an autoregressive or a block prediction mode. The benchmark is deterministic given the boundary/initial conditions but produces \emph{bursty} gradients during early training due to the random variability of boundary and initial conditions from sample to sample and from long‑horizon error propagation.

\subsubsection{Transformer implementation and code features}
We use the current MLX implementation (Apple Silicon) with:
(i) per‑layer INT8 weight quantization; 
(ii) Rotary positional encodings (RoPE);
(iii) stable parameter–path strings for bucketization (\lstinline|layer_key_fn=param_path|) so that Kourkoutas–\(\beta\) maintains separate \(\beta_2\) tracks per module; 
(iv) a \(350\)-step warm‑up for Kourkoutas–\(\beta\) during which \(\beta_2\) is held at \(\tfrac{1}{2}(\beta_{2,\min}+\beta_{2,\max})\) and the sunspike statistic is quenched (no dips).%
\footnote{Warm‑up only affects Kourkoutas–\(\beta\)'s internal statistics; the optimizer step size/schedule is identical across methods.}
All runs share the same model, batch size, and learning‑rate schedule; only the optimizer differs.

\subsubsection{Why this test highlights differences among optimizers}
Two design choices are intentionally stress‑inducing for second‑moment methods:
(1) \textbf{Small‑data regime.} In Kassinos \& Alexiadis \cite{kassinos2025Beyond}, the training of the Transformer for the particular configuration was done with a dataset of \(12{,}000\) samples
split into a training (\(8{,}400\)), validation (\(2{,}400\)) and test sets (\(1{,}200\)) samples. Here, we train on just \(4{,}000\) samples, split \(2{,}800/800/400\) (train/val/test), which amplifies gradient noise and rare spikes. 
(2) \textbf{Quantization.} We introduce per‑layer INT8 quantization, which sharpens non‑smoothness during early training.
Such settings strongly favor an optimizer that can shorten its second‑moment memory window on spikes and lengthen it when training calms—precisely the behavior Kourkoutas–\(\beta\) delivers.

\subsubsection{Metrics and reporting}
We run \(100\) epochs and report the \emph{training} MSE at epoch \(100\) (lower is better). 
To obtain matched, paired observations we repeat each optimizer over \(30\) seeds \(\{0,\dots,9,12,\dots,31\}\) using identical data shuffles, initializations, and schedules.
We summarize performance by mean\(\pm\)sd and median\([\mathrm{IQR}]\), and we assess significance using paired differences per seed: (i) a paired \(t\)-test with 95\% CIs, and (ii) a Wilcoxon signed‑rank test.
We also report wall‑clock minutes/epoch (median\([\mathrm{IQR}]\)).
Unless noted, Kourkoutas–\(\beta\) uses \(\beta_{2,\min}=0.88\), \(\beta_{2,\max}=0.999\), \(\beta_1=0.9\), \(\alpha_{\text{EMA}}=0.93\), \(\varepsilon=10^{-8}\), \texttt{layer\_key\_fn=param\_path}.
Adam baselines use the same schedule and \(\beta_1=0.9\) with \(\beta_2\in\{0.95,0.999\}\).

\subsubsection{Results}
\label{sec:Trans_Results}

We train the Transformer over 100 epochs, using a hand–coded, piecewise–constant schedule specified as an epoch$\!\to$LR map. Let $e$ denote the epoch index; then
\[
\rho(e)=
\begin{cases}
1.0\times 10^{-3}, & 5 \le e < 30,\\
5.0\times 10^{-4}, & 30 \le e < 40,\\
1.0\times 10^{-4}, & 40 \le e < 60,\\
1.0\times 10^{-5}, & 60 \le e \le 100~,
\end{cases}
\]
i.e., the LR is updated at the \emph{start} of the listed epochs and held constant in between. Training for this benchmark stops at $e{=}100$. The schedule is applied \emph{identically} to all optimizers and seeds. Note that the 350‑step warm‑up discussed for Kourkoutas–$\beta$ affects only the dynamic $\beta_2$ signal and is independent of the LR schedule.

Table~\ref{tab:transformer-abs} shows losses at $e{=}100$ and Table~\ref{tab:transformer-stats-3tests} gives paired statistics. 
Across \(30\) matched seeds, Kourkoutas–\(\beta\) improves the \emph{mean} loss by \textbf{12.8\%} over Adam‑95 and \textbf{39.4\%} over Adam‑999; using \emph{medians}, the improvements are \textbf{11.6\%} and \textbf{37.8\%}. 
It beats Adam‑95 on \(24/30\) seeds and Adam‑999 on \(30/30\) seeds.
Effect sizes are large: vs.\ Adam‑95, \(d_z{=}0.768\) (paired \(r{=}0.615\)); vs.\ Adam‑999, \(d_z{=}2.038\) (paired \(r{=}0.901\)).

\begin{table}[h]
\centering
\caption{Epoch‑100 training MSE ($\downarrow$ is better). Values are mean$\pm$sd over 30 matched seeds ($\times 10^{-6}$). }
\label{tab:transformer-abs}
\begin{minipage}{0.49\linewidth}
\begin{tabular}{lcc}
\toprule
\textbf{optimizer} & \textbf{MSE (mean$\pm$sd)} & \textbf{Median [IQR]} \\
\midrule
K--$\beta$    & $\mathbf{2.203 \pm 0.312}$ & $2.229\ [0.312]$ \\
Adam--95      & $2.527 \pm 0.239$          & $2.522\ [0.421]$ \\
Adam--999     & $3.636 \pm 0.640$          & $3.584\ [0.569]$ \\
\bottomrule
\end{tabular}
\end{minipage}

\medskip
\footnotesize  The “Median [IQR]” column reports the across‑seed median and interquartile range.
\end{table}

\begin{table}[h]
\centering
\caption{Paired comparisons at epoch~100 (30 seeds). Differences are Adam $-$ K--$\beta$ ($\times 10^{-7}$).}
\label{tab:transformer-stats-3tests}
\setlength{\tabcolsep}{5pt}
\small
\begin{minipage}{0.49\linewidth}
\centering
\textbf{Parametric + effect sizes}\\[3pt]
\begin{tabular}{lcccc}
\toprule
\textbf{Comp.} & \textbf{Mean} & \textbf{95\% CI} & $\mathbf{t(29)}$ & $\mathbf{d_z}$; $\mathbf{r}$\\
\midrule
Adam--95  & 3.235  & [1.662, 4.808]   & 4.205  & 0.768; 0.615 \\
Adam--999 & 14.324 & [11.699, 16.949] & 11.160 & 2.038; 0.901 \\
\bottomrule
\end{tabular}
\end{minipage}\hfill
\begin{minipage}{0.49\linewidth}
\centering
\textbf{Distribution‑free tests}\\[3pt]
\begin{tabular}{lccc}
\toprule
\textbf{Comp.} & \textbf{Sign $+$/n (p)} & $\mathbf{W^{+}}$ & $\mathbf{p_{\text{Wilcoxon}}}$ \\
\midrule
Adam--95  & 24/30 (1.43$\times$10$^{-3}$) & 412 & 7.91$\times$10$^{-5}$ \\
Adam--999 & 30/30 (1.86$\times$10$^{-9}$) & 465 & 1.86$\times$10$^{-9}$ \\
\bottomrule
\end{tabular}
\end{minipage}

\medskip
\footnotesize Effect sizes: $d_z{=}t/\sqrt{n}$ (paired‑samples Cohen’s $d$); $r{=}\sqrt{t^2/(t^2+\mathrm{df})}$.
\end{table}

\noindent With diagnostics enabled (used to collect \(\beta_2\)/sunspike histories) Kourkoutas–\(\beta\) runs at \(\,1.525\,[1.506,1.542]\) min/epoch, versus \(1.356\,[1.344,1.382]\) (Adam‑95) and \(1.340\,[1.336,1.348]\) (Adam‑999).
As shown in Table~\ref{tab:transformer-time}, re‑runs with diagnostics \emph{off} give \(1.29\text{–}1.33\) min/epoch for Kourkoutas–\(\beta\) on multiple seeds, indicating parity or a slight speed advantage.

\begin{table}[H]
\centering
\caption{Median time per epoch (minutes/epoch). K‑$\beta$ (diag.~off) is from six spot‑checks; Adam rows are medians over 30 seeds.}
\label{tab:transformer-time}
\begin{tabular}{lcc}
\toprule
\textbf{optimizer} & \textbf{Median time/epoch} \\
\midrule
K--$\beta$ (diag.\ off) & $1.318$ \\
Adam--95                         & $1.356$ \\
Adam--999                        & $1.340$ \\
\bottomrule
\end{tabular}
\end{table}

In summary, the results of Tables~\ref{tab:transformer-abs}–\ref{tab:transformer-time} support the following overall assessment:
\begin{itemize}
    \item \textbf{K--$\beta$ vs Adam--999.} The advantage is large, universal across seeds (30/30), and extremely significant statistically. This strongly supports the idea that spiky early training benefits from dynamic $\beta_2$.
    \item \textbf{K--$\beta$ vs Adam--95.} Average $\sim 12.8\%$ lower mean loss (median gap $\sim 11.6\%$) with moderate-to-large effect size and strong significance. There are a few seeds where Adam--95 slightly wins, but the overall paired evidence (t‑test + sign test + Wilcoxon) favors Kourkoutas-$\beta$.
    \item \textbf{Stability.} Across-seed variability (sd or CV) is lower than Adam--999 and a bit higher than Adam--95, i.e., Kourkoutas-$\beta$ is more consistent than high-memory Adam, and has roughly comparable stability to an aggressively tuned Adam--95.
    \item \textbf{Speed.} With diagnostics off, Kourkoutas-$\beta$ is at least as fast or slightly faster; with diagnostics on (as in these runs), it carries the expected $\sim 12\%$ overhead from histograms/trackers.
\end{itemize}

\subsubsection{Visual impression of the physical significance of achieved loss at epoch 100} 

Figure~\ref{fig:heat2d-seed0-loss} shows the evolution of training loss during training for the case of seed=0. Adam--999 produces a more unstable training with significant jitter in the training loss and tends to plateau earlier than the other two optimizers during each phase of the learning rate schedule. Adam--95 and Kourkoutas-$\beta$ produce a more stable training and track each other closely during the early epochs, but eventually  K-$\beta$ overtakes both Adam--95 and Adam--999 producing a final loss at epoch=100 that is roughly 40\% lower. 

\begin{figure}[h]
  \centering
  \includegraphics[width=0.75\linewidth, clip,trim=5pt 5pt 5pt 5pt ]{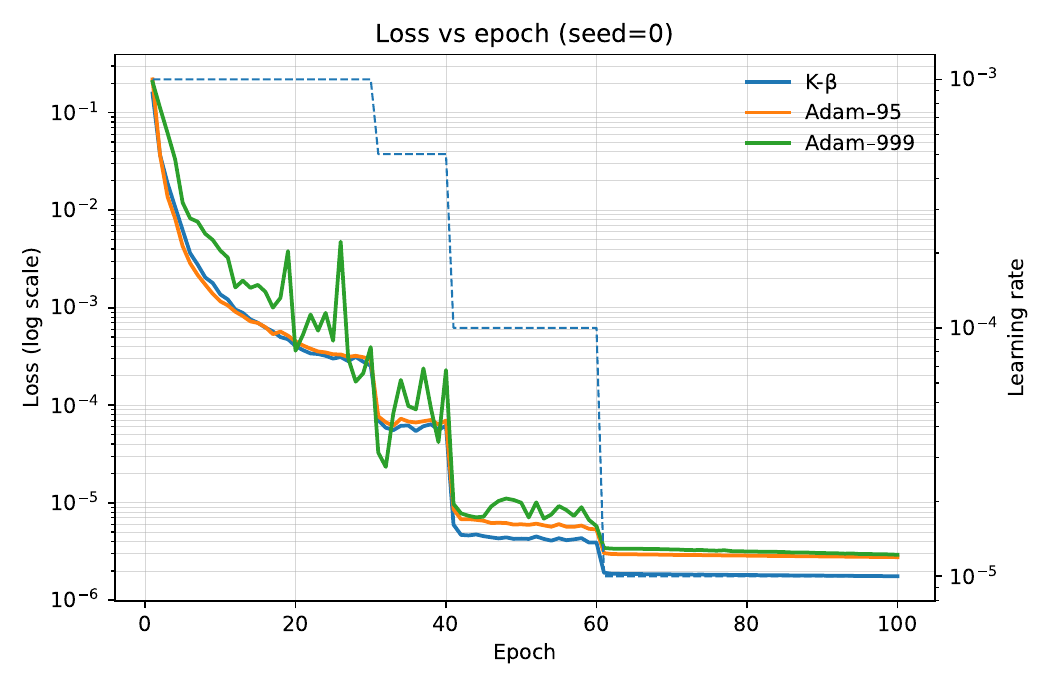}
  \caption{Training MSE vs.\ epoch for the Transformer (Heat2D), seed=0.
  The dashed curve overlays the learning–rate schedule used identically for all optimizers:
  $\rho(e)=10^{-3}$ for $5\le e<30$, $5\!\times\!10^{-4}$ for $30\le e<40$,
  $10^{-4}$ for $40\le e<60$, and $10^{-5}$ for $60\le e\le100$ (schedule defined in §\ref{sec:Trans_Results}).}
  \label{fig:heat2d-seed0-loss}
\end{figure}

\paragraph{Interpretation of the loss curves}
The persistent jitter in Adam--999, even late in training, arises from its very long second-moment memory: with $\beta_2=0.999$, the time constant is $\approx 1/(1-\beta_2) \approx 1000$ steps (half-life $\approx 693$ steps). This long‐lived EMA causes the denominator to lag behind rapid shifts in gradient scale, so the effective step size chases a moving target and “chatters.” Kourkoutas-$\beta$ produces smoother, steadier descent because $\beta_{2,t}$ typically hovers around $0.93$–$0.96$ in calm phases (memory $\approx 15$–$35$ steps) and dips toward $\beta_{2,\min}$ on gradient spikes. This keeps the second-moment estimate aligned with the current regime rather than the distant past.
Learning-rate drops at epochs $e=30,40,60$ benefit Kourkoutas-$\beta$ markedly: its memory shortens at the transition, allowing rapid re-estimation of $v_t$, producing the visible downward kinks in the loss. Adam--999, in contrast, requires $O(10^3)$ steps to “forget” the pre-transition scale, stalling or overshooting. Adam--95 avoids jitter but also cannot adapt mid-training—the fixed $\beta_2$ prevents exploiting the LR schedule.
In the noise-dominated late stage, Adam--999’s long memory keeps injecting stale variance, flattening progress. Kourkoutas-$\beta$’s moderate memory and occasional dips suppress this residual noise, allowing the loss to keep edging downward instead of rattling sideways.

Figure~\ref{fig:heat2d-contrast} shows predictions by the Transformer model trained by Kourkoutas-$\beta$ and Adam--999 using seed=0, for which the final losses achieved at epoch=100 were 1.764$\times 10^{-6}$ and 2.930$\times 10^{-6}$ respectively. The leftmost panel shows the ground-truth temperature distribution for a particular instance during the temperature field evolution (step=6). The middle and rightmost panels shows the predictions of the Kourkoutas-$\beta$ and the Adam--999 models for the same step. The bottom row shows the same images but renormalized with contrast equalization to enhance clarity.
  
\begin{figure}[h]
  \centering
  \includegraphics[width=\linewidth]{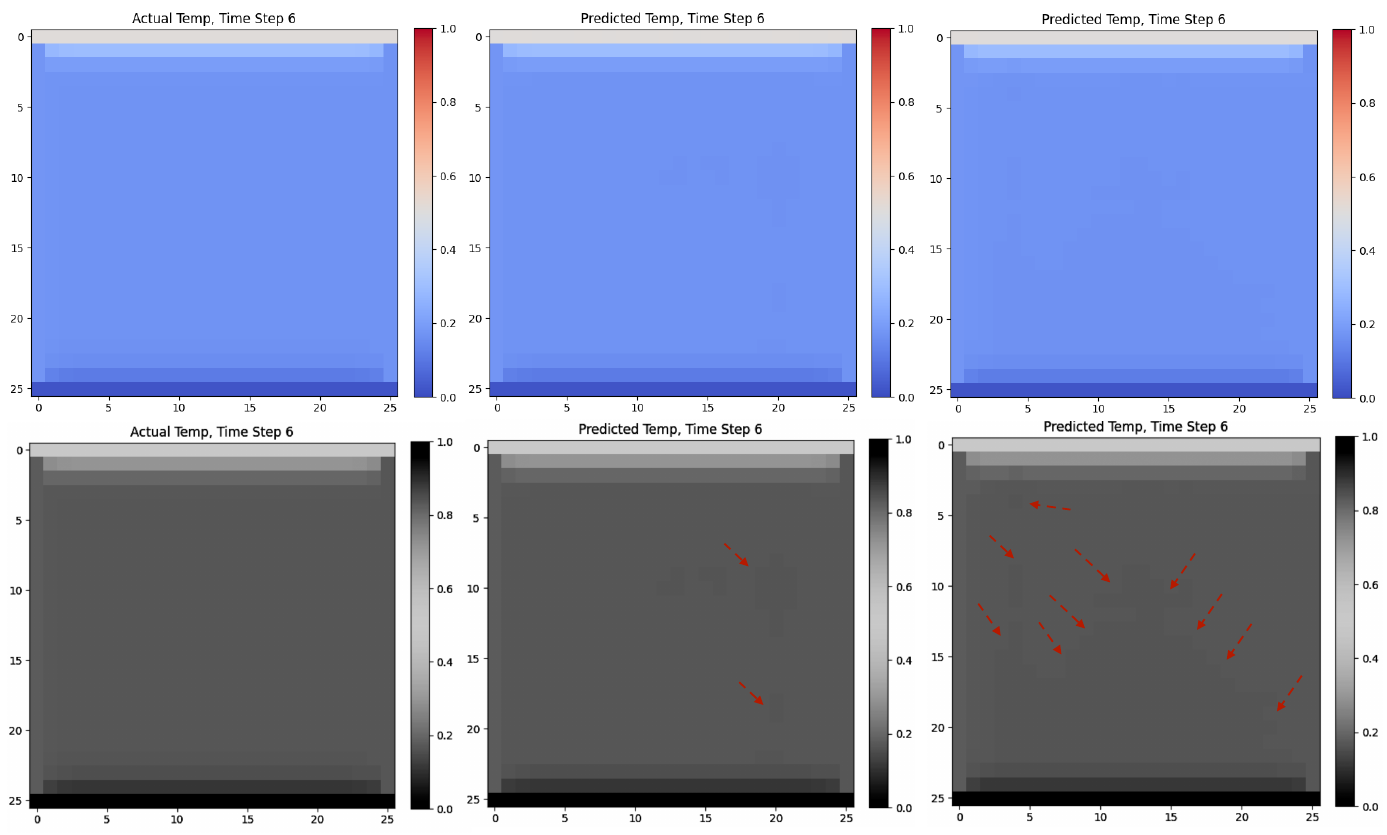}
  \caption{\textbf{Heat2D Transformer – visual impact of small MSE differences.}
  One test case at time step $t{=}6$: \emph{left} ground truth, \emph{middle} prediction from the Kourkoutas-$\beta$ (K-$\beta$) model, \emph{right} prediction from Adam-999. 
  Top row shows raw temperatures with a shared physical color scale; bottom row shows the same images linearly re‑normalized per panel to enhance interior deviations (contrast equalization only, no filtering). 
  Although the top row appears nearly identical, the contrast‑equalized bottom row reveals structured interior ``shadowing'' in the Adam-999 prediction that is largely absent with K‑$\beta$. 
  Final validation MSE at epoch 100: K‑$\beta$ \(1.76{\times}10^{-6}\) vs. Adam-999 \(2.93{\times}10^{-6}\), a \textbf{39.8\%} reduction (Adam is \(1.66{\times}\) higher).}
  \label{fig:heat2d-contrast}
\end{figure}
\medskip

\noindent Even when the absolute error is already in the 10$^{-6}$ range, the $\sim$ 40\% reduction from Kourkoutas-$\beta$ is not just numerology: after contrast equalization, the Adam--999 model exhibits coherent interior artifacts (“shadows”) that are largely suppressed in the Kourkoutas‑$\beta$ result. This is consistent with our hypothesis that dynamic $\beta_2$ dampens the optimizer’s sensitivity to bursty gradients arising from sample‑to‑sample shifts (boundary/initial condition changes), thereby reducing spurious interior structure while keeping the boundary layers intact.

\subsubsection{Dynamics of $\beta_2$ and sunspike}

\paragraph{Reading the sunspike and \(\beta_2\) violins}
We use violin and heatmap plots to visualize how the \emph{sunspike} ratio and the
resulting \(\beta_2\) evolve through training. Recall that \(\mathrm{sun}\in[0,1)\) is a bounded,
layer–wise signal,
\[
\mathrm{sun}_t^{(\ell)} \;=\; \frac{\mathrm{raw}_t^{(\ell)}}{1+\mathrm{raw}_t^{(\ell)}},
\qquad
\mathrm{raw}_t^{(\ell)} \;=\; \frac{\bigl\lVert g_t^{(\ell)}\bigr\rVert}
{\,r_t^{(\ell)}+\varepsilon_{\text{spike}}\,},
\]
so it measures how large the current gradient norm is relative to its recent EMA.
High sunspike (closer to \(1\)) means a genuine spike, i.e. \(\lVert g_t^{(\ell)}\rVert\) greatly
exceeds its history, while low sunspike (near \(0\)) means the step is mild or typical. 
In short, \texttt{sunspike} is an online measure of how ``spiky'' the current gradient is. Thus, the distribution plots offer a quick view of whether each epoch was dominated by mild vs.\ bursty updates.
Kourkoutas–\(\beta\) converts the sunspike signal directly into the second–moment discount via
\[
\beta_{2,t}^{(\ell)} \;=\;
\beta_{2,\max} - \bigl(\beta_{2,\max}-\beta_{2,\min}\bigr)\,\mathrm{sun}_t^{(\ell)}.
\]
Thus epochs with mass near \(\mathrm{sun}\!\approx\!1\) correspond to noticeably lower \(\beta_2\)
(more agile adaptation), while mass near \(\mathrm{sun}\!\approx\!0\) keeps
\(\beta_2\) close to \(\beta_{2,\max}\) (strong smoothing). Each violin summarizes the
distribution across layers at that epoch; the companion \(\beta_2\) violins are the
\emph{image} of the sunspike violins under the linear map above. A dashed line at
\(\beta_2{=}0.999\) marks the fixed‐\(\beta_2\) Adam reference.

Figure~\ref{fig:heat2d-dynamics-seed0} visualizes Kourkoutas--$\beta$ on the Transformer (seed=0).
After the short warmup (sunspike held at zero), the distribution concentrates around
$\texttt{sun}\!\approx\!0.45\text{--}0.55$, which maps linearly to $\beta_2\!\approx\!0.93\text{--}0.96$
($\beta_{2,\min}\!=\!0.88$, $\beta_{2,\max}\!=\!0.999$). Early epochs show broader violins (occasional
dips toward~0.92 and excursions near~0.97), then the spread narrows after $\sim$60 epochs, indicating
that the optimizer “acts like” a well‑tuned Adam‑0.95 on average while retaining the ability to react
to sporadic spikes—something a fixed $\beta_2$ cannot do.

\begin{figure}[H]
  \centering
  \captionsetup[subfigure]{justification=centering}
  \begin{subfigure}[t]{0.46\linewidth}
    \includegraphics[width=\linewidth, clip,trim=5pt 5pt 5pt 5pt]{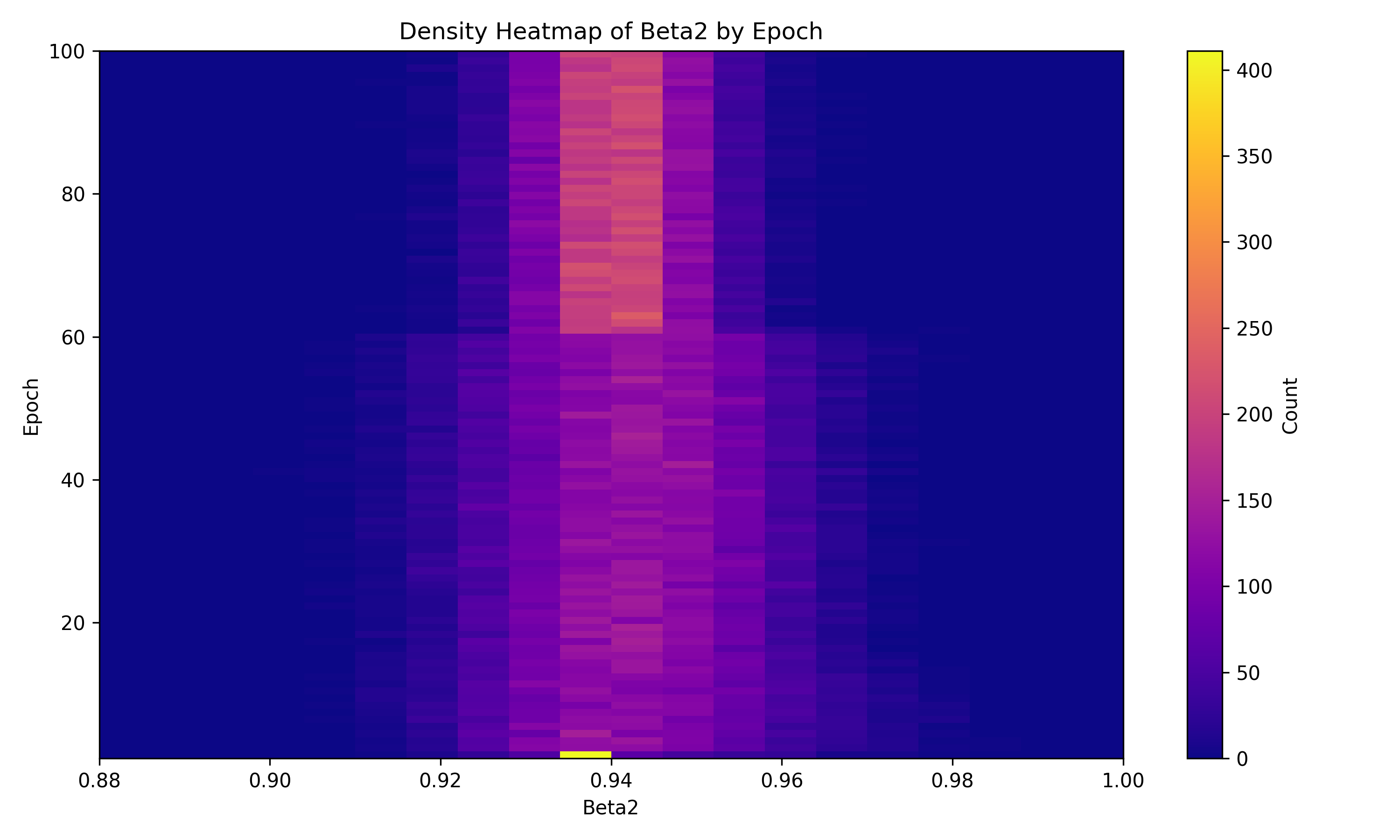}
    \caption{$\beta_2$ density by epoch}
  \end{subfigure}\hfill
  \begin{subfigure}[t]{0.46\linewidth}
    \includegraphics[width=\linewidth, clip,trim=5pt 5pt 5pt 5pt]{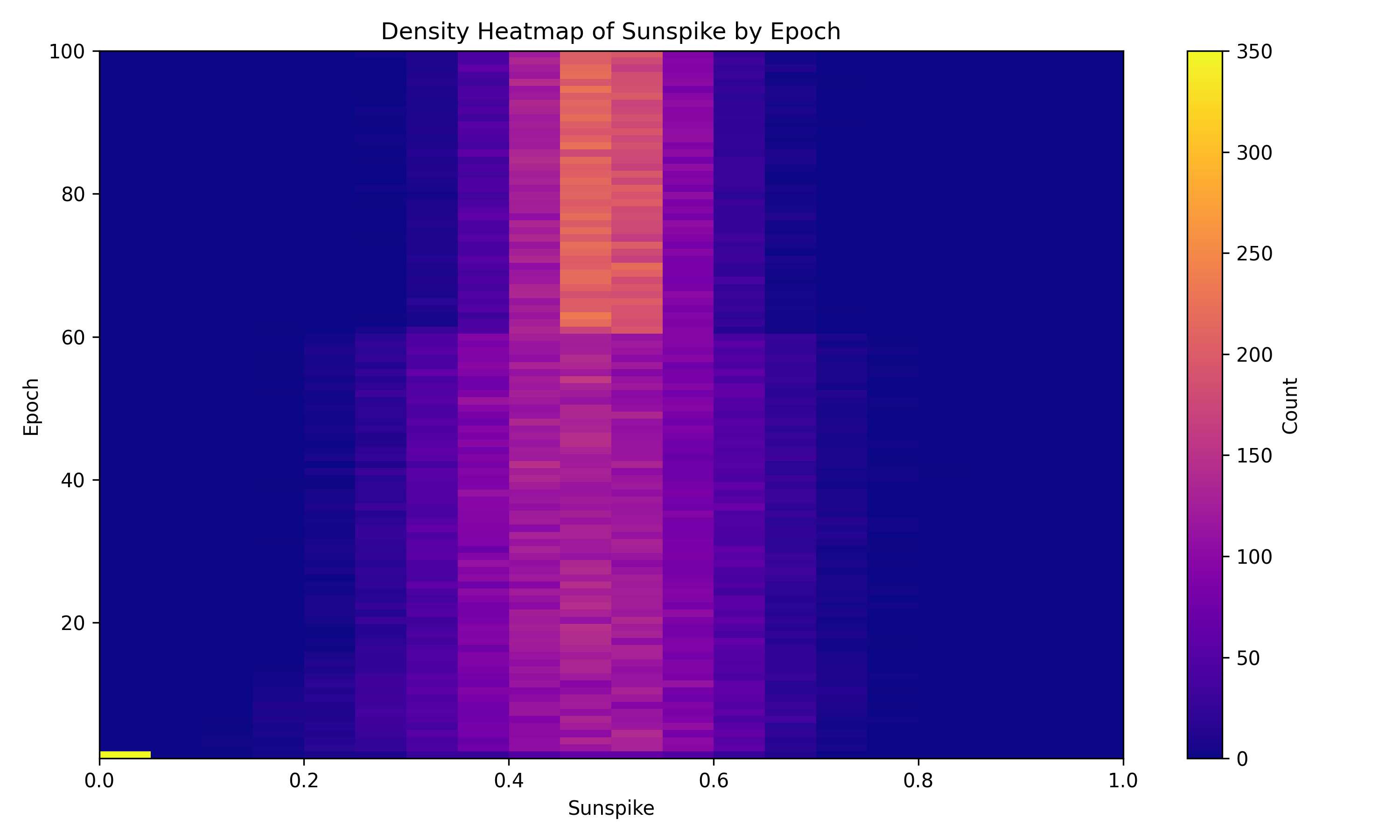}
    \caption{\texttt{sunspike} density by epoch}
  \end{subfigure}

  \vspace{0.4em}

  \begin{subfigure}[t]{0.46\linewidth}
    \includegraphics[width=\linewidth, clip,trim=5pt 5pt 5pt 5pt]{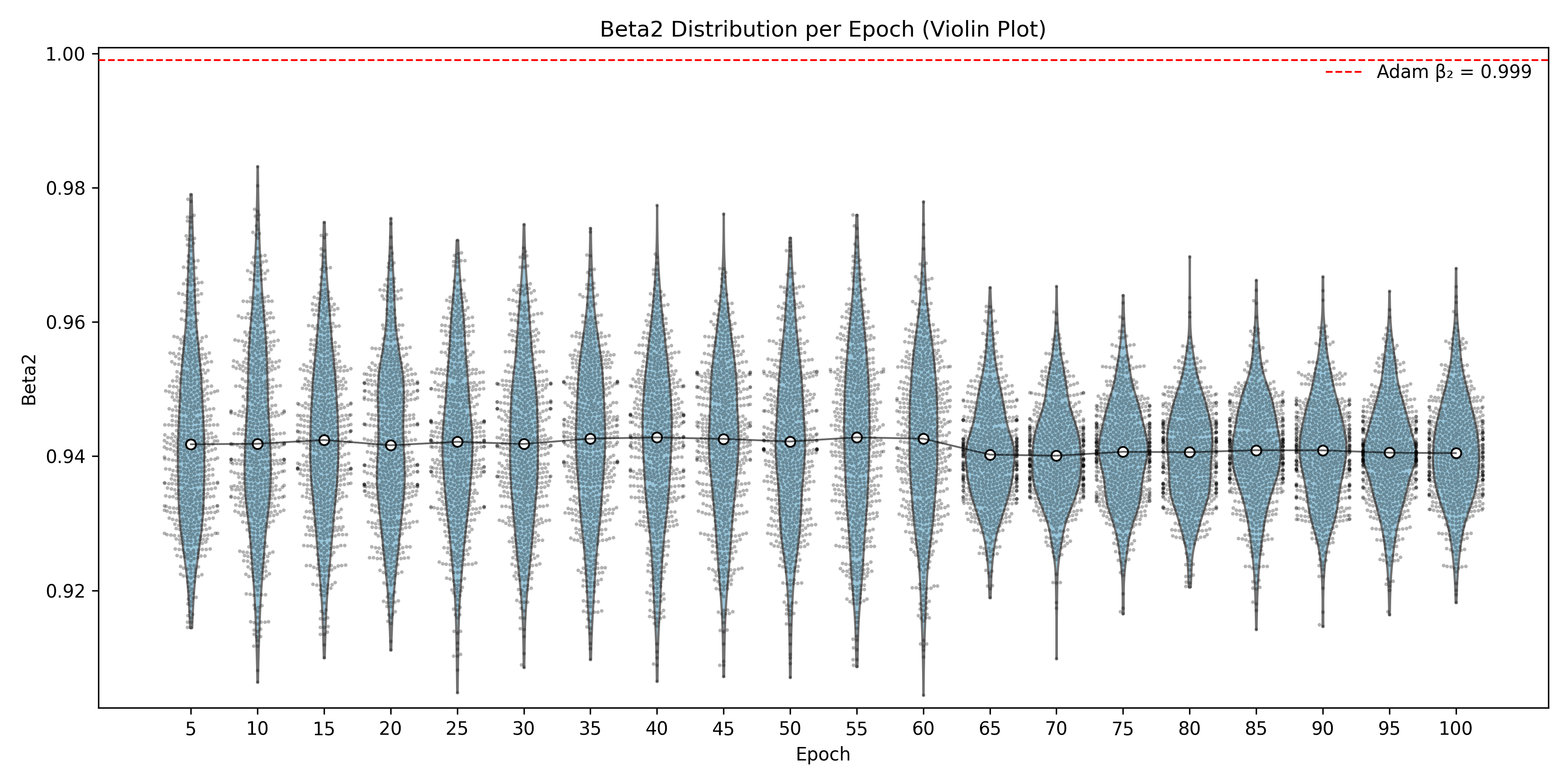}
    \caption{$\beta_2$ violins by epoch (white dots: medians; dashed line: Adam $\beta_2{=}0.999$)}
  \end{subfigure}\hfill
  \begin{subfigure}[t]{0.456\linewidth}
    \includegraphics[width=\linewidth, clip,trim=5pt 5pt 5pt 5pt]{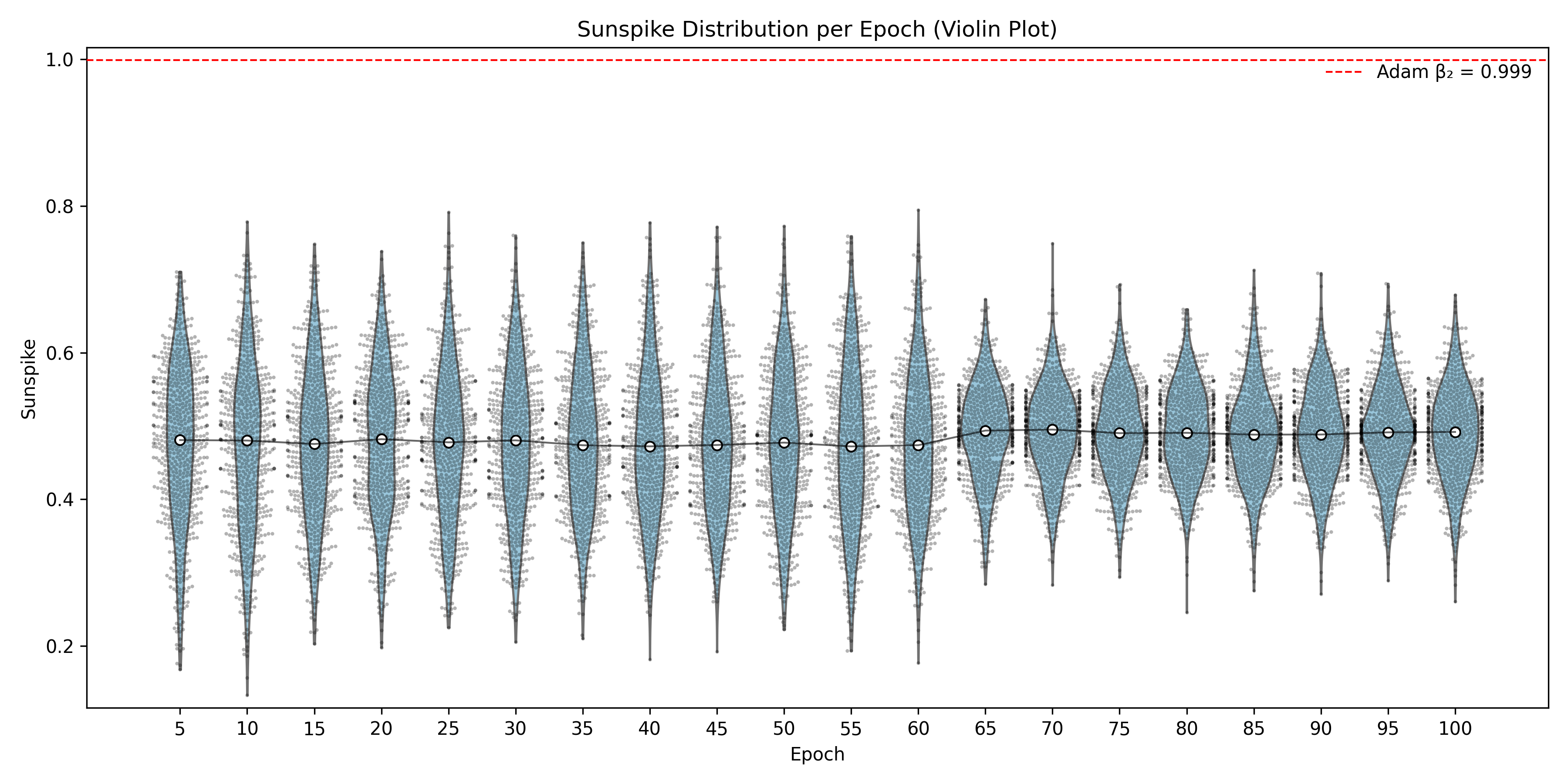}
    \caption{\texttt{sunspike} violins by epoch (white dots: medians)}
  \end{subfigure}

  \caption{\textbf{K--$\beta$ dynamics on the Transformer (Heat2D), seed=0.}
  Bucketing uses stable parameter paths (fine granularity).
  Warmup holds \texttt{sun}=0 and fixes $\beta_2$ to $\tfrac12(\beta_{2,\min}+\beta_{2,\max})\approx0.9395$,
  producing the near‑horizontal band at epoch~1 in the heatmaps.
  After release, \texttt{sunspike} concentrates around $0.45$–$0.55$ and the induced
  $\beta_2$ mass sits in $0.93$–$0.96$, rarely visiting the Adam‑like extreme $\beta_2\!=\!0.999$
  (dashed line). Early‑epoch violins are tall (reactive dips during spikes); by $\sim$epoch~60 the
  distributions tighten, indicating stable training while preserving occasional adaptive kicks.}
  \label{fig:heat2d-dynamics-seed0}
\end{figure}

\paragraph{\textbf{Takeaways}}
(1) Kourkoutas–\(\beta\) consistently outperforms both fixed‑\(\beta_2\) baselines on this quantized, small-data Transformer, with large  \emph{mean} loss improvment gaps (\(+12.8\%\) vs Adam‑95, \(+39.4\%\) vs Adam‑999). 
(2) The paired CIs exclude zero for both baselines, and all \(30\) paired differences against Adam‑999 are positive. 
(3) Runtime with diagnostics disabled is at least on par with Adam, so the accuracy gains do \emph{not} come at a speed cost.

Thus, the layer-wise sunspike adaptation delivers a systematic improvement over both fixed-$\beta_2$ baselines,  without costing runtime when diagnostics are off. 
This aligns with the PINN results discussed next: early, spiky gradients benefit from transiently shorter second-moment memory, while late epochs naturally settle near a $\beta_2\approx 0.94$ band without hand-tuning.

\medskip
\noindent\emph{Reproducibility.} Code is in the companion repository \repoTrans; an environment file pins package versions. The exact seeds, logs, and per‑epoch diagnostics (optional) are included in the artifact bundle.


\subsection{{\bf Testbed B:} 3D Cylindrical PINN (Heat3D)}
\label{sec:PINN}

We train a PINN to satisfy the heat equation in a cylindrical domain with periodicity and mixed boundary terms. The composite loss (residual{+}boundary/periodic terms) is stiff and tends to produce large gradient bursts. We use coarse bucketization (\lstinline|p.shape|), which yields a smooth, data‑efficient \(\beta_2\) signal with minimal overhead. 

All implementation details (geometry sampling, loss decomposition, training harness, LR schedule, and logging utilities) are available in the public GitHub repositories~\cite{kbeta_software},~\cite{pinn3d_software}: optimizer code in \repoKbeta{} and the PINN testbed in \repoPINN{} (repositories provide a pinned environment file to reproduce results exactly).

\subsubsection{Problem physics}
We consider the steady heat-diffusion (Laplace) equation in a cylindrical domain 
\(
  0 < r < r_{\text{out}}(\theta), \quad
  0 \le \theta \le 2\pi, \quad
  0 < z < L
\),
where \(r_{\text{out}}(\theta) = r_{\max} + 0.25\,r_{\max}\,\sin(3\theta)\)
defines an undulating outer boundary. The PDE is:
\[
  \nabla^2 T(r,\theta,z) \;=\; 0,
\]
written in cylindrical coordinates \((r,\theta,z)\) with
\[
  \nabla^2 T 
  \;=\;
  \frac{1}{r} \frac{\partial}{\partial r}
  \Bigl(r \,\frac{\partial T}{\partial r}\Bigr)
  \;+\;
  \frac{1}{r^2} \frac{\partial^2 T}{\partial \theta^2}
  \;+\;
  \frac{\partial^2 T}{\partial z^2}.
\]
\smallskip

\noindent {\it Boundary conditions:} 
\begin{enumerate}
  \item \emph{Inner cylinder:} \(r = r_{\min}\), \(T=1\).
  \item \emph{Inlet plane:} \(z=0\), \(T=1\).
  \item \emph{Outlet plane:} \(z=L\), \(\displaystyle \frac{\partial T}{\partial z}=0\) (insulated).
  \item \emph{Distorted outer boundary:} \(r=r_{\text{out}}(\theta)\) with a piecewise-specified flux \(\displaystyle q(z)\).  
        We impose \(\displaystyle \frac{\partial T}{\partial n} = -\,q(z)\) 
        on this outer boundary, where \(\partial/\partial n\) denotes the outward-normal derivative.
  \item \emph{\(\theta\)-periodicity:} 
        \(T(r,0,z) = T(r,2\pi,z)\), and \(\partial T/\partial\theta\bigl(r,0,z\bigr) = \partial T/\partial\theta\bigl(r,2\pi,z\bigr).\)
\end{enumerate}

\noindent Here, \(r_{\min} = 0.2\), \(r_{\max}=1.0\), and \(L=10\,r_{\max}\). We sample 
(\emph{i})~\textbf{interior points} to enforce \(\nabla^2 T=0\) (collocation), 
(\emph{ii})~\textbf{boundary points} for the Dirichlet/Neumann/flux conditions, 
and (\emph{iii})~\textbf{periodicity points} to tie \(\theta=0\) and \(\theta=2\pi\).

\subsubsection{PINN Implementation and Code Features}

We use a fully-connected MLP comprising 16 layers, each with 128 SiLU-activated neurons, to represent the network $T_{\theta}(r,\theta,z)$, which we train to satisfy:

\[
  \mathcal{L}_{\text{total}}
  =
  \underbrace{\mathcal{L}_{\text{PDE}}( \nabla^2 T_{\theta} )}_{\text{interior}}
  \;+\;
  \underbrace{\mathcal{L}_{\text{BCs}}( T_{\theta}, \partial T_{\theta}/\partial n )}_{\text{boundary}}
  \;+\;
  \underbrace{\mathcal{L}_{\text{periodic}}( T_{\theta} )}_{\theta\text{-periodicity}}.
\]
In code, we explicitly compute the cylindrical Laplacian and outward normal derivative 
to form a composite loss function. We also implement piecewise flux on the outer 
boundary. Data sampling occurs randomly in \(r,\theta,z\) (for interior points) 
and along each boundary region.

The code is a \emph{single-process}, Python-based script using the \texttt{mlx} 
framework; all PDE sampling, MLP modeling, and training loops happen in 
one file. Visualization routines (2D slices, 3D scatter, etc.) are optionally 
invoked at runtime.

\subsubsection{Why This Test Highlights Differences Among Optimizers}

Although Laplace’s equation is linear, the domain has:
\begin{itemize}
  \item \textbf{A wavy outer boundary} \(r_{\text{out}}(\theta)\) and piecewise flux 
        that create spatially-varying boundary conditions.
  \item \textbf{Aggressive learning-rate scheduling} and lower weighting on boundary terms, 
        which yields stiffer PDE gradients in early training.
\end{itemize}

Here we use a cosine decay learning rate schedule starting with an aggressive initial learning rate of $\rho_0=10^{-2}$  that is shared across 
all optimizers (see~\ref{app:schedules}, Listing~\ref{lst:lr-pinn3d}). Under these conditions, a conventional Adam optimizer with fixed 
\(\beta_2\) can get stuck in suboptimal minima or exhibit partial blow-up, 
whereas the proposed \emph{Kourkoutas-$\beta$} optimizer adapts its second-moment 
decay (\(\beta_2\)) to handle large gradient bursts, thus converging to 
a significantly lower final loss. This problem, therefore, serves as an 
effective stress test: in simpler PDE settings (milder domain or smaller LR), 
both optimizers converge similarly, but in this more demanding scenario, 
Kourkoutas$-\beta$ consistently outperforms both Adam--95 and Adam--999. It consistently converges in all
runs where the Adam variants fail and reaches significantly lower final loss in all cases.
As shown below, the advantage does not come from picking a
better constant; it comes from \emph{mobility and granularity}.

\paragraph{Metrics and reporting}
We report medians over repeated runs (seeds) and, when a success threshold $\tau$
is specified, success rates with exact Clopper–Pearson 95\% CIs. Thus, final‑loss medians
are computed over successful runs only. Timings are reported as \emph{per‑epoch} wall‑clock averages over the entire run
(100,000 epochs) with diagnostics disabled and
\emph{no} untimed warm‑up. This means that  any one‑time JIT/compile overhead is therefore included, but
is negligible at these horizons. 

\subsubsection{Results.}

All PINN--3D runs use the same two‑phase schedule: a cosine decay from
$\rho_0=10^{-2}$ to $\rho_{\min}=10^{-5}$ over the first $T_{\mathrm{ramp}}=40{,}000$
steps, followed by a constant plateau at $\rho_{\min}$. Formally,
\[
\rho_t \;=\;
\begin{cases}
\rho_{\min} + \tfrac{1}{2}\bigl(\rho_0 - \rho_{\min}\bigr)\,\bigl(1+\cos(\pi\,t/T_{\mathrm{ramp}})\bigr),
& 0 \le t \le T_{\mathrm{ramp}},\\[4pt]
\rho_{\min}, & t > T_{\mathrm{ramp}}.
\end{cases}
\]
We applied this same $\rho_t$ to all optimizers (Kourkoutas--$\beta$, Adam--0.95, Adam--0.999).
No untimed warm‑up was used; timings average over the full training run.

We evaluate Kourkoutas--\(\beta\) with fixed hyperparameters across seeds:
\(\beta_1{=}0.9\), \(\beta_{2,\max}{=}0.999\), \(\beta_{2,\min}{=}0.88\),
\(\alpha{=}0.93\), \texttt{tiny\_spike}{=}\(10^{-9}\), \texttt{tiny\_denom}{=}\(10^{-8}\),
\texttt{adaptive\_tiny}{=}\texttt{True}, \texttt{decay}{=}0.98, \texttt{max\_ratio}{=}3,
\texttt{warmup\_steps}{=}0, \(\varepsilon{=}10^{-8}\), \(\texttt{bias\_correction}{=}\)``beta2max''. We
compare 
against Adam--0.95 and Adam--0.999 using the MLX Adam implementation with  bias correction on and \(\varepsilon{=}10^{-8}\).
Ten seeds (0--9) are used with no per‑seed retuning.

\smallskip

\noindent A run is deemed a \emph{success} if the final loss at epoch $e=100$K satisfies
\(\text{loss} \le 9\times10^{-5}\). Table~\ref{tab:pinn3d_summary} shows the success/failure metrics and the median final loss for all optimizers. 
Kourkoutas-\(\beta\) reduces instabilities under aggressive schedules and training remains stable without extra tuning of \(\rho\) or \(\beta_1\), 
in line with the optimizer’s design goal of handling spiky gradients Kourkoutas-\(\beta\) . 
As a result, with fixed hyperparameters across seeds, Kourkoutas–$\beta$ succeeds on all 10 seeds and attains a markedly lower median final loss among successes ($1.66\times10^{-6}$).
 Adam--999 fails in all runs and Adam-95 succeeds in only one. 
\begin{table}[H]
\centering
\caption{PINN--3D (100k epochs, 10 seeds). Success means final loss $\le \tau$ (here $\tau=9\times10^{-5}$).
Medians for loss are computed over successful runs only. Hardware: Apple M2 Ultra (198\,GB RAM).}
\label{tab:pinn3d_summary}
\begin{tabular}{@{}lcccc@{}}
\toprule
\textbf{Optimizer} & \textbf{Success (\# / 10)} & \textbf{95\% CI (Clopper--Pearson)} & \textbf{Median final loss}  \\
\midrule
K--$\beta$ & \textbf{10 / 10} & \textbf{[69.2\%,\,100\%]} & \textbf{$1.66\times10^{-6}$}  \\
Adam ($\beta_2{=}0.95$) & 1 / 10 & [0.25\%,\,44.5\%] & n/a\textsuperscript{\dag}  \\
Adam ($\beta_2{=}0.999$) & 0 / 10 & [0\%,\,30.8\%] & n/a \\
\bottomrule
\end{tabular}

\smallskip
\raggedright
\footnotesize \textsuperscript{\dag}\,With only one successful run, a “median among successes” is not informative; we therefore omit it.
\end{table}

As shown in Table~\ref{tab:pinn3d-mcnemar}, the success advantage is statistically significant under paired tests on the per‑seed success indicator (McNemar's exact, two‑sided):

\begin{table}[h]
\centering
\caption{Paired significance on success/failure per seed (McNemar's exact, two‑sided).
\(b\): seeds where the row method succeeds and the column fails; \(c\): vice versa; \(n=b{+}c\).}
\label{tab:pinn3d-mcnemar}
\begin{tabular}{@{}lccc@{}}
\toprule
\textbf{Comparison} & \(b\) & \(c\) & \textbf{\(p\)-value} \\
\midrule
K--\(\beta\) vs Adam--0.999 & 10 & 0 & 0.00195 \\
K--\(\beta\) vs Adam--0.95  & \phantom{0}9 & 0 & 0.00391 \\
Adam--0.95 vs Adam--0.999            & \phantom{0}1 & 0 & 1.000 \\
\bottomrule
\end{tabular}
\end{table}

Kourkoutas-$\beta$ per‑epoch cost is comparable to Adam and slightly lower in median (94.5 vs.\ 99.9$\,$ms for Adam–0.95 and 96.3$\,$ms for Adam–0.999). All timings include any one‑time JIT/compile overhead (no untimed warm‑up) and were collected on an Apple M2 Ultra (198 GB RAM).

\begin{table}[h]
\centering
\caption{Per‑epoch wall‑clock time (ms/epoch), medians across 10 seeds; diagnostics off, no untimed warm‑up.}
\label{tab:pinn3d_time}
\begin{tabular}{@{}lcc@{}}
\toprule
\textbf{Method} & \textbf{Median (ms/epoch)}  \\
\midrule
K--$\beta$            & \bf{94.5} \\
Adam--0.95          & 99.9 \\
Adam--0.999        & 96.3 \\
\bottomrule
\end{tabular}
\end{table}


\subsubsection{Visual impression of the physical significance of achieved loss at epoch 100K}

Figure~\ref{fig:pinn3d_scatter_seed0} shows 3D scatter plots of the steady-state temperature distribution in the domain corresponding to the PINN solution at epoch 100K for seed=0. The Kourkoutas‑$\beta$ solution preserves non‑negativity across the domain, while Adam‑95 shows a small fraction of points with T<0 and Adam‑999 degenerates toward a nearly uniform field.  This mirrors the large gap in final loss and underscores the benefit of a dynamic $\beta_2$ in stiff, composite PINN losses.

\begin{figure}[H]
  \centering
  \begin{subfigure}[b]{0.32\linewidth}
    \centering
    \includegraphics[width=\linewidth]{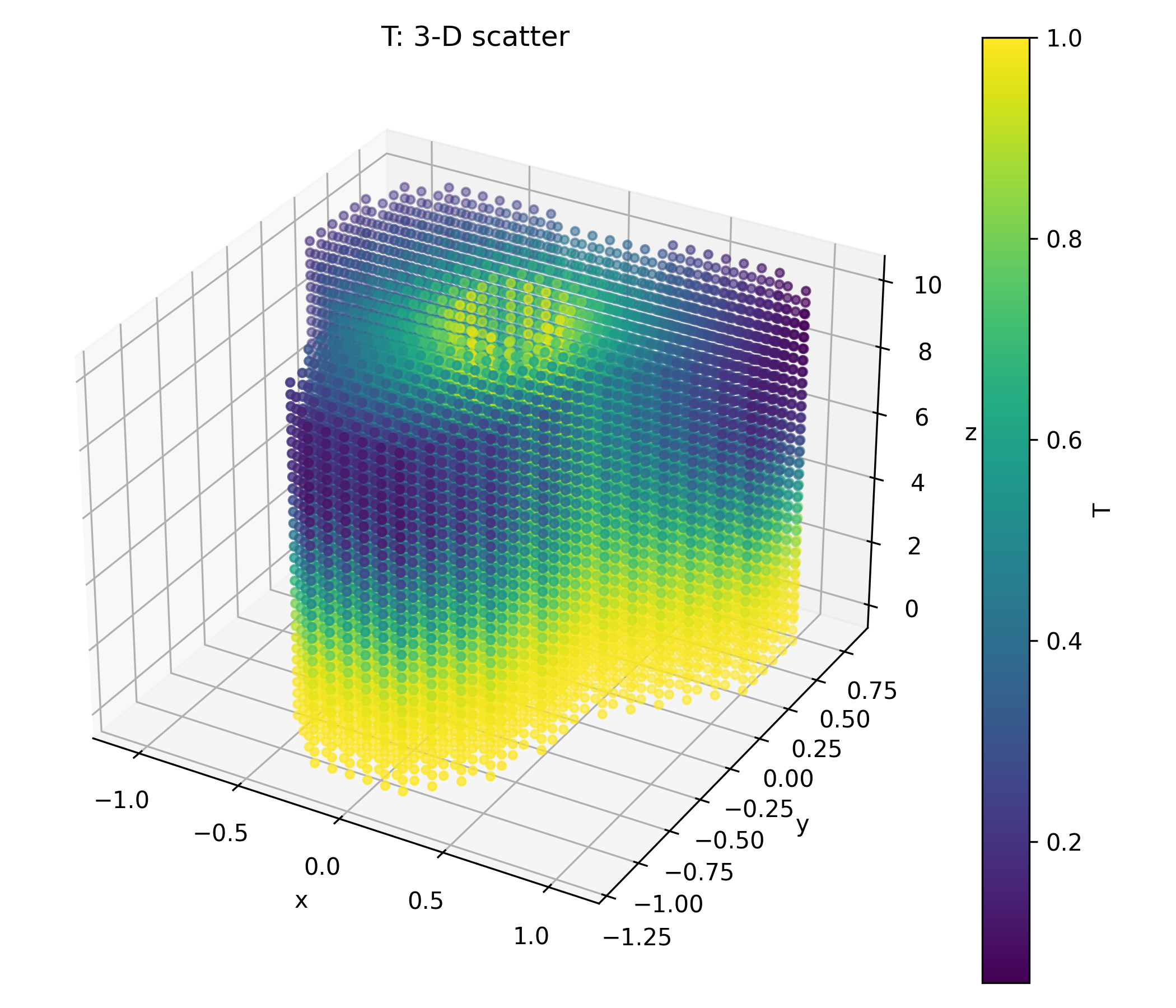}
    \caption{Kourkoutas-$\beta$}
  \end{subfigure}\hfill
  \begin{subfigure}[b]{0.32\linewidth}
    \centering
    \includegraphics[width=\linewidth]{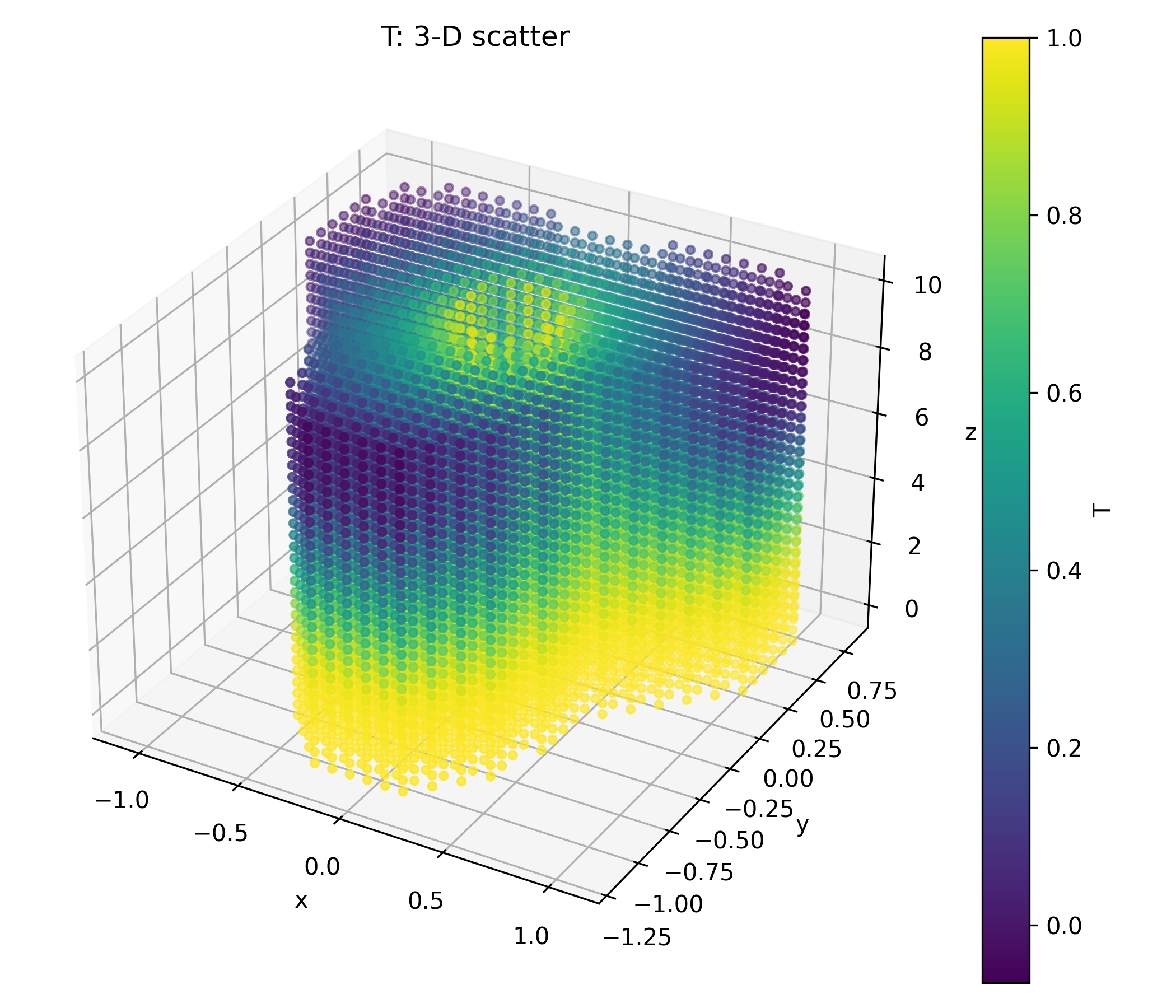}
    \caption{Adam--0.95}
  \end{subfigure}\hfill
  \begin{subfigure}[b]{0.32\linewidth}
    \centering
    \includegraphics[width=\linewidth]{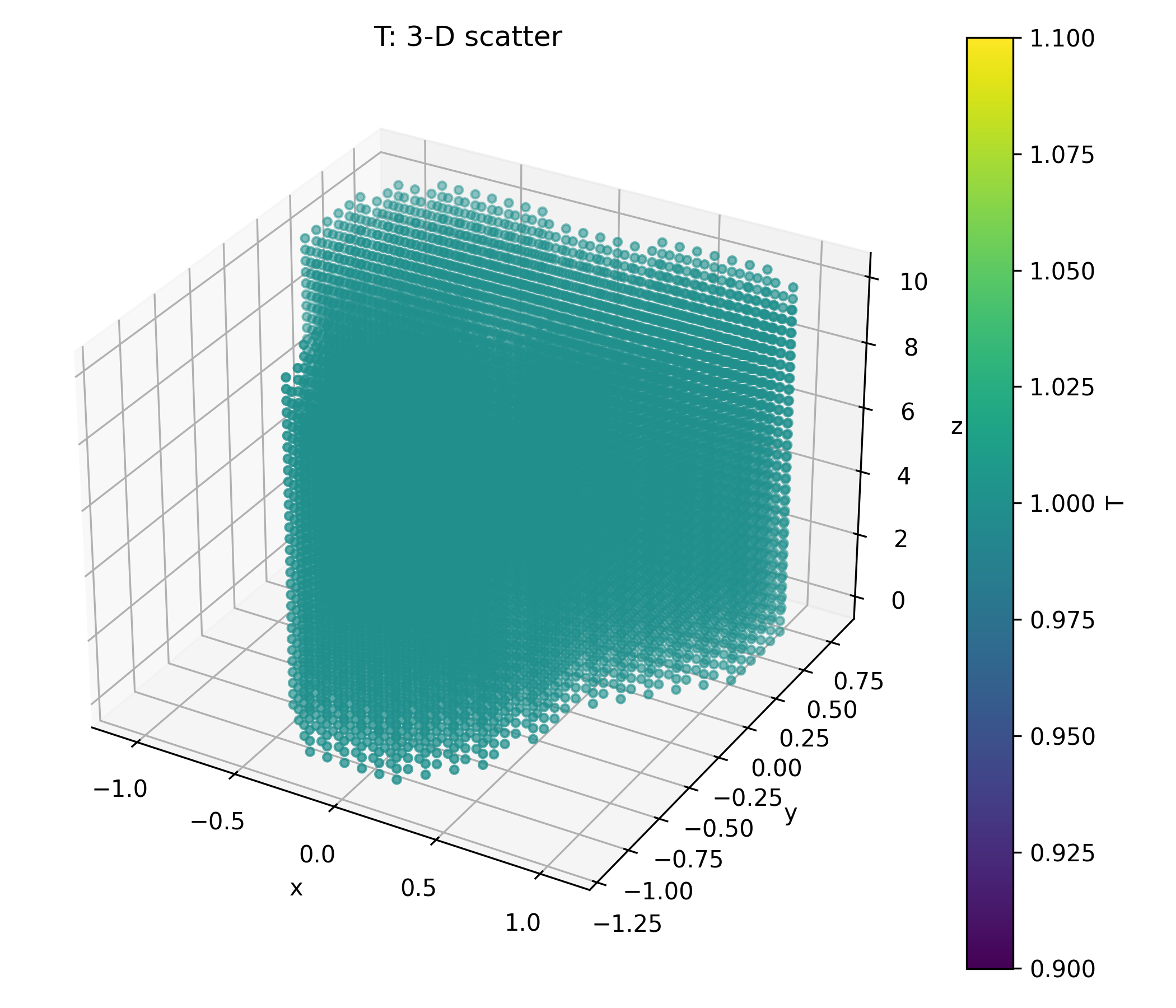}
    \caption{Adam--0.999}
  \end{subfigure}

  \vspace{0.1em}
  \caption{\textbf{PINN--3D temperature field, seed=0 (3-D scatter of collocation points; color = $T$).}
  All runs use the same network, data, and schedule; only the optimizer changes.
  K-$\beta$ yields a physically consistent field with $T\ge 0$ throughout the domain.
  Adam--0.95 is close but exhibits small \emph{undershoots} into negative temperatures at some points
  (its colormap extends slightly below $0$), indicating discrete maximum-principle violations.
  Adam--0.999 collapses toward an almost uniform field ($T{\approx}1$; note its 0.9–1.1 scale),
  consistent with excessive second-moment smoothing.  }
  \label{fig:pinn3d_scatter_seed0}
\end{figure}

\subsubsection{Dynamics of $\beta_2$ and sunspike}

Figure~\ref{fig:kbeta-violins} shows that the sunspike ratio concentrates in the 0.3–0.6 band throughout training, producing $\beta_2$ values in the 0.93–0.96 range via the rule 
$\beta_{2,t}=\beta_{2,\max}-(\beta_{2,\max}-\beta_{2,\min})\,\mathrm{sun}_t$ (with $\beta_{2,\min}{=}0.88$, $\beta_{2,\max}{=}0.999$). The dashed line at $\beta_2{=}0.999$ highlights that Kourkoutas–$\beta$ rarely selects the Adam‑like 
extreme and instead maintains moderate smoothing while staying agile during spikes. Seed‑to‑seed variability is small but not entirely absent showing the dynamic adaptability of Kourkoutas-$\beta$. A mild ripple tracks the cosine‑ramp→constant learning‑rate schedule.

\begin{figure*}[h]
  \centering
  \begin{subfigure}{0.48\textwidth}
    \centering
    \includegraphics[width=\linewidth]{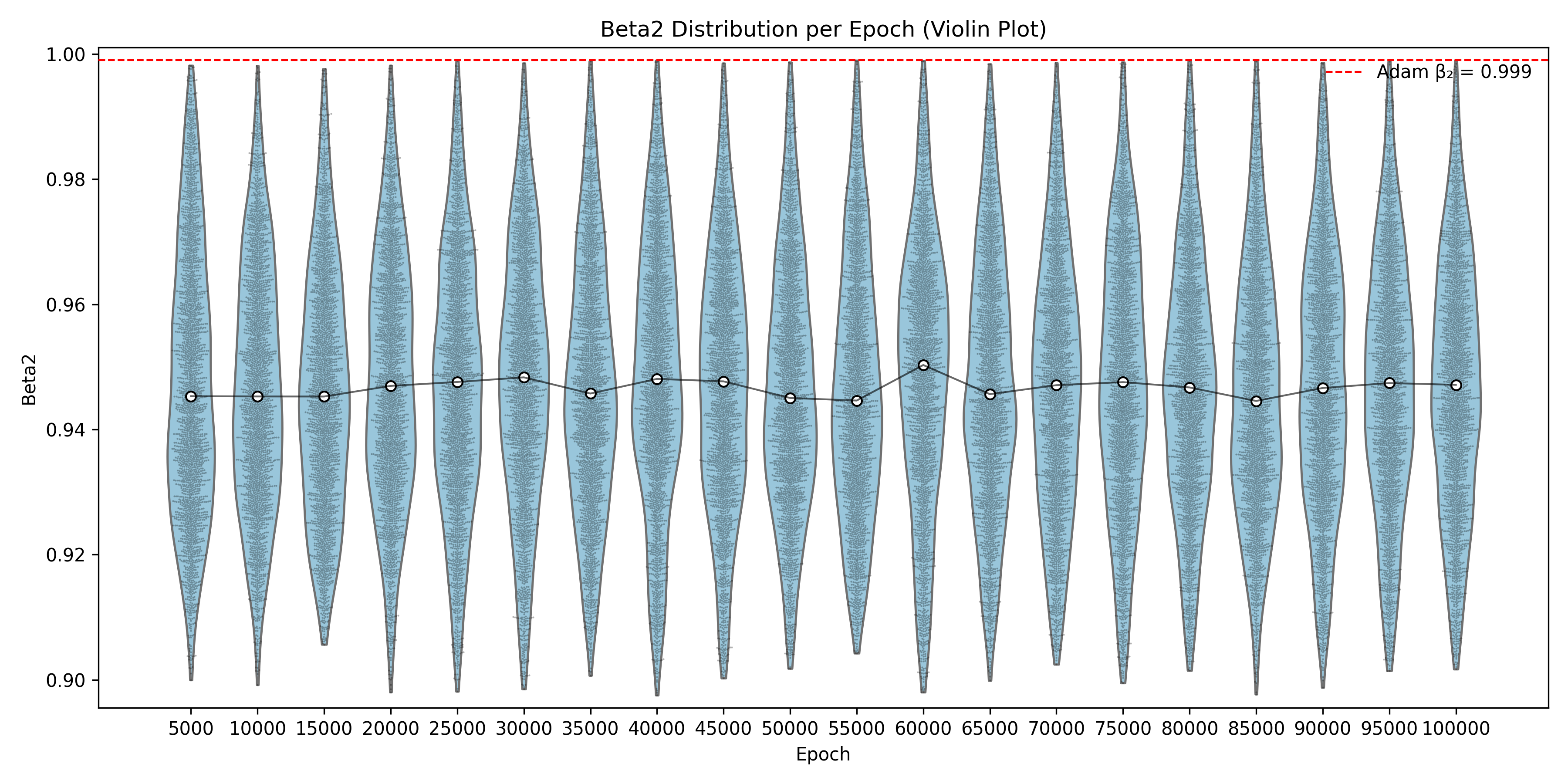}
    \caption{seed=0: $\beta_2$ violins by epoch.}
  \end{subfigure}\hfill
  \begin{subfigure}{0.48\textwidth}
    \centering
    \includegraphics[width=\linewidth]{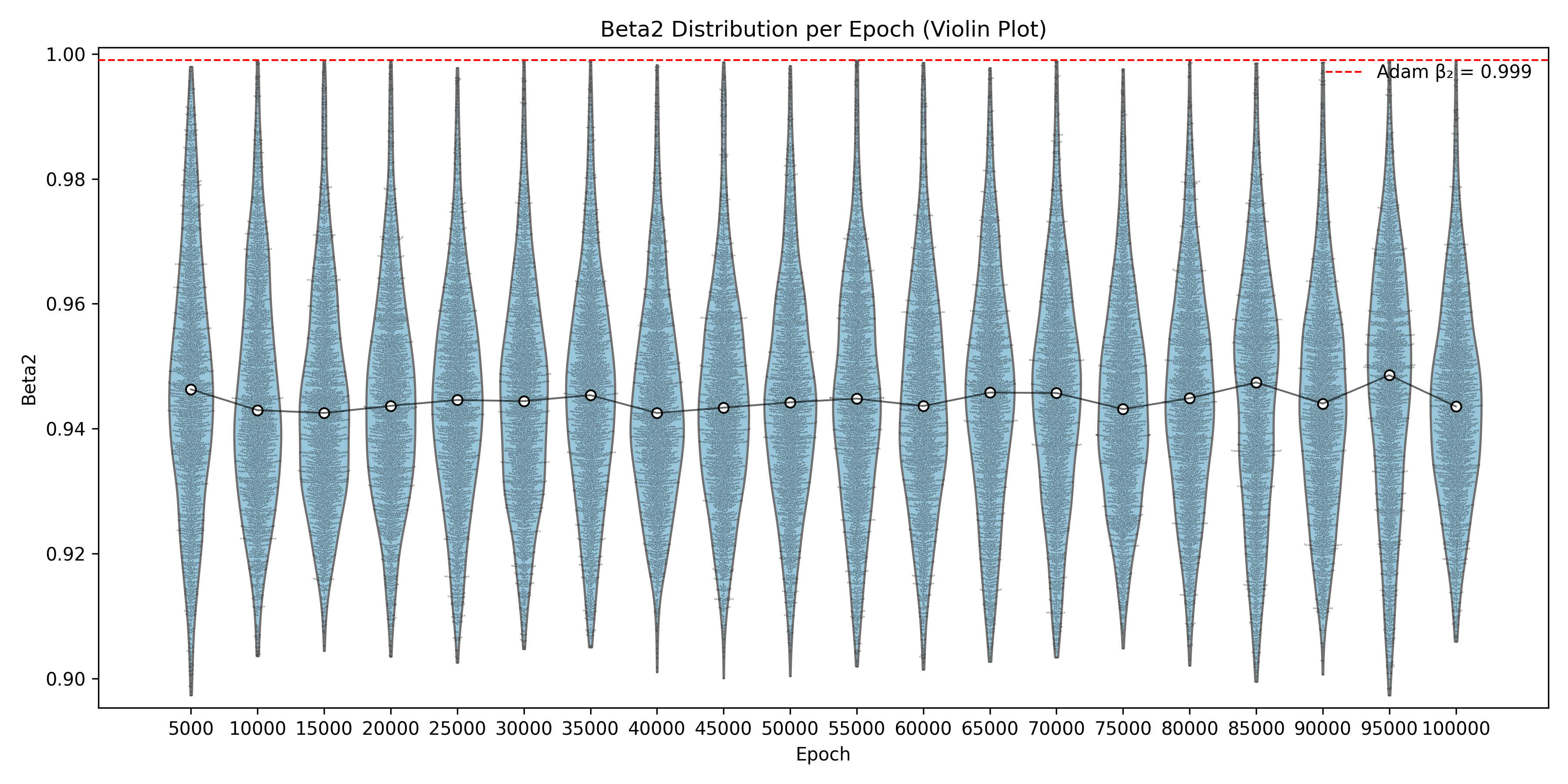}
    \caption{Seed 1: $\beta_2$ violins by epoch.}
  \end{subfigure}

  \medskip

  \begin{subfigure}{0.48\textwidth}
    \centering
    \includegraphics[width=\linewidth]{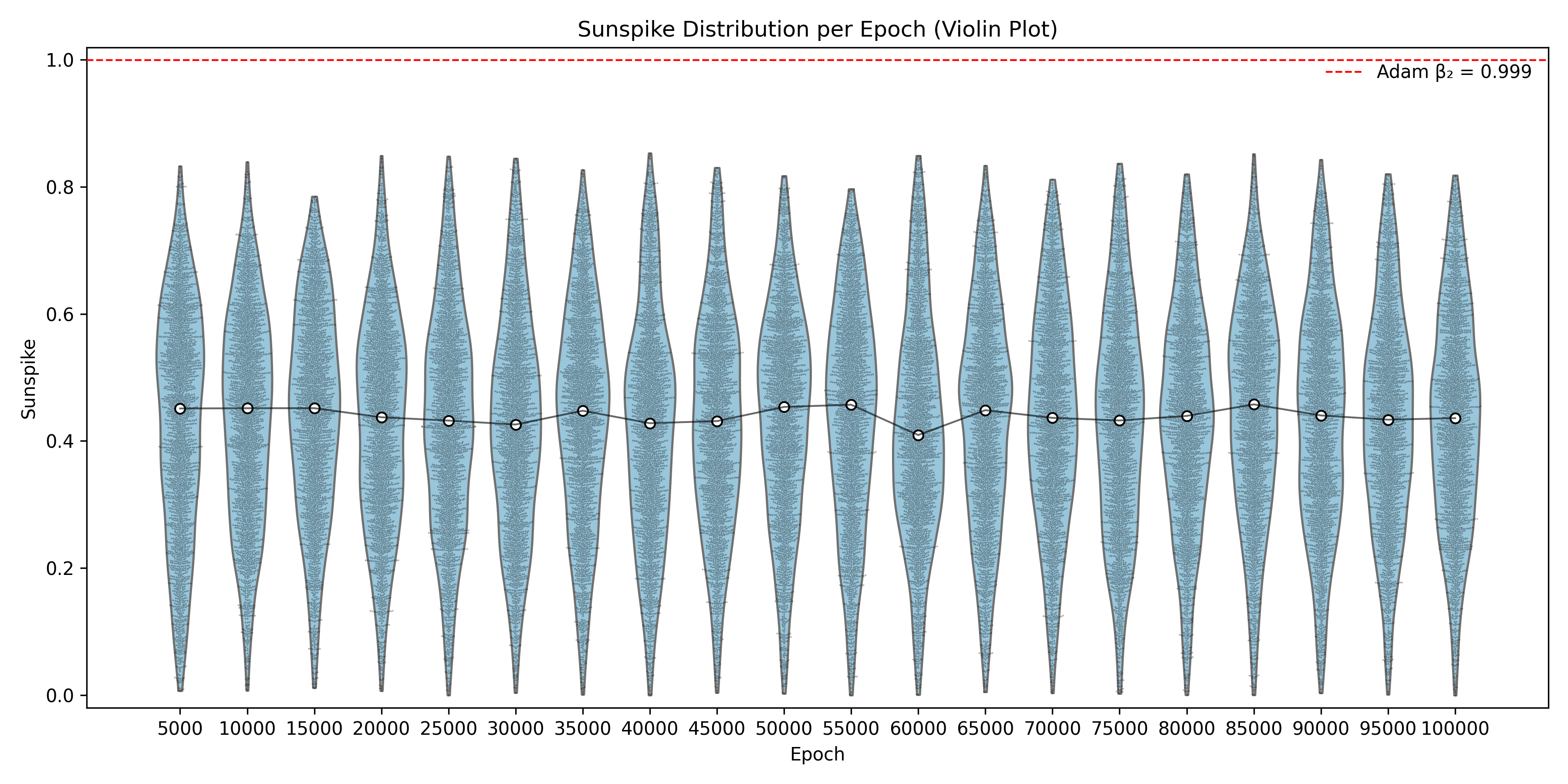}
    \caption{seed=0: sunspike violins by epoch.}
  \end{subfigure}\hfill
  \begin{subfigure}{0.48\textwidth}
    \centering
    \includegraphics[width=\linewidth]{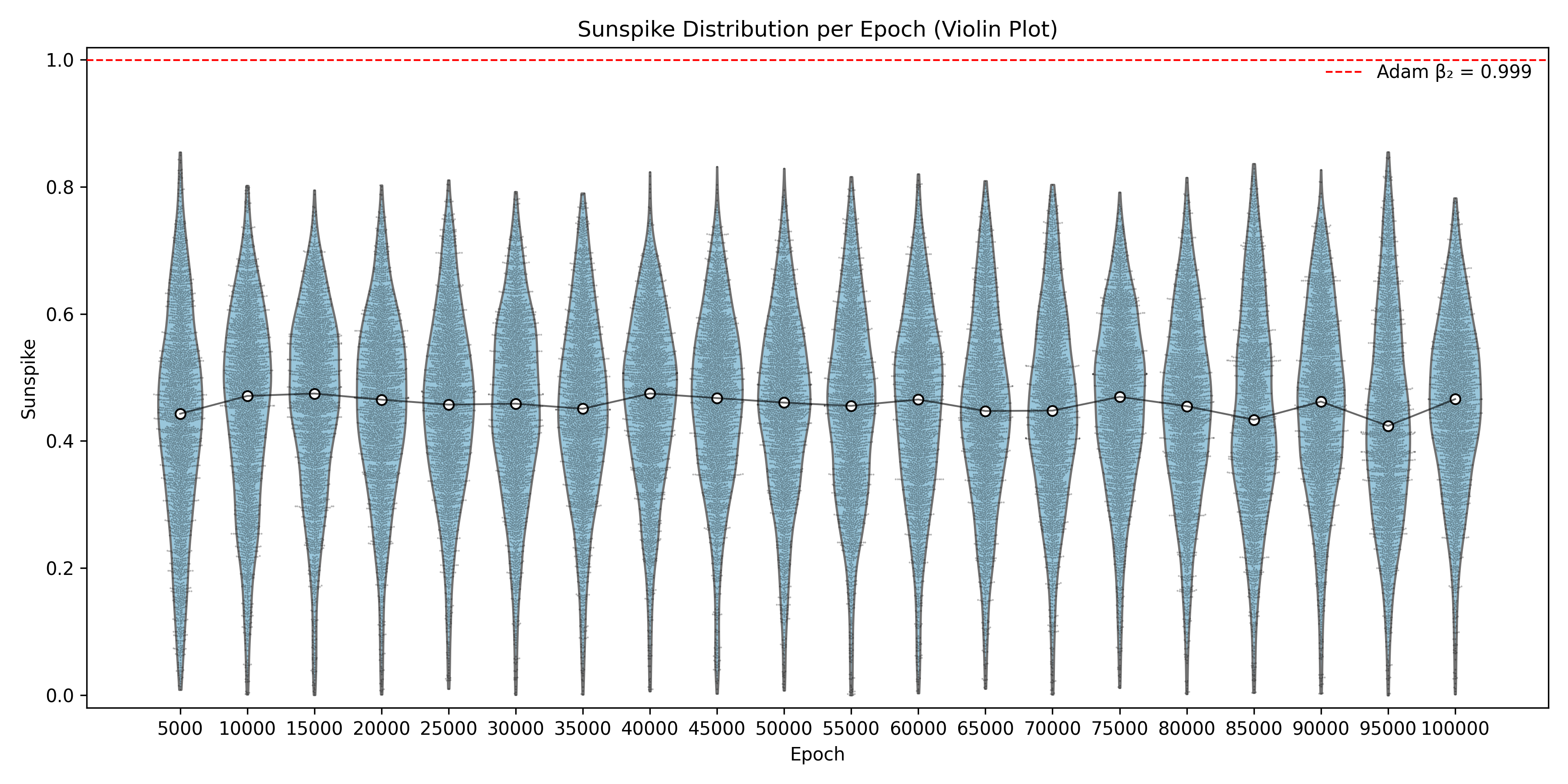}
    \caption{Seed 1: sunspike violins by epoch.}
  \end{subfigure}

  \caption{K–$\beta$ dynamics on PINN–3D (100k epochs). 
  Dashed line marks Adam’s fixed $\beta_2{=}0.999$. K–$\beta$ 
  selects $\beta_2$ predominantly in the $0.93$–$0.96$ band, driven by 
  sunspike values concentrated around $0.3$–$0.6$. Seed‑to‑seed behavior 
  is consistent.}
  \label{fig:kbeta-violins}
\end{figure*}

\paragraph{Median vs.\ mobility}
Across epochs the \emph{median} \(\beta_2\) under Kourkoutas–\(\beta\) typically
hovers near \(0.95\)—i.e., close to the fixed setting of Adam–0.95. Yet Adam–0.95
fails on most seeds while Kourkoutas–\(\beta\) succeeds on all (see
Table~\ref{tab:pinn3d_summary}). This suggests that the advantage does not come from picking a
better constant; it comes from \emph{mobility and granularity}. Kourkoutas–\(\beta\)
moves \(\beta_2\) \emph{in time and per layer}: it dips below the median during
spiky phases and climbs above it during calm phases, yielding a broad,
state‑dependent distribution around the same central value. The violins make this
visible (the interquartile band spreads and shifts) whereas a fixed \(\beta_2\) is
just a horizontal line. In short: similar medians, very different dynamics, and
the dynamic, layer‑wise span is what turns instability into reliable convergence.

\ref{appendix:PINNheatmaps} includes heatmaps of the $\beta_2$ and sunspike distributions which provide an additional method to visualize the agile adaptation that Kourkoutas-$\beta$ provides throughout the PINN training. 


\subsection{{\bf Testbed C:} Length--Jitter + Rare Trigger (MLX)}
\label{sec:toy-rare-trigger}

\paragraph{Purpose}
This toy isolates the “needle in a haystack” pattern that often produces spiky, heavy‑tailed gradient statistics: a single rare token may appear anywhere in a variable‑length sequence, and the task is to detect its presence. Such sparsity and length jitter are a good stress test for optimizers’ second‑moment adaptation.

\paragraph{Data generator (exactly as in \texttt{rare\_trigger\_toy.py})}
We sample batch size \(B=64\) sequences with lengths \(L_i \sim \mathrm{Unif}\{80,\dots,256\}\).
Tokens are i.i.d.\ integers in \([1,255]\); the padding id is \(0\).
With probability \(p_{\text{trigger}}=0.01\) a per‑sample “rare trigger” token id \(255\) is placed at a uniformly random valid position.
The label is \(y=1\) iff the trigger appears inside the unpadded region, else \(0\).
All randomness is from MLX (\texttt{mlx.random}) with a fixed base seed and per‑step seed offsets.

\paragraph{Model}
A tiny bag‑of‑embeddings classifier:
\(\mathbf{E}\in\mathbb{R}^{256\times d}\) with \(d=64\);
we mean‑pool valid token embeddings and apply a single linear head \(\mathbf{w}\in\mathbb{R}^{d\times 1}, b\in\mathbb{R}\).
Loss is binary cross‑entropy (BCE) with logits.
No dropout, weight decay, clipping, or trust region.

\paragraph{Training protocol}
Each run trains for \(30{,}000\) steps at learning rate \(\rho=10^{-2}\), with 10 JIT warmup steps (not timed).
We sweep seeds \(s\in\{0,\dots,9\}\).
The script is self‑contained and MLX‑only; it imports the release optimizer if available and otherwise falls back to MLX Adam.

\paragraph{Optimizers and hyperparameters (as used)}
\begin{itemize}
\item \textbf{Kourkoutas‑\(\boldsymbol{\beta}\)}: \(\beta_1=0.9\), dynamic \(\beta_2\in[0.88,0.999]\),
\(\alpha=0.93\), \(\varepsilon=10^{-8}\), bias correction \texttt{"beta2max"}, warmup \(=50\) steps,
single global bucket, degenerate AMSGrad \texttt{decay=0} (so \(\widehat v_t=v_t\); a \(v_{\max}\)
buffer is still allocated but unused), \texttt{max\_ratio=None}, no adaptive‑tiny.
\item \textbf{Adam--95}: MLX Adam with \(\beta_1=0.9,\ \beta_2=0.95\), \(\varepsilon=10^{-8}\), bias correction enabled.
\item \textbf{Adam--999}: MLX Adam with \(\beta_1=0.9,\ \beta_2=0.999\), \(\varepsilon=10^{-8}\), bias correction enabled.
\end{itemize}

\subsubsection*{Why this testbed?}
The gradient signal for the trigger embedding is extremely sparse (present in only \(\sim 1\%\) of sequences) and appears at random positions in sequences whose lengths vary over a \(3.2\times\) range.
This creates intermittent, high‑kurtosis gradient bursts on top of small, length‑dependent background gradients from the mean‑pool.
An optimizer either (i) over‑smooths the second moment and is slow/fragile on the bursts, or (ii) adapts quickly to bursts without destroying stability elsewhere.
Kourkoutas‑\(\beta\) explicitly targets this regime via a \emph{dynamic} \(\beta_2\in[0.88,0.999]\) and a modest \(\alpha=0.93\) EMA over pooled grad‑norms, with bias correction matched to the current \(\beta_2\) cap (\texttt{"beta2max"}).
That design should help when informative gradients are rare and spiky—exactly the setting here.

\subsubsection*{Results}

\paragraph{Per‑seed final losses (after 30K steps)} Table~\ref{tab:toyrare-seeds} shows the final binary cross-entropy (BCE) loss after 30{,}000 steps along with the ratio of Adam losses to that of Kourkoutas-$\beta$. Median and geometric mean values across seeds appear 
below each column. We note that two seeds (5 and 8) resulted in failure for all optimizers (loss \(\gg 10^{-2}\)) and additionally Adam-999 failed for seed=3. For all of the 8 successful runs,  Kourkoutas-$\beta$ consistently achieves lower loss than both 
Adam variants. 

\begin{table}[h]
\centering
\small
\begin{tabular}{rccc cc}
\toprule
Seed & K-\(\beta\) & Adam--95 & Adam--999 & Adam--95/K-\(\beta\) & Adam--999/K-\(\beta\) \\
\midrule
0 & 0.001055 & 0.001265 & 0.008644 & 1.199 & 8.195 \\
1 & 0.001142 & 0.001285 & 0.008542 & 1.125 & 7.480 \\
2 & 0.000703 & 0.000776 & 0.006253 & 1.104 & 8.900 \\
3 & 0.001498 & 0.001809 & 0.010653 & 1.208 & 7.110 \\
4 & 0.001280 & 0.001405 & 0.009059 & 1.098 & 7.079 \\
5 & 0.524838 & 0.517825 & 0.374564 & 0.987 & 0.714 \\
6 & 0.001195 & 0.001270 & 0.006817 & 1.062 & 5.706 \\
7 & 0.001229 & 0.001330 & 0.005996 & 1.082 & 4.878 \\
8 & 0.103896 & 0.102899 & 0.086955 & 0.991 & 0.837 \\
9 & 0.002445 & 0.002824 & 0.007676 & 1.155 & 3.141 \\
\midrule
Median & 0.001255 & 0.001368 & 0.008593 & 1.103 & 6.390 \\
Geo‑mean & 0.00342 & 0.00378 & 0.0141 & 1.105 & 4.13 \\
\bottomrule
\end{tabular}
\caption{Final BCE loss by seed. “Geo‑mean” is the geometric mean across seeds; ratio columns are per‑seed factors (baseline/K-\(\beta\)). Two seeds (\#5 and \#8) are “failures” for both Adam--95 and K-\(\beta\) (loss \(\gg 10^{-2}\)), and Adam--999 additionally degrades on seed \#3.}
\label{tab:toyrare-seeds}
\end{table}

\paragraph{Aggregate and significance}
Losses are heavy‑tailed due to occasional failures, so in Table~\ref{tab:toyrare-stats} we also report robust summaries and paired tests on \(\log_{10}\) loss, which turns multiplicative effects into additive ones. As, expected the advantage of Kourkoutas-$\beta$ is 
statistically significant. 

\begin{table}[H]
\centering
\small
\begin{tabular}{lcc}
\toprule
 & Adam--95 vs K-\(\beta\) & Adam--999 vs K-\(\beta\) \\
\midrule
Wins for K-\(\beta\) (out of 10) & \(8/10\) (sign test \(p=0.109\)) & \(8/10\) (sign test \(p=0.109\)) \\
Paired \(t\) on \(\log_{10}\) loss & \(t(9)=4.29,\ p=0.0022\) & \(t(9)=4.80,\ p=9.7\times10^{-4}\) \\
Wilcoxon on \(\log_{10}\) loss & \(p=0.0098\) (exact) & \(p=0.0098\) (exact) \\
Geo‑mean ratio (baseline/K-\(\beta\)) & \(1.105\times\) \([1.046,\,1.153]\) & \(4.13\times\) \([2.12,\,8.05]\) \\
Median ratio (baseline/K-\(\beta\)) & \(1.103\times\) & \(6.39\times\) \\
\bottomrule
\end{tabular}
\caption{Paired comparisons across 10 seeds. The geometric‑mean ratio uses the paired mean of \(\log_{10}\) ratios and is reported with a \(95\%\) CI transformed back to multiplicative units.}
\label{tab:toyrare-stats}
\end{table}

\paragraph{Timing (informal)}
Wall‑clock times per run varied with JIT/cache state. Medians:
Kourkoutas-\(\beta\) \(79.9\,\mathrm{s}\), Adam--95 \(73.4\,\mathrm{s}\), Adam--999 \(33.5\,\mathrm{s}\) for \(30\text{K}\) steps on the same machine.
We did not tune for speed here; this toy is aimed at optimizer behavior rather than throughput.

\subsubsection*{Takeaways.}
Across 10 seeds, Kourkoutas‑\(\beta\) achieves consistently lower final loss than Adam(0.95) on \(8/10\) seeds, and the advantage is statistically significant on \(\log_{10}\) loss (\(t(9)=4.29,\ p=0.0022\); Wilcoxon \(p=0.0098\)).
The typical improvement is \(\approx1.10\times\) vs Adam--95 and \(\approx4.1\times\) vs Adam--999 in geometric‑mean loss, with median per‑seed factors \(1.10\times\) and \(6.39\times\), respectively.
Occasionally,  unlucky trigger realizations (class imbalance within batches / late discoveries) can lead to  “failures” where  loss \(\gg 10^{-2}\). 
Both Kourkoutas-\(\beta\) and Adam--95 experience the same two “failed” seeds; on those particular runs the absolute loss is slightly worse for Kourkoutas-\(\beta\), which dominates arithmetic means but is de‑emphasized by robust statistics (median / log‑scale).
Given that we purposely disabled AMSGrad/decay and any trust‑region (\texttt{max\_ratio=None}) to keep the toy minimal, these results suggest Kourkoutas‑\(\beta\)’s dynamic \(\beta_2\) is beneficial precisely in the sparse, bursty‑gradient regime this toy induces.
  
\subsection*{Reproducibility notes.}
Exact settings: \(B=64\), \(d=64\), \(L\in[80,256]\), \(p_{\text{trigger}}=0.01\), vocab \(=256\) with \(\text{pad\_id}=0\) and \(\text{trig\_id}=255\);
\(30{,}000\) steps at \(\eta=10^{-2}\); 10 warmup steps.
Kourkoutas-\(\beta\) hyperparameters and Adam baselines are as listed above; the script (\texttt{rare\_trigger\_toy.py}) toggles between Kourkoutas-\(\beta\) and MLX Adam via \texttt{HAVE\_KBETA}.




\subsection{{\bf Testbed D:} Character-Level Language Modeling on \texttt{small-enwik8} (10 seeds)}
\label{sec:enwik8}

\paragraph{Purpose}
This compact character‑level language‑modeling benchmark stresses second‑moment adaptation in a realistic Transformer training loop under \emph{small data}, pronounced \emph{length‑jitter} (sequence lengths $L\!\in[16,512]$), and abrupt \emph{piecewise‑constant learning‑rate drops}. Although training is fully deterministic (no dropout), these factors induce non‑stationary gradient scales with occasional spikes. The goal is to test whether a \emph{layer‑wise dynamic} $\beta_2$ (Kourkoutas‑$\beta$) can track such bursts better than fixed‑$\beta_2$ Adam ($\beta_2\!\in\!\{0.95,\,0.999\}$) while keeping runtime on par. %
Dataset construction and SHA‑256 checksums for \texttt{small‑enwik8} are given in the data paragraph for this testbed.

\subsubsection*{Dataset and creation (verifiable)}
We use the first 30\,MB of \texttt{enwik8} (the classic Hutter Prize corpus). The slice is created deterministically:
\begin{lstlisting}[language=bash]
curl -L -o enwik8.zip https://data.deepai.org/enwik8.zip
unzip enwik8.zip
head -c 30000000 enwik8 > small-enwik8.txt
\end{lstlisting}
Checksums on our machine:
\begin{lstlisting}[language=bash]
sha256sum enwik8             # 2b49720e...c024a8
sha256sum small-enwik8.txt   # e0152eee...298b7
\end{lstlisting}
Re‑creating \texttt{small-enwik8.txt} reproduced the same SHA‑256 (bit‑for‑bit identity).

\subsubsection*{Model and training protocol (as in the provided script)}
A 6‑block Transformer (\(d_{\text{model}}{=}512\), \(n_{\text{head}}{=}8\), FFN width \(4d\)), GELU, LayerNorm, causal self‑attention; no dropout or weight decay. Training uses variable sequence length with deterministic bucketing: \(L\in[16,512]\) rounded to a multiple of 32, batch \(=4\), context window \(=512\). Steps \(=50{,}001\); learning‑rate schedule (applied identically to all methods): \(1\!\times\!10^{-3}\) for \(1\le s<30\mathrm{k}\), \(5\!\times\!10^{-4}\) for \(30\mathrm{k}\le s<40\mathrm{k}\), then \(1\!\times\!10^{-4}\) for \(40\mathrm{k}\le s\le 50\mathrm{k}\).
Evaluation uses a fixed held‑out batch (length 256, \(B{=}128\)) to report cross‑entropy and bits‑per‑character (BPC). We run 10 matched seeds (0–9).

\subsubsection*{Optimizers and settings}
\begin{itemize}
\item \textbf{Kourkoutas–$\beta$} (our method): \(\beta_1{=}0.9\); dynamic \(\beta_2\in[0.88,0.999]\); \(\alpha{=}0.93\) (EMA for sunspike); \(\varepsilon{=}10^{-8}\); warm‑up \(=250\) steps; \texttt{bias\_correction="beta2max"}; per‑\emph{array} stable buckets (\texttt{layer\_key\_fn} yields a stable id per parameter); no AMSGrad/clip/adaptive‑tiny; diagnostics off.
\item \textbf{Adam--95}: MLX Adam (\(\beta_1{=}0.9, \beta_2{=}0.95, \varepsilon{=}10^{-8}\)), bias correction on.
\item \textbf{Adam--999}: MLX Adam (\(\beta_1{=}0.9, \beta_2{=}0.999, \varepsilon{=}10^{-8}\)), bias correction on.
\end{itemize}

\subsubsection*{Metrics and reporting}
Primary metric: final Bits-Per-Character (BPC )at step \(50\mathrm{k}\). We summarize mean\(\pm\)sd and median\([\mathrm{IQR}]\) over seeds and perform paired statistics on the seedwise BPC differences (Adam – K‑\(\beta\)). 
All optimizer comparisons are paired by seed. We report two‑sided paired t‑tests on per‑seed differences with 95\% CIs and paired effect sizes ($d_z$). As distribution‑free complements we report Wilcoxon signed‑rank tests (exact two‑sided p) and sign tests (exact two‑sided p); for binary outcomes we use McNemar’s exact test. Where two pairwise comparisons are made within a testbed (Kourkoutas--$\beta$ vs Adam--95 and Kourkoutas--$\beta$ vs Adam--999), we additionally report Holm‑adjusted p-values. We do not use Welch’s t‑test because it assumes independent samples and discards the seed pairing.

Hardware: Apple Studio M2 Ultra (198\,GB). Timings are full‑run wall‑clock seconds (no untimed warm‑up).

\subsubsection*{Why this testbed?}
Beyond PDE surrogates and the length‑jitter + rare‑trigger toy, we want a mainstream sequence‑modeling task that still stresses second‑moment adaptivity. Character‑level modeling on a compact slice of \texttt{enwik8} offers: (i) nontrivial long‑range dependencies, (ii) variable effective sequence lengths during training, and (iii) sharp loss shifts when the model locks onto frequent symbol patterns. These features can produce intermittent gradient bursts. This testbed checks whether layer‑wise dynamic $\beta_2$ helps outside physics workloads.

\subsubsection*{Results (10 seeds)}
Table~\ref{tab:enwik8-abs} reports final BPC at step $50\mathrm{k}$. 
Kourkoutas--$\beta$ achieves $\mathbf{1.639\pm0.027}$ (median $1.650$ [0.036]), 
well below both Adam--95 ($2.637\pm0.681$, median $2.378$ [1.143]) 
and Adam--999 ($3.906\pm0.087$, median $3.926$ [0.128]).

\begin{table}[H]
\centering
\caption{Final BPC at step $50\mathrm{k}$ (10 seeds). Lower is better.}
\label{tab:enwik8-abs}
\begin{tabular}{lcc}
\toprule
\textbf{optimizer} & \textbf{Mean$\pm$sd} & \textbf{Median [IQR]} \\
\midrule
K--$\beta$      & $\mathbf{1.639 \pm 0.027}$ & $1.650\ [0.036]$ \\
Adam--95        & $2.637 \pm 0.681$          & $2.378\ [1.143]$ \\
Adam--999       & $3.906 \pm 0.087$          & $3.926\ [0.128]$ \\
\bottomrule
\end{tabular}
\end{table}

Paired differences (Adam $-$ K-\(\beta\)) are large and significant (Table~\ref{tab:enwik8-stats}): 
vs.\ Adam--95 the mean BPC gain is \(0.997\) with \(95\%\) CI \([0.504,\,1.491]\), 
\(t(9){=}4.58\), \(d_z{=}1.45\), \(r{=}0.836\); 
vs.\ Adam--999 it is \(2.267\) with \(95\%\) CI \([2.210,\,2.323]\), 
\(t(9){=}90.93\), \(d_z{=}28.76\), \(r{=}0.999\). 
Kourkoutas-\(\beta\) wins on \(10/10\) seeds against both baselines. 
For both comparisons the Wilcoxon signed‑rank (exact two‑sided) has all 10 paired differences positive (no ties), giving the maximal statistic \(W^{+}{=}55\) and the same exact two‑sided \(p{=}0.001953\).
(For completeness, Holm‑corrected two‑sided \(t\)-test \(p\)-values across the two comparisons are \(1.33\times 10^{-3}\) for Adam--95 and \(2.40\times 10^{-14}\) for Adam--999.)

\begin{table}[H]
\centering
\caption{Paired BPC differences at step \(50\mathrm{k}\) (Adam – K-$\beta$). 
Effect sizes: $d_z$ (paired Cohen’s $d$), $r$ (correlation). 
All nonparametric $p$-values are exact two-sided; Holm adjustment is across the two paired $t$-tests within this testbed.}
\label{tab:enwik8-stats}
\renewcommand{\arraystretch}{1.1}
\small
\begin{tabular}{lccccccc}
\toprule
\textbf{Comparison} & \textbf{Mean} & \textbf{95\% CI} & $\mathbf{t(9)}$ & $\mathbf{d_z}$; $\mathbf{r}$ & \textbf{Wilcoxon $p$} & \textbf{Wins} & \textbf{Holm $p$} \\
\midrule
Adam--95  & 0.997  & [0.504,\,1.491] & 4.58  & 1.45; 0.836 & $1.953\times 10^{-3}$ & 10/10 & $1.33\times 10^{-3}$ \\
Adam--999 & 2.267  & [2.210,\,2.323] & 90.93 & 28.76; 0.999 & $1.953\times 10^{-3}$ & 10/10 & $2.40\times 10^{-14}$ \\
\bottomrule
\end{tabular}
\end{table}

\subsubsection*{Visual interpretation of the loss curves}

Figure~\ref{fig:enwik8-curves-seed} shows the evolution of Bits-Per-Character (BPC) with training step, 
grouped by optimizer (different line styles) and by seed (different colors). 
Kourkoutas--$\beta$ curves descend rapidly and smoothly and remain tightly clustered across seeds. 
Adam--999 curves plateau early and show virtually no further loss reduction beyond the initial drop. 
Adam--95 exhibits wide seed-to-seed variability, with a mean that is significantly higher than the tightly clustered 
Kourkoutas--$\beta$ BPC lines (see Figure~\ref{fig:enwik8-mean-std}).

\begin{figure}[H]
  \centering
  \includegraphics[width=0.80\linewidth, clip,trim=5pt 5pt 5pt 5pt]{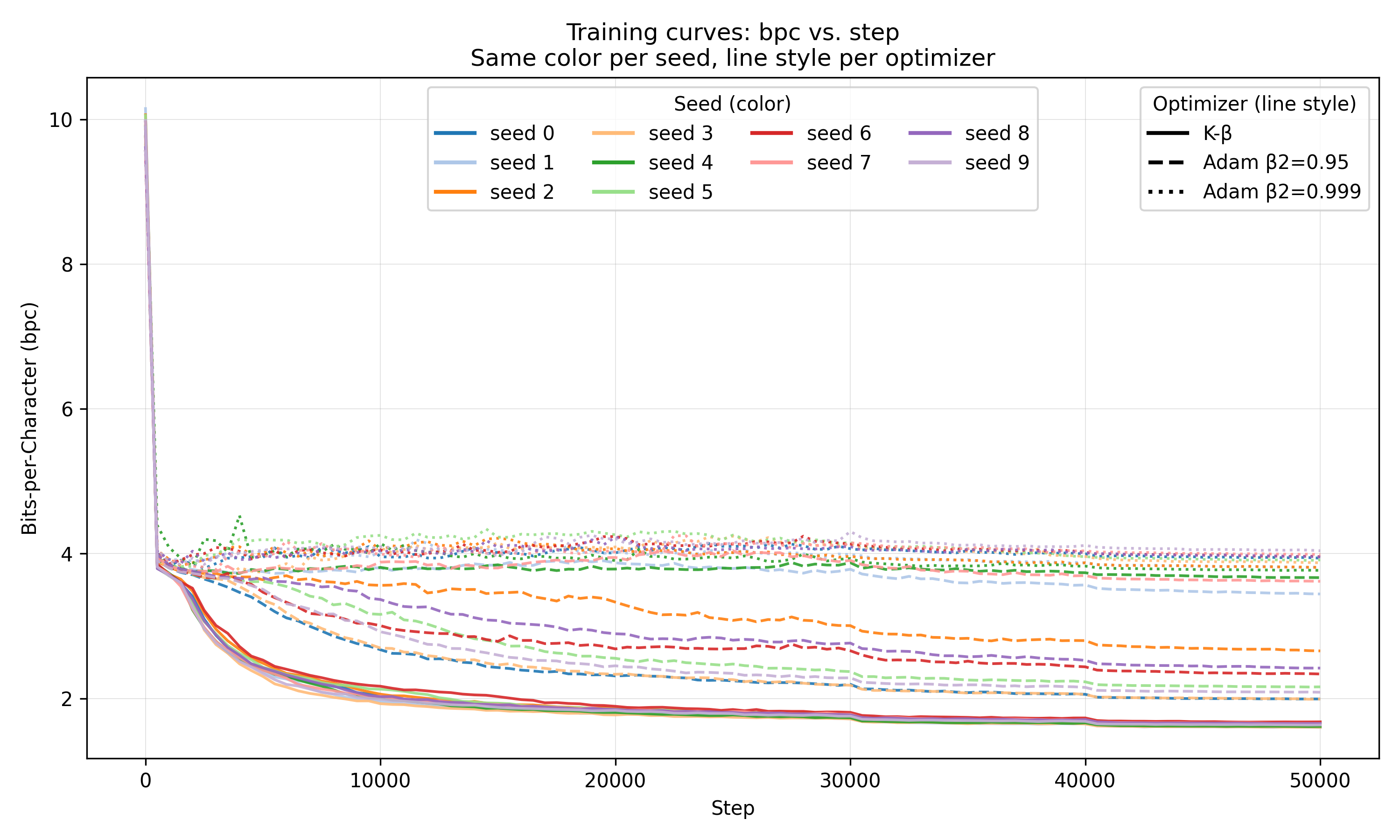}
  \caption{\textbf{BPC vs.\ step across 10 seeds (same color per seed; line‑style per optimizer).}
  K--$\beta$ descends smoothly and tightly clustered, while Adam--95 exhibits wide seed‑to‑seed spread and Adam--999 plateaus high.}
  \label{fig:enwik8-curves-seed}
\end{figure}

\begin{figure}[H]
  \centering
  \includegraphics[width=0.80\linewidth, clip,trim=5pt 5pt 5pt 5pt]{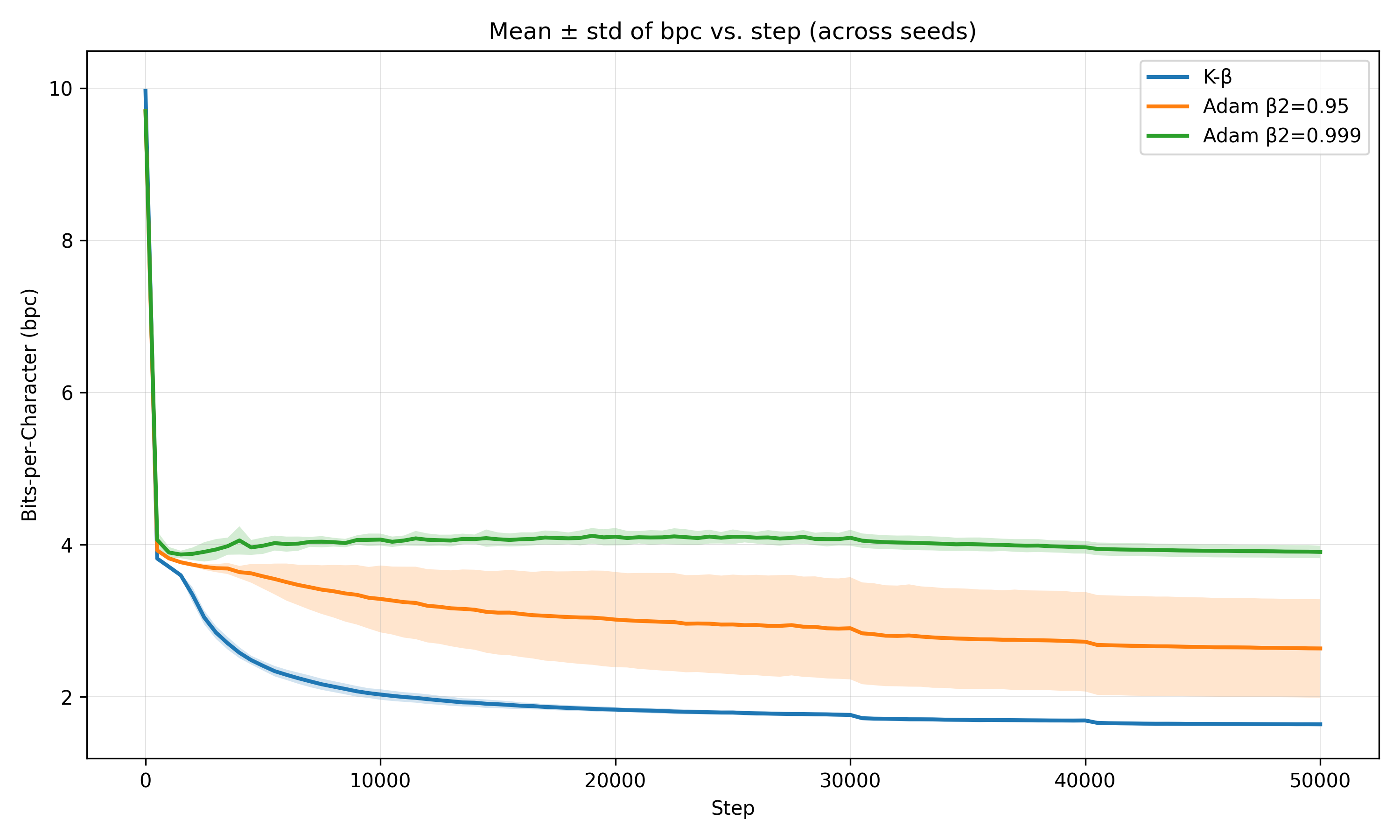}
  \caption{\textbf{Mean $\pm$ 1\,sd BPC over seeds (smoothed curves).}
  Bands show across‑seed variability; K--$\beta$’s band is narrow throughout, indicating robust convergence.}
  \label{fig:enwik8-mean-std}
\end{figure}

\paragraph{Timing}
Median wall-clock per full run: Kourkoutas--$\beta$ \(2109\,\mathrm{s}\) vs.\ Adam--95 \(1992\,\mathrm{s}\) and Adam--999 \(1994\,\mathrm{s}\); 
i.e.\ Kourkoutas--$\beta$ is \(\approx 6\%\) slower here. All runs use identical compile/barrier settings. 
In our other tests (Secs.~\ref{sec:transformer}--\ref{sec:PINN}), Kourkoutas--$\beta$ was parity-speed with diagnostics off. 
This language-modeling harness pays a small overhead for per-array bucketing while remaining far more stable across seeds. 
Part of this overhead may stem from the extra ordering of the parameter tree needed for seed reproducibility, 
which in this testbed was implemented ad-hoc without leveraging MLX tree utilities. 
This step could likely be optimized further in future work.

\subsubsection*{Takeaways}
On \texttt{small-enwik8}, layer‑wise dynamic \(\beta_2\) (via sunspike) yields:
(i) markedly lower BPC than fixed‑\(\beta_2\) Adam (\(\sim38\%\) mean reduction vs Adam–95 and \(\sim58\%\) vs Adam–999), 
(ii) dramatically lower across‑seed variance (IQR \(0.036\) vs \(1.143\) for Adam–95), and 
(iii) universal wins across seeds. This extends the benefits observed on PDE workloads to a standard language‑modeling task and suggests that adapting the second‑moment memory to bursty periods is beneficial beyond physics.

\subsection{Ablations (compact)}
\label{sec:ablations-compact}

\subsubsection{Kourkoutas-$\beta$ configured as Adam}
We verify two equivalences used as controls: (i) with fixed $\beta_2$ and \texttt{bias\_correction}=\texttt{"none"}, Kourkoutas-$\beta$ reproduces Adam with bias correction \emph{off}; (ii) with fixed $\beta_2$ and \texttt{bias\_correction}=\texttt{"beta2max"}, it reproduces Adam with bias correction \emph{on} (see App.~\ref{appendix:ablations}). All other settings are identical across methods (\,$\rho$, $\beta_1$, schedules, seeds\,). When $\beta_2$ is fixed, the choice of \texttt{layer\_key\_fn} (single global bucket $\lambda\!:\!0$, coarser pooling \texttt{p.shape}, or a fine module path/id) has no effect on the update.

\subsubsection{Hyperparameter ablations}
We assess sensitivity to the EMA coefficient $\alpha$ used for the pooled gradient‑norm statistic. Results for the Transformer (Heat2D) at epoch 100 are shown in Table~\ref{tab:alpha-ablation}. The best validation loss is obtained near $\alpha=0.93$; values that are markedly smaller or larger degrade performance modestly, but Kourkoutas-$\beta$ maintains its edge over the Adam variants in all cases.
 
\begin{table}[H]
\centering
\caption{Effect of $\alpha$ on Transformer (Heat2D), seed $=0$. Losses at epoch $=100$.}
\label{tab:alpha-ablation}
\begin{tabular}{c c c}
\toprule
$\alpha$ & Training Loss & Validation Loss \\
\midrule
0.85 & $2.2634 \times 10^{-6}$ & $2.4053 \times 10^{-6}$ \\
0.90 & $2.2792 \times 10^{-6}$ & $2.3993 \times 10^{-6}$ \\
0.93 & $\mathbf{1.7639 \times 10^{-6}}$ & $\mathbf{1.8462 \times 10^{-6}}$ \\
0.98 & $2.1601 \times 10^{-6}$ & $2.2429 \times 10^{-6}$ \\
\bottomrule
\end{tabular}
\end{table}

Table\ref{tab:beta2min-ablation} summarizes the ablation results of $\beta_{2,\min}$ for the same case. The best validation loss is obtained near $\beta_{2,\min}=0.88$ which is the default setting. Performance degrades only modestly away from this setting 
and in all cases Kourkoutas-$\beta$ continues to outperform both Adam variants by a significant margin.  

\begin{table}[H]
\centering
\caption{$\beta_{2,\min}$ ablation results for Transformer (Heat2D), Seed = 0. 
\label{tab:beta2min-ablation}
Losses are reported at epoch  = 100.}
\begin{tabular}{c c c}
\toprule
$\beta_{2,\min}$ & Training Loss & Validation Loss \\
\midrule
0.85 & $2.2271 \times 10^{-6}$ & $2.3600 \times 10^{-6}$ \\
0.88 & $\mathbf{1.7639} \times 10^{-6}$ & $\mathbf{1.8462} \times 10^{-6}$ \\
0.90 & $1.8978 \times 10^{-6}$ & $2.0052 \times 10^{-6}$ \\
0.93 & $1.9852 \times 10^{-6}$ & $2.0999 \times 10^{-6}$ \\
\bottomrule
\end{tabular}
\end{table}

\paragraph{Takeaway} 
The $\alpha$ and $\beta_{2,\min}$ sweeps show modest sensitivity to either parameter, 
with shallow optima near $\alpha \!\approx\! 0.93$ and $\beta_{2,\min} \!\approx\! 0.88$. 
Across the tested ranges $\alpha \in [0.85, 0.98]$ and $\beta_{2,\min} \in [0.85, 0.93]$, 
Kourkoutas--$\beta$ remains robust and continues to outperform Adam under matched settings.

\subsection{Reproducibility and artifacts}
We release training harnesses, configs, and scripts for both testbeds, together with snapshot helpers to log sunspike and \(\beta_2\) per layer (\lstinline|snapshot_sunspike_history()|). The optimizer is exactly the \lstinline|KourkoutasSoftmaxFlex| class described in Sections~\S\ref{sec:overview}–\S\ref{sec:practice}; MLX Adam uses bias correction on in the main comparisons (Transformer/PINN/Testbed C); bias correction is off only in the toy \emph{sanity} checks.

\emph{All baselines use MLX~\texttt{v0.26.3} (pinned in each repository’s \texttt{pyproject.toml}); the full, pinned environment is enumerated in~\ref{app:env}.}

\paragraph{Code availability}
We release the optimizer (\repoKbeta) and all testbeds (\repoPINN, \repoTrans) under open‑source MIT licenses. Exact package pins (including \texttt{mlx==0.26.3}) and scripts are included to reproduce all tables and figures.
The code for the Testbed~C (Length--Jitter + Rare Trigger)  is listed in \ref{app:trigger-code}. The full training script for Testbed~D (\texttt{testbed\_d.py}) is included  under ``/examples'' in \repoKbeta{} (versioned release: \doiKbeta).

\subsection{Discussion}
Kourkoutas-\(\beta\) preserves Adam’s simplicity while adding a layer-wise, bounded adaptation of \(\beta_2\) that is well-suited to deterministic yet heterogeneous PDE workloads. It is drop-in compatible, requires no schedule changes, and its overhead is small in practice. While plain Adam remains a strong default on large, well-conditioned tasks, we find Kourkoutas-\(\beta\) advantageous when gradient norms are intermittently spiky due to boundary/initial condition heterogeneity or stiff composite losses.

\section{Analyzing Convergence: A Step-by-Step Proof Sketch}
\emph{Regret} is a measure of how much extra cost an algorithm incurs compared 
to the best fixed decision in hindsight. If you’re navigating an unfamiliar city 
without a map, regret is the extra distance you traveled because you didn’t 
know the best route upfront. Formally, in an online optimization setting:
\[
  R(T) 
  \;=\;
  \sum_{t=1}^{T} f_t(x_t)
  \;-\;
  \min_{x \,\in\,\mathcal{X}} \sum_{t=1}^{T} f_t(x).
\]
Sublinear regret, \(R(T) = o(T)\), implies the \emph{average} regret 
\(R(T)/T\) goes to zero, meaning the algorithm effectively learns 
an action that is nearly as good as the single best decision in hindsight.

\subsection*{Why This Matters for Adam}
In proofs for Adam-like methods, we typically show \(\sum_{t=1}^T f_t(\theta_t) - f_t(\theta^*)\) 
is sublinear when \(\|g_t\|\) is bounded and the second-moment estimate 
remains well-behaved. For Kourkoutas-\(\beta\), we ensure the second moment
$\,v_t\,$ remains bounded \emph{even though} $\beta_2$ changes per iteration, 
thus preserving sublinear regret.

We now integrate our Kourkoutas-\(\beta\) approach into an Adam-style analysis, 
showing that sublinear regret or diminishing gradient is still achievable. 
We adapt the bounding steps from \cite{kingma2015adam} but let \(\beta_2\) vary 
each iteration in the range \([\beta_{2,\min}, \beta_{2,\max}]\).

\subsection{Recurrence and Worst-Case Bound on \texorpdfstring{$v_t$}{v\_t}}
If the second moment is updated as
\[
v_{t+1}
=
\beta_{2,t}\,v_t
+\bigl(1-\beta_{2,t}\bigr)\,g_t^2,
\]
and \(\|g_t\|\le G\), then unrolling yields:
\[
v_{t+1}
=\;
v_1 \prod_{j=2}^{t+1}\beta_{2,j}
+\sum_{i=1}^{t}
\bigl(1-\beta_{2,i}\bigr)
\bigl(\!\!\ \prod_{j=i+1}^{t+1}\,\, \!\!\!\!\beta_{2,j}\bigr)\,\|g_i\|^2.
\]
Take the worst-case: \(\beta_{2,i} \ge \beta_{2,\min}\) means 
\((1-\beta_{2,i}) \le (1-\beta_{2,\min})\), and 
\(\prod_{j=i+1}^{t+1}\beta_{2,j}\le \beta_{2,\max}^{\,t-i}\). Hence,
\[
v_{t+1}
\;\le\;
v_1\,\beta_{2,\max}^{\,t}
\;+\;(1-\beta_{2,\min})\,\sum_{i=1}^t \|g_i\|^2 \,\beta_{2,\max}^{\,t-i}.
\]
With $\|g_i\|\le G$ and a geometric series sum
\(\sum_{k=0}^{t-1} \beta_{2,\max}^k 
= \frac{1-\beta_{2,\max}^t}{1-\beta_{2,\max}},\)
we get
\[
v_{t+1}
\;\le\;
v_1\,\beta_{2,\max}^t
+\bigl(1-\beta_{2,\min}\bigr)G^2
\frac{1-\beta_{2,\max}^t}{1-\beta_{2,\max}}.
\]
As $t\to\infty$, if $\beta_{2,\max}<1$, 
$\beta_{2,\max}^t\to 0$, so
\[
v_{t+1} 
= 
O\Bigl(\tfrac{(1-\beta_{2,\min})}{1-\beta_{2,\max}}\,G^2\Bigr).
\]
Thus $v_t$ remains bounded by a factor involving $(\beta_{2,\min}, \beta_{2,\max})$.

\subsection{From Bounded \texorpdfstring{$v_t$}{v\_t} to Sublinear Regret}
Standard Adam proofs (e.g.\ \cite{kingma2015adam}) show sublinear regret or diminishing 
$\|\nabla f(\theta)\|$ if: 
\begin{enumerate}
    \item $\rho_t$ (learning rate) decays at a suitable rate, e.g.\ $\rho_t\sim 1/\sqrt{t}$,
    \item $v_t$ remains in a reasonable range, thus $\frac{1}{\sqrt{v_t}}$ 
          is not too large nor too small.
\end{enumerate}
Because Kourkoutas-\(\beta\) ensures a geometric weighting (bounded by $\beta_{2,\max}^t$ 
above), we inherit Adam’s sublinear or no-worse convergence rate, up to factors
depending on $(1-\beta_{2,\min})/(1-\beta_{2,\max})$. Hence the dynamic 
\(\beta_2\) does \emph{not} break the essential bounding steps that yield 
Adam’s typical $O(\sqrt{T})$ regret or diminishing gradient norms in nonconvex settings.

\section{Conclusion and Future Work}
Kourkoutas-\(\beta\) is a drop‑in Adam variant that dynamically modulates the second‑moment discount \(\beta_2\) \emph{per layer} using a bounded “sunspike” signal (current pooled gradient norm relative to its EMA). With \(\beta_2\) constrained to \([\beta_{2,\min},\beta_{2,\max}]\subset(0,1)\), a simple geometric bound preserves Adam‑style guarantees (sublinear regret / diminishing gradients) while enabling rapid reaction to bursty gradients. Empirically, Kourkoutas-\(\beta\) stabilizes training and lowers final loss on PDE surrogates, stiff PINNs, quantization‑aware training, and attention models with large, sporadic gradients, and we observe the same pattern on a standard character‑level \texttt{enwik8} task.

On \texttt{small‑enwik8} (10 seeds), Kourkoutas‑$\beta$ reduces final BPC by  $\sim38\%$  vs.\ Adam–0.95  and $\sim58\%$vs.\ Adam–0.999, with universal per‑seed wins and much tighter seed dispersion.
This, together with the PDE results, suggests that per‑layer dynamic $\beta_2$ is a broadly useful lever whenever gradient scales are bursty or regime‑shifting.

\textbf{Future work.} We will study refined scheduling and partial momentum adaptation (e.g., dynamic \(\beta_1\)), scale experiments to multi‑physics PDE settings, and quantify compute/memory overheads—especially under quantization. We will also explore LLM regimes—starting with small-batch adapter SFT and QAT, then long-context and multi-task mixtures, followed by preference learning (DPO/IPO), RLHF/RLAIF (with our trust-region option), and finally MoE and full pretraining.

It is neither feasible nor desirable for a single group to exhaustively validate Kourkoutas‑$\beta$ across the full spectrum of architectures, data regimes, and training protocols where Adam‑style methods are used. Given this diversity, it is reasonable to expect clear gains in some settings and little to no advantage in others. Our aim here is to establish sufficient promise, both empirical and theoretical, to warrant releasing Kourkoutas‑$\beta$ as an open‑source, installable package accompanied by reproducible testbeds, so the community can exercise it in new contexts, map its strengths and limitations, and iterate on the design. Consistent with this intent, we provide a permissive open‑source release and pinned environments to encourage replication, benchmarking, and improvements by others.

\section*{Acknowledgments}
We thank the original Adam authors \cite{kingma2015adam} for the foundational ideas 
and the entire PDE-solver community for pushing 
practical, robust optimization algorithms. This work is a small enhancement to their original ideas 
that we hope will be relevant to some challenging gradient regimes as discussed. We also thank the 
authors and contributors of the MLX Array Framework \cite{MLX_GitHub} for providing an open-source 
framework that has transformed GPU computing on Apple M series computers.

\appendix
\clearpage


\section{\bf Formal Convergence Discussion for Kourkoutas--$\beta$}
\label{appendix:theorem-style}
\noindent\emph{Context.} This appendix restates the convergence sketch in a theorem--lemma style, matching the notation used in the main text (\S2--\S3). Even though $\beta_2$ varies per iteration and per bucket with $\beta_{2,t}^{(\ell)} \in [\beta_{2,\min},\beta_{2,\max}] \subset (0,1)$, the second moment remains bounded, which preserves Adam‑style sublinear‑regret / diminishing‑gradient guarantees under standard assumptions. (See also the algorithmic description in the main text, Alg.~\ref{alg:Algorithms}.) 

\medskip
\medskip
\noindent{\bf $\bullet$ Setup and Assumptions}

We consider a sequence of functions \(\{f_t\}\) in an online optimization
framework or a single function \(f\) in a batch setting with mini-batches
indexed by \(t\).  In the online setting:

\[
  R(T) 
  \;=\;
  \sum_{t=1}^{T} f_t(\theta_t)
  \;-\;
  \min_{x \,\in\,\mathcal{X}} \sum_{t=1}^{T} f_t(x).
\]

\begin{assumption}[Bounded Gradients]
\label{assump:bounded-grad}
There exists a constant \(G>0\) such that
\(\|\mathbf{g}_t\|\;\le\; G\) for all \(t\), 
where \(\mathbf{g}_t = \nabla f_t(\theta_{t-1})\).
\end{assumption}


\begin{assumption}[Step sizes and momentum bounds]
\label{assump:lr-bounds}
Let $\beta_{2,t}\in[\beta_{2,\min},\beta_{2,\max}]\subset(0,1)$ with $\beta_{2,\max}<1$ for all $t$,
and fix $\beta_1\in[0,1)$. The optimizer step size is denoted by $\rho_t>0$ and
decays as $\rho_t \asymp t^{-1/2}$ (e.g., $\rho_t=\rho_0/\sqrt{t}$).
We reserve $\alpha\in(0,1)$ exclusively for the EMA coefficient used in the pooled
gradient-norm statistic $r_t=\alpha r_{t-1}+(1-\alpha)\|g_t\|$.
\end{assumption}

These assumptions closely mirror those used in the Adam and AdaBelief analyses
\cite{kingma2015adam,adabelief2020}, except we permit a changing \(\beta_2\).

\subsection*{Bounding the Second-Moment Term}

\begin{lemma}[Bound on the Second Moment $v_t$]
\label{lemma:vt-bound}
Suppose Assumption~\ref{assump:bounded-grad} holds and 
\(\beta_{2,t} \in [\beta_{2,\min}, \beta_{2,\max}]\subset(0,1)\).  Then the
second-moment sequence \(v_t\) defined by
\[
  v_{t+1} 
  \;=\; 
  \beta_{2,t}\,v_t \;+\; \bigl(1-\beta_{2,t}\bigr)\, (\mathbf{g}_t \odot \mathbf{g}_t)
\]
remains bounded.  In particular,
\[
  v_{t+1} 
  \;\le\;
  v_1\,\beta_{2,\max}^{\,t}
  \;+\;(1-\beta_{2,\min})\,G^2 
    \sum_{i=1}^{t} \beta_{2,\max}^{\,t-i},
\]
and hence
\[
  v_{t+1}
  \;=\;
  O\,\!\Bigl(\frac{(1-\beta_{2,\min})}{1-\beta_{2,\max}} \cdot G^2\Bigr).
\]
\end{lemma}

\begin{proof}[Proof]
By definition,
\[
  v_{t+1}
  \;=\;
  \beta_{2,t}\,v_t 
  + (1-\beta_{2,t})\,\|\mathbf{g}_t\|^2.
\]
Applying the worst-case bounds 
\(\beta_{2,t} \le \beta_{2,\max}\) and \(1-\beta_{2,t}\ge 1-\beta_{2,\min}\),
we unroll the recurrence to get
\[
  v_{t+1}
  \;\le\;
  v_1 \prod_{j=2}^{t+1}\beta_{2,j}
  \;+\;\sum_{i=1}^{t} (1-\beta_{2,\min})\,\|\mathbf{g}_i\|^2 
    \prod_{j=i+1}^{t+1}\beta_{2,j}.
\]
Using \(\|\mathbf{g}_i\|\le G\) from Assumption~\ref{assump:bounded-grad}
and \(\beta_{2,j}\le \beta_{2,\max}\), this becomes
\[
  v_{t+1}
  \;\le\;
  v_1\,\beta_{2,\max}^{\,t}
  \;+\;
  (1-\beta_{2,\min})\,G^2
  \sum_{i=1}^t
  \beta_{2,\max}^{\,t-i}.
\]
Finally, the geometric series
\(\sum_{k=0}^{t-1}\beta_{2,\max}^k
= \frac{1-\beta_{2,\max}^t}{1-\beta_{2,\max}}\)
implies the bound
\[
  v_{t+1}
  \;\le\;
  v_1\,\beta_{2,\max}^t
  \;+\;
  \frac{(1-\beta_{2,\min})\,G^2}{1-\beta_{2,\max}}\, 
  \bigl(1-\beta_{2,\max}^t\bigr).
\]
As \(t\to\infty\), \(\beta_{2,\max}^t\to 0\) if \(\beta_{2,\max}<1\), 
so this yields
\[
  v_{t+1} 
  = 
  O\bigl(\tfrac{(1-\beta_{2,\min})}{1-\beta_{2,\max}}\,G^2\bigr).
\]
\end{proof}

\clearpage

\section{Statistical Transparency}
\label{app:stat}

\paragraph{Global conventions}
All tests are \emph{two-sided} and \emph{paired by seed} unless noted. We report
(i) a parametric paired $t$-test on seedwise differences when appropriate, with
effect sizes $d_z{=}t/\sqrt{n}$ and $r{=}\sqrt{t^2/(t^2+\mathrm{df})}$ and $95\%$ CIs; 
(ii) distribution-free complements (Wilcoxon signed-rank, exact two-sided; sign test); 
(iii) for binary success/failure, McNemar’s exact test (two‑sided) with Clopper–Pearson $95\%$ CIs on rates.
Where two planned pairwise contrasts are made within a testbed (Kourkoutas--$\beta$ vs.\ Adam--95 and vs.\ Adam--999), 
we apply \textbf{Holm} adjustment to the two $p$-values. 
Seed lists and pairing are identical across methods within each testbed.

\begin{table}[H]
\centering
\caption{Tests and rationale by testbed. Endpoints and labels match the main text/tables.}
\label{tab:stat-transparency}
\setlength{\tabcolsep}{4pt}           
\renewcommand{\arraystretch}{1.1}     
\small
\begin{tabular}{@{}l l l l l l@{}}
\toprule
\makecell[l]{\textbf{Testbed}} &
\makecell[l]{\textbf{Endpoint}} &
\makecell[l]{\textbf{Primary}\\\textbf{test}} &
\makecell[l]{\textbf{Robust /}\\\textbf{Exact}} &
\makecell[l]{\textbf{Effect}\\\textbf{size(s)}} &
\makecell[l]{\textbf{Multiplicity}} \\
\midrule
\makecell[l]{A (Heat2D)} &
\makecell[l]{Train MSE\\@ epoch 100} &
\makecell[l]{Paired $t$ on\\seedwise diffs} &
\makecell[l]{Wilcoxon;\\Sign test} &
\makecell[l]{$d_z$, $r$;\\Median [IQR]} &
\makecell[l]{Holm across\\2 contrasts} \\
\addlinespace[1pt]
\makecell[l]{B (Heat3D PINN)} &
\makecell[l]{Success @100k;\\final loss among\\successes} &
\makecell[l]{McNemar (exact)\\on success/failure} &
\makecell[l]{Clopper–Pearson CI;\\bootstrap CI for\\medians\textsuperscript{$\dagger$}} &
\makecell[l]{Success rate diff;\\exact OR (optional)} &
\makecell[l]{n/a (binary\\endpoint); Holm if\\2 contrasts} \\
\addlinespace[1pt]
\makecell[l]{C (Length--Jitter\\+ Rare Trigger)} &
\makecell[l]{Final BCE\\@ 30k steps} &
\makecell[l]{Paired $t$ on\\$\log_{10}$ loss} &
\makecell[l]{Wilcoxon (exact)\\on $\log$ loss;\\Sign test} &
\makecell[l]{Mean log‑ratio\\$\Rightarrow$ geo‑mean\\ratio; rank‑biserial $r$\\(optional)} &
\makecell[l]{Holm across\\2 contrasts} \\
\addlinespace[1pt]
\makecell[l]{D (small‑enwik8)} &
\makecell[l]{Final BPC\\@ 50k steps} &
\makecell[l]{Paired $t$ on\\seedwise diffs} &
\makecell[l]{Wilcoxon (exact);\\Sign test} &
\makecell[l]{$d_z$, $r$;\\Median [IQR]} &
\makecell[l]{Holm across\\2 contrasts} \\
\bottomrule
\end{tabular}

\vspace{0.3em}
\footnotesize \textsuperscript{$\dagger$}\,We summarize final losses for PINN only among successful runs; we report a median with a (percentile or BCa) bootstrap CI as a descriptive complement.
\end{table}

\subsection*{Why these tests}
\textbf{Testbed A (Heat2D).} Matched seeds yield paired observations. We test seedwise \emph{differences} in MSE with a paired $t$-test; 
Q–Q checks on differences were acceptable, but we also report Wilcoxon/sign as distribution‑free complements. 
Effect sizes ($d_z$, $r$) accompany CIs. See Tables~\ref{tab:transformer-abs}–\ref{tab:transformer-stats-3tests}.

\textbf{Testbed B (Heat3D PINN).} The primary outcome is binary convergence at 100k steps under identical seeds; 
McNemar’s exact test is the appropriate paired test on the success indicator. 
We show Clopper–Pearson CIs for rates, and (descriptively) medians of final loss among successes with a bootstrap CI. 
See Tables~\ref{tab:pinn3d_summary}–\ref{tab:pinn3d-mcnemar}.

\textbf{Testbed C (Length–Jitter + Rare Trigger).} Losses are heavy‑tailed with occasional “fail” seeds; 
we analyze \emph{$\log_{10}$ loss} (multiplicative effects $\Rightarrow$ additive on log‑scale). 
Primary: paired $t$ on log‑loss; secondary: Wilcoxon (exact) and sign test on the same transform. 
We report geometric‑mean loss ratios and CIs. See Tables~\ref{tab:toyrare-seeds}–\ref{tab:toyrare-stats}.

\textbf{Testbed D (small‑enwik8).} Continuous BPC per seed, paired design. 
Primary: paired $t$ on differences; secondary: Wilcoxon/sign (exact). 
We report $d_z$, $r$, and Holm‑adjusted $p$ across the two planned contrasts. See Tables~\ref{tab:enwik8-abs}–\ref{tab:enwik8-stats}.

\paragraph{Assumptions and handling rules}
(i) Seeds are independent across runs and identical across optimizers (pairing). 
(ii) All tests are two‑sided. 
(iii) Normality is only assumed for the \emph{differences} in the paired $t$-tests; we always include a distribution‑free complement. 
(iv) For PINN, medians among successes are descriptive; inference on the binary endpoint uses McNemar. 
(v) No outliers were removed; heavy tails are handled via log‑scale (Testbed C) and robust summaries (medians/IQR). 
(vi) When two contrasts are reported within a testbed, $p$-values are Holm‑adjusted.

\paragraph{Additional robustness (not used in the main text)}
Repeating Testbed~D over 20 seeds (0–19) confirmed the 10‑seed pattern:
Kourkoutas$\beta$ $1.635\pm0.027$ vs.\ Adam–95 $2.452\pm0.596$ and Adam–999 $3.768\pm0.509$ (final BPC),
with Kourkoutas‑$\beta$ winning 20/20 paired seeds against both baselines.

\clearpage
\section{Synthetic Sanity‑Check Scripts}
\label{app:synthetic-code}
All three scripts run under Python~3.11 with \texttt{mlx} v0.26.3 (or later) and complete in a few seconds on an Apple M‑series GPU. We keep bias correction \emph{off} here to match the code paths in \S4 and to make the Adam‑equivalence control exact. 
\vspace{0.5em}

\subsection{Quadratic Bowl (Strongly Convex)}
\begin{lstlisting}[language=Python,caption={\texttt{quadratic\_bowl.py}}]
# sanity1_extended.py
import time, statistics as stats
import mlx.core as mx
from kbeta.optim import KourkoutasSoftmaxFlex as Kbeta
from mlx.optimizers import Adam as MLXAdam  # MLX baseline 

# ------------------------
# Data + loss/grad helpers
# ------------------------
def make_data(n=1000, d=100, seed=0):
    mx.random.seed(seed)
    X = mx.random.normal((n, d))
    w_star = mx.random.normal((d,))
    y = X @ w_star + 0.01 * mx.random.normal((n,))
    return X, y

def loss_and_grad(X, y):
    def loss_fn(w):
        r = X @ w - y
        return 0.5 * mx.mean(mx.square(r))
    return loss_fn, mx.grad(loss_fn)

def init_params(d, seed):
    mx.random.seed(seed)
    return {"w": 0.01 * mx.random.normal((d,))}

# ------------------------
# Optimizer factories
# ------------------------
def make_opt_kbeta_dynamic():
    # Kbeta with dynamic beta2 (matches paper: bias_correction="none")
    return Kbeta(
        learning_rate=1e-3,   # rho
        beta1=0.9,
        beta2_max=0.999, beta2_min=0.88,
        eps=1e-8,
        alpha=0.95,           # EMA for pooled grad norm
        decay=None, adaptive_tiny=False, max_ratio=None,
        bias_correction="none",   # BC off
        warmup_steps=0,
        layer_key_fn=lambda _: 0, # single global bucket
        diagnostics=False,
    )

def make_opt_kbeta_fixed_no_bc():
    # Kbeta configured to emulate Adam with fixed beta2 and BC off
    return Kbeta(
        learning_rate=1e-3,
        beta1=0.9,
        beta2_max=0.999, beta2_min=0.999,  # fixed beta2
        eps=1e-8,
        alpha=0.95,
        decay=None, adaptive_tiny=False, max_ratio=None,
        bias_correction="none",            # BC off
        warmup_steps=0,
        layer_key_fn=lambda _: 0,
        diagnostics=False,
    )

def make_opt_mlx_adam():
    # MLX Adam baseline (bias correction off by default)
    return MLXAdam(
        learning_rate=1e-3,
        betas=[0.9,0.999],
        eps=1e-8,
        # AMSGrad off by default
    )

# ------------------------
# Step adapter (handles both APIs, incl. in-place updates)
# ------------------------
def apply_step(opt, params, grads):
    if hasattr(opt, "apply_gradients"):
        # Kbeta (subclass) returns updated params
        return opt.apply_gradients(grads, params)
    if hasattr(opt, "update"):
        out = None
        try:
            out = opt.update(params, grads)
        except TypeError:
            # Some versions use (grads, params)
            out = opt.update(grads, params)
        # If MLX mutates in-place and returns None, keep original dict
        return params if out is None else out
    raise RuntimeError("Unknown optimizer interface")

# ------------------------
# Single run
# ------------------------
def run_once(opt, X, y, init_w, steps=10_000, warmup=200):
    params = {"w": init_w + mx.zeros_like(init_w)}  # identical start per arm
    loss_fn, grad_w = loss_and_grad(X, y)

    # warm-up (untimed; helps JIT/stabilize kernels)
    for _ in range(warmup):
        g = {"w": grad_w(params["w"])}
        params = apply_step(opt, params, g)
        mx.eval(params["w"])

    # timed loop
    t0 = time.perf_counter()
    for _ in range(steps):
        g = {"w": grad_w(params["w"])}
        params = apply_step(opt, params, g)
        mx.eval(params["w"])
    t1 = time.perf_counter()

    final_loss = float(loss_fn(params["w"]).item())
    return final_loss, (t1 - t0)

# ------------------------
# Median benchmark
# ------------------------
def median_benchmark(runs=10, steps=10_000, base_seed=0):
    X, y = make_data(seed=0)  # fixed data for all runs
    d = X.shape[1]

    benches = [
        ("K-beta (dynamic beta2, bc=off)", make_opt_kbeta_dynamic),
        ("K-beta (fixed beta2, bc=off)",   make_opt_kbeta_fixed_no_bc),
        ("MLX Adam (bc=off)",        make_opt_mlx_adam),
    ]

    results = {name: {"loss": [], "time": []} for name, _ in benches}

    for r in range(runs):
        # Per-run identical initialization for all arms
        init_w = init_params(d, seed=base_seed + r)["w"] 
        # Alternate order each repeat to avoid cache/thermal bias
        order = benches if (r % 2 == 0) else list(reversed(benches))
        for name, maker in order:
            opt = maker()
            loss, t = run_once(opt, X, y, init_w, steps=steps, warmup=200)
            results[name]["loss"].append(loss)
            results[name]["time"].append(t)

    # Report medians
    for name, _ in benches:
        med_loss = stats.median(results[name]["loss"])
        med_time = stats.median(results[name]["time"])
        print(f"{name:26s}  loss median: {med_loss:.6e}  time median: {med_time:.3f}s")

if __name__ == "__main__":
    median_benchmark()
\end{lstlisting}

\subsection{Logistic Regression (Convex)}
\begin{lstlisting}[language=Python,caption={\texttt{logistic\_regression.py}}]
# sanity2_extended.py
import time, statistics as stats
import mlx.core as mx
from kbeta.optim import KourkoutasSoftmaxFlex as Kbeta
from mlx.optimizers import Adam as MLXAdam  # MLX baseline (BC off by default)

# ------------------------
# Data + loss/grad helpers
# ------------------------
def make_data(n=1000, d=100, seed=0):
    mx.random.seed(seed)
    X = mx.random.normal((n, d))
    w_true = mx.random.normal((d,))
    logits_true = X @ w_true
    y = (logits_true > 0).astype(X.dtype)  # float {0,1}
    return X, y

def loss_and_grad(X, y):
    # Binary logistic (sigmoid cross-entropy): mean( log(1+exp(z)) - y*z )
    def loss_fn(w):
        z = X @ w
        return mx.mean(mx.logaddexp(0.0, z) - y * z)
    return loss_fn, mx.grad(loss_fn)

def init_params(d, seed):
    mx.random.seed(seed)
    return {"w": 0.01 * mx.random.normal((d,))}

def accuracy(X, w, y):
    z = X @ w
    preds = (z > 0)
    y_bool = (y > 0.5)
    return float(mx.mean((preds == y_bool).astype(mx.float32)).item())

# ------------------------
# Optimizer factories
# ------------------------
def make_opt_kbeta_dynamic():
    # Kbeta with dynamic beta2 (matches paper: bias_correction="none")
    return Kbeta(
        learning_rate=1e-2,   # rho (logistic benefits from a larger step)
        beta1=0.9,
        beta2_max=0.999, beta2_min=0.88,
        eps=1e-8,
        alpha=0.95,           # EMA for pooled grad norm
        decay=None, adaptive_tiny=False, max_ratio=None,
        bias_correction="none",   # BC off
        warmup_steps=0,
        layer_key_fn=lambda _: 0, # single global bucket
        diagnostics=False,
    )

def make_opt_kbeta_fixed_no_bc():
    # Kbeta configured to emulate Adam with fixed beta2 and BC off
    return Kbeta(
        learning_rate=1e-2,
        beta1=0.9,
        beta2_max=0.999, beta2_min=0.999,  # fixed beta2
        eps=1e-8,
        alpha=0.95,
        decay=None, adaptive_tiny=False, max_ratio=None,
        bias_correction="none",            # BC off
        warmup_steps=0,
        layer_key_fn=lambda _: 0,
        diagnostics=False,
    )

def make_opt_mlx_adam():
    # MLX Adam baseline (bias correction off by default)
    return MLXAdam(
        learning_rate=1e-2,
        betas=[0.9, 0.999],
        eps=1e-8,
        # AMSGrad off by default
    )

# ------------------------
# Step adapter (handles both APIs, incl. in-place updates)
# ------------------------
def apply_step(opt, params, grads):
    if hasattr(opt, "apply_gradients"):
        # Kbeta (subclass) returns updated params
        return opt.apply_gradients(grads, params)
    if hasattr(opt, "update"):
        out = None
        try:
            out = opt.update(params, grads)
        except TypeError:
            # Some versions use (grads, params)
            out = opt.update(grads, params)
        # If MLX mutates in-place and returns None, keep original dict
        return params if out is None else out
    raise RuntimeError("Unknown optimizer interface")

# ------------------------
# Single run
# ------------------------
def run_once(opt, X, y, init_w, steps=20_000, warmup=200):
    params = {"w": init_w + mx.zeros_like(init_w)}  # identical start per arm
    loss_fn, grad_w = loss_and_grad(X, y)

    # warmup (untimed; helps JIT/stabilize kernels)
    for _ in range(warmup):
        g = {"w": grad_w(params["w"])}
        params = apply_step(opt, params, g)
        mx.eval(params["w"])

    # timed loop
    t0 = time.perf_counter()
    for _ in range(steps):
        g = {"w": grad_w(params["w"])}
        params = apply_step(opt, params, g)
        mx.eval(params["w"])
    t1 = time.perf_counter()

    final_loss = float(loss_fn(params["w"]).item())
    final_acc  = accuracy(X, params["w"], y)
    return final_loss, final_acc, (t1 - t0)

# ------------------------
# Median benchmark
# ------------------------
def median_benchmark(runs=5, steps=20_000, base_seed=0):
    X, y = make_data(seed=0)  # fixed data for all runs
    d = X.shape[1]

    benches = [
        ("Kbeta (dynamic beta2, bc=off)", make_opt_kbeta_dynamic),
        ("Kbeta (fixed beta2, bc=off)",   make_opt_kbeta_fixed_no_bc),
        ("MLX Adam (bc=off)",        make_opt_mlx_adam),
    ]

    results = {name: {"loss": [], "acc": [], "time": []} for name, _ in benches}

    for r in range(runs):
        # Per-run identical initialization for all arms
        init_w = init_params(d, seed=base_seed + r)["w"]
        # Alternate order each repeat to avoid cache/thermal bias
        order = benches if (r % 2 == 0) else list(reversed(benches))
        for name, maker in order:
            opt = maker()
            loss, acc, t = run_once(opt, X, y, init_w, steps=steps, warmup=200)
            results[name]["loss"].append(loss)
            results[name]["acc"].append(acc)
            results[name]["time"].append(t)

    # Report medians
    for name, _ in benches:
        med_loss = stats.median(results[name]["loss"])
        med_acc  = stats.median(results[name]["acc"])
        med_time = stats.median(results[name]["time"])
        print(f"{name:26s}  loss median: {med_loss:.6e}  acc median: {med_acc:.3f}  time median: {med_time:.3f}s")

if __name__ == "__main__":
    median_benchmark()
\end{lstlisting}

\subsection{Concave Log‑Likelihood Ascent}
\begin{lstlisting}[language=Python,caption={\texttt{concave\_utility\_ascent.py}}]
# sanity3_extended.py
import time, statistics as stats
import mlx.core as mx
from kbeta.optim import KourkoutasSoftmaxFlex as kourkoutas
from mlx.optimizers import Adam as MLXAdam  # MLX baseline (BC off by default)

# ------------------------
# Data: separable binary labels, utility = logistic log-likelihood
# ------------------------
def make_data(n=1000, d=100, seed=0):
    mx.random.seed(seed)
    X = mx.random.normal((n, d))
    w_true = mx.random.normal((d,))
    y = (X @ w_true) > 0  # boolean labels in {False, True}
    # Convert to {-1, +1} stype for logits = s * (X @ w)
    one = mx.array(1.0, dtype=X.dtype)
    zero = mx.array(0.0, dtype=X.dtype)
    s = 2.0 * mx.where(y, one, zero) - 1.0
    return X, y, s

def utility_and_grad(X, y, s):
    # utility(w) = mean(log sigma(s * Xw)) = - mean(log(1 + exp(-s * Xw)))
    def utility(w):
        z = s * (X @ w)
        return -mx.mean(mx.logaddexp(0.0, -z))
    def loss_fn(w):
        return -utility(w)  # minimize -utility == maximize utility
    return utility, loss_fn, mx.grad(loss_fn)

def init_params(d, seed):
    mx.random.seed(seed)
    return {"w": 0.01 * mx.random.normal((d,))}

# ------------------------
# Optimizer factories (bc=off everywhere for fairness)
# ------------------------
def make_opt_kbeta_dynamic():
    return kourkoutas(
        learning_rate=5e-2,     # rho
        beta1=0.9,
        beta2_max=0.999, beta2_min=0.88,   # dynamic beta2
        eps=1e-8,
        alpha=0.95,                         # EMA for pooled grad norm
        decay=None, adaptive_tiny=False, max_ratio=None,
        bias_correction="none",             # BC off
        warmup_steps=0,
        layer_key_fn=lambda _: 0,           # single global bucket
        diagnostics=False,
    )

def make_opt_kbeta_fixed_no_bc():
    return kourkoutas(
        learning_rate=5e-2,
        beta1=0.9,
        beta2_max=0.999, beta2_min=0.999,  # fixed beta2
        eps=1e-8,
        alpha=0.95,
        decay=None, adaptive_tiny=False, max_ratio=None,
        bias_correction="none",            # BC off
        warmup_steps=0,
        layer_key_fn=lambda _: 0,
        diagnostics=False,
    )

def make_opt_mlx_adam():
    # MLX Adam baseline (bias correction off by default)
    return MLXAdam(learning_rate=5e-2, betas=[0.9, 0.999], eps=1e-8)

# ------------------------
# Step adapter (handles both APIs)
# ------------------------
def apply_step(opt, params, grads):
    if hasattr(opt, "apply_gradients"):  # Kbeta
        return opt.apply_gradients(grads, params)
    if hasattr(opt, "update"):           # MLX Adam
        out = None
        try:
            out = opt.update(params, grads)
        except TypeError:
            out = opt.update(grads, params)
        return params if out is None else out
    raise RuntimeError("Unknown optimizer interface")

# ------------------------
# Single run
# ------------------------
def run_once(opt, X, y, s, init_w, steps=50_000, warmup=200):
    params = {"w": init_w + mx.zeros_like(init_w)}  # identical start per arm
    utility, loss_fn, grad_w = utility_and_grad(X, y, s)

    # warmup (untimed; helps JIT/stabilize kernels)
    for _ in range(warmup):
        g = {"w": grad_w(params["w"])}
        params = apply_step(opt, params, g)
        mx.eval(params["w"])

    # timed loop
    t0 = time.perf_counter()
    for _ in range(steps):
        g = {"w": grad_w(params["w"])}
        params = apply_step(opt, params, g)
        mx.eval(params["w"])
    t1 = time.perf_counter()

    # metrics
    u = float(utility(params["w"]).item())
    loss = float(loss_fn(params["w"]).item())
    pred = (X @ params["w"]) > 0
    one = mx.array(1.0, dtype=X.dtype)
    zero = mx.array(0.0, dtype=X.dtype)
    acc = float(mx.mean(mx.where(pred == y, one, zero)).item())
    return u, loss, acc, (t1 - t0)

# ------------------------
# Median benchmark
# ------------------------
def median_benchmark(runs=5, steps=50_000, base_seed=0):
    X, y, s = make_data(seed=0)  # fixed data for all runs
    d = X.shape[1]
    benches = [
        ("rho (dynamic beta2, bc=off)", make_opt_kbeta_dynamic),
        ("rho (fixed beta, bc=off)",   make_opt_kbeta_fixed_no_bc),
        ("MLX Adam (bc=off)",        make_opt_mlx_adam),
    ]
    results = {name: {"u": [], "loss": [], "acc": [], "time": []} for name, _ in benches}

    for r in range(runs):
        init_w = init_params(d, seed=base_seed + r)["w"]  # identical init per run
        # Alternate order each repeat to reduce cache/thermal bias
        order = benches if (r % 2 == 0) else list(reversed(benches))
        for name, maker in order:
            opt = maker()
            u, l, acc, t = run_once(opt, X, y, s, init_w, steps=steps, warmup=200)
            results[name]["u"].append(u)
            results[name]["loss"].append(l)
            results[name]["acc"].append(acc)
            results[name]["time"].append(t)

    # Report medians
    for name, _ in benches:
        med_u = stats.median(results[name]["u"])
        med_l = stats.median(results[name]["loss"])
        med_a = stats.median(results[name]["acc"])
        med_t = stats.median(results[name]["time"])
        print(f"{name:26s}  util median: {med_u:.6e}  loss median: {med_l:.6e}  acc median: {med_a:.3f}  time median: {med_t:.3f}s")

if __name__ == "__main__":
    median_benchmark()

\end{lstlisting}

\clearpage

\section{Reproducibility: rare\_trigger\_toy.py (K‑$\beta$ vs Adam baselines, Testbed 3)}
\label{app:trigger-code}
All three scripts run under Python~3.11 with \texttt{mlx} v0.26.3 (or later) and complete in a few seconds on an Apple M‑series GPU. 
\vspace{0.5em}


\begin{lstlisting}[language=Python,caption={\texttt{ rare\_trigger\_toy.py}}]
# rare_trigger_toy.py
# MLX synthetic "length-jitter + rare-trigger" toy
# - Fixes all MLX RNG calls (randint/uniform/normal use correct signatures)
# - Imports release-style Kourkoutas-beta and runs it with minimal, faithful settings
# - No external data; runs in a few seconds on Apple Silicon

from dataclasses import dataclass
import time
import math
import mlx.core as mx

# --- Optimizer import: real class (release-like) -----------------------------
try:
    # Kourkoutas-beta release class name in kbeta repo
    from kbeta.optim import KourkoutasSoftmaxFlex as Kbeta
    HAVE_KBETA = True
except Exception as e:
    print("[WARN] Could not import kbeta.optim.KourkoutasSoftmaxFlex:", e)
    print("[WARN] Falling back to MLX Adam so the script still runs.")
    from mlx.optimizers import Adam as MLXAdam
    HAVE_KBETA = False
    
HAVE_KBETA=True #(use False for an Adam run)

# --- Config ------------------------------------------------------------------
@dataclass
class Cfg:
    # data
    B: int = 64               # batch size
    lmin: int = 80            # min length
    lmax: int = 256           # max length
    vocab: int = 256          # token ids in [0, vocab-1]
    pad_id: int = 0           # padding token
    trig_id: int = 255        # "rare trigger" token id (keep within [0, vocab-1])
    p_trigger: float = 0.01   # probability a sample contains the trigger
    d_model: int = 64         # embedding dimension

    # optimization
    steps: int = 30000
    lr: float = 1e-2
    seed: int = 0

# --- RNG helpers (correct MLX signatures) ------------------------------------
def seed_all(s: int):
    mx.random.seed(int(s))

def randi(low: int, high: int, shape):
    # MLX signature: randint(low, high, shape)
    return mx.random.randint(low, high, shape)

def runif(shape, low: float = 0.0, high: float = 1.0):
    # MLX signature: uniform(shape, low=..., high=...)
    return mx.random.uniform(low=low, high=high, shape=shape)

def rnormal(shape, mean: float = 0.0, std: float = 1.0):
    # MLX signature: normal(shape); scale manually
    return mean + std * mx.random.normal(shape)

def bernoulli(shape, p: float):
    return (runif(shape) < p)

# --- Toy data generator -------------------------------------------------------
def make_batch(cfg: Cfg, step_seed: int):
    """
    Returns:
      tokens: (B, Lmax) int32
      mask:   (B, Lmax) bool (True where valid token, False padding)
      y:      (B,) float32, label = 1 if rare trigger present inside valid region else 0
    """
    seed_all(cfg.seed + step_seed)

    B, Lmax = cfg.B, cfg.lmax

    # random per-sample lengths in [lmin, lmax]
    lens = randi(cfg.lmin, cfg.lmax + 1, (B,))  # int32

    # pad mask
    arange_L = mx.arange(Lmax, dtype=mx.int32)[None, :]   # (1, L)
    mask = (arange_L < lens[:, None])                     # (B, L) bool

    # random tokens everywhere, then we will overwrite the trigger positions
    # keep tokens in [1, vocab-1] so 0 can serve as clear PAD id
    tokens = randi(1, cfg.vocab, (B, Lmax))

    # decide which sequences get the rare trigger
    has_trig = bernoulli((B,), cfg.p_trigger)  # bool
    y = has_trig.astype(mx.float32)           # labels

    # choose a valid trigger position per sample (0 .. len-1)
    # (floor(u * len)) is in [0, len-1] when len>=1; our lmin>=1
    u = runif((B,))
    trig_pos = (mx.floor(u * lens.astype(mx.float32))).astype(mx.int32)  # (B,)

    # place the trigger id at (i, trig_pos[i]) only if has_trig[i] and inside mask
    # build one-hot position per sample and gate with has_trig
    pos_oh = (arange_L == trig_pos[:, None])                    # (B, L) bool
    place_mask = pos_oh & has_trig[:, None] & mask              # (B, L) bool

    # tokens = where(place_mask, trig_id, tokens)  - keep dtypes consistent
    tokens = mx.where(place_mask, mx.array(cfg.trig_id, dtype=tokens.dtype), tokens)

    # also ensure padding is exactly pad_id
    tokens = mx.where(mask, tokens, mx.array(cfg.pad_id, dtype=tokens.dtype))
    mx.eval(tokens, mask, y)
    return tokens, mask, y

# --- Tiny model: bag-of-embeddings -> sigmoid(logit) -------------------------
def init_params(cfg: Cfg, seed_offset: int = 0):
    seed_all(cfg.seed + 1000 + seed_offset)
    E = 0.02 * rnormal((cfg.vocab, cfg.d_model))        # embedding
    w = 0.02 * rnormal((cfg.d_model, 1))                # final linear
    b = mx.zeros((1,), dtype=E.dtype)
    mx.eval(E, w, b)
    return {"E": E, "w": w, "b": b}

def forward(params, tokens, mask):
    # emb lookup
    # shape: (B, L, d)
    emb = mx.take(params["E"], tokens, axis=0)

    # mask and pool (mean over valid tokens)
    mask_f = mask.astype(emb.dtype)
    emb_sum = (emb * mask_f[..., None]).sum(axis=1)            # (B, d)
    denom = mask_f.sum(axis=1, keepdims=True) + 1e-9           # (B, 1)
    pooled = emb_sum / denom                                   # (B, d)

    # logits: (B,)
    logits = (pooled @ params["w"]).squeeze() + params["b"].squeeze()
    return logits

def bce_with_logits(logits, targets):
    # mean( log(1 + exp(logit)) - y * logit )
    return mx.mean(mx.logaddexp(0.0, logits) - targets * logits)

# --- Optimizer factories ------------------------------------------------------
def make_kbeta_release(cfg: Cfg):
    """
    Minimal release-like Kourkoutas-beta settings for a toy:
      - dynamic beta2 in [0.88, 0.999]
      - alpha=0.93 (EMA for pooled grad norm)
      - eps=1e-8
      - bias_correction='beta2max'
      - single global bucket (layer_key_fn=lambda _: 0)
    """
    return Kbeta(
        learning_rate=cfg.lr,
        beta1=0.9,
        beta2_max=0.999,
        beta2_min=0.88,
        alpha=0.93,
        eps=1e-8,
        decay=None,            # no (soft-max) AMSGrad for this toy
        max_ratio=None,        # no trust region for this toy
        adaptive_tiny=False,
        bias_correction="beta2max",
        warmup_steps=50,       # short warmup is fine here
        layer_key_fn=lambda _: 0,
        diagnostics=False,
    )

def make_adam_baseline(cfg: Cfg):
    from mlx.optimizers import Adam
    return Adam(
        learning_rate=cfg.lr,
        betas=[0.9, 0.999],
        eps=1e-8,
        # bias_correction True/False both okay; True is common default
        bias_correction=True,
    )

# --- Training harness ---------------------------------------------------------
def apply_step(opt, params, grads):
    # Kbeta returns updated params via apply_gradients(grads, params).
    if HAVE_KBETA and isinstance(opt, Kbeta):
        return opt.apply_gradients(grads, params)
    # MLX Adam may mutate in-place; return value can be None
    out = None
    try:
        out = opt.update(params, grads)
    except TypeError:
        out = opt.update(grads, params)
    return params if out is None else out

def train_once(cfg: Cfg, make_opt, label=""):
    params = init_params(cfg)
    opt = make_opt(cfg)

    # JIT warmup (not timed)
    for _ in range(10):
        toks, msk, y = make_batch(cfg, step_seed=0)
        def loss_for_grad(p):
            logits = forward(p, toks, msk)
            return bce_with_logits(logits, y)
        g = mx.grad(loss_for_grad)(params)
        params = apply_step(opt, params, g)
        mx.eval(params["E"], params["w"], params["b"])

    # Timed loop
    t0 = time.perf_counter()
    losses = []
    for step in range(cfg.steps):
        toks, msk, y = make_batch(cfg, step_seed=1 + step)
        def loss_for_grad(p):
            logits = forward(p, toks, msk)
            return bce_with_logits(logits, y)
        g = mx.grad(loss_for_grad)(params)
        params = apply_step(opt, params, g)
        # materialize a scalar for logging
        loss_val = float(loss_for_grad(params).item())
        losses.append(loss_val)
        if (step + 1) % max(1, cfg.steps // 10) == 0:
            print(f"[{label}] step {step+1:5d}/{cfg.steps}: loss={loss_val:.6f}")
    t1 = time.perf_counter()
    return params, losses, (t1 - t0)

def main():
    cfg = Cfg()
    print("> Running Kourkoutas-beta (release-like settings)..." if HAVE_KBETA
          else "> Running MLX Adam (fallback)...")

    if HAVE_KBETA:
        params, losses, secs = train_once(cfg, make_kbeta_release, label="Kbeta")
    else:
        params, losses, secs = train_once(cfg, make_adam_baseline, label="Adam")

    print(f"With seed= {cfg.seed} done in {secs:.2f}s. Final loss: {losses[-1]:.6f}")

if __name__ == "__main__":
    main()
\end{lstlisting}

\clearpage

\section{Code availability and exact commands (Testbed D)}
\label{app:codeTestbedD}

The full training script is \texttt{testbed\_d.py} in \repoKbeta{} (archived, versioned release: \doiKbeta).
All runs use \texttt{mlx==0.26.3} (pinned in the released environment files).

\paragraph{Dataset creation (deterministic)}-----

\begin{lstlisting}[language=bash]
curl -L -o enwik8.zip https://data.deepai.org/enwik8.zip
unzip -o enwik8.zip
head -c 30000000 enwik8 > small-enwik8.txt
# optional: sha256sum small-enwik8.txt
\end{lstlisting}

\paragraph{Environment (minimal)}-----

\begin{lstlisting}[language=bash]
python -m venv .venv && source .venv/bin/activate
pip install -e ".[dev]"   # pins mlx==0.26.3
\end{lstlisting}

\paragraph{Single-seed commands (reported schedule)}-----

\begin{lstlisting}[language=bash,caption={Kourkoutas-beta (Testbed D).}]
python -u testbed_d.py --text ./small-enwik8.txt \
  --steps 50001 --batch 4 --d_model 512 --n_layer 6 --n_head 8 \
  --ctx 512 --lmin 16 --lmax 512 --warmup 250 --opt kbeta \
  --layer_bucket per-array --barrier_every 100 --eval_every 500 \
  --lr 1e-3 --seed SEED --fixed_eval_seed 1234 --deterministic --compile \
  --wd 0.0 --lr_schedule "1:1e-3,30000:5e-4,40000:1e-4,60000:1e-5"
\end{lstlisting}

\begin{lstlisting}[language=bash,caption={Adam with beta2=0.95 (baseline).}]
python -u testbed_d.py --text ./small-enwik8.txt \
  --steps 50001 --batch 4 --d_model 512 --n_layer 6 --n_head 8 \
  --ctx 512 --lmin 16 --lmax 512 --warmup 250 --opt adam --adam_beta2 0.95 \
  --layer_bucket per-array --barrier_every 100 --eval_every 500 \
  --lr 1e-3 --seed SEED --fixed_eval_seed 1234 --deterministic --compile \
  --wd 0.0 --lr_schedule "1:1e-3,30000:5e-4,40000:1e-4,60000:1e-5"
\end{lstlisting}

\begin{lstlisting}[language=bash,caption={Adam with beta2=0.999 (baseline).}]
python -u testbed_d.py --text ./small-enwik8.txt \
  --steps 50001 --batch 4 --d_model 512 --n_layer 6 --n_head 8 \
  --ctx 512 --lmin 16 --lmax 512 --warmup 250 --opt adam --adam_beta2 0.999 \
  --layer_bucket per-array --barrier_every 100 --eval_every 500 \
  --lr 1e-3 --seed SEED --fixed_eval_seed 1234 --deterministic --compile \
  --wd 0.0 --lr_schedule "1:1e-3,30000:5e-4,40000:1e-4,60000:1e-5"
\end{lstlisting}

\paragraph{Seed sweep used in the paper (10 seeds: 0--9)}-----

\begin{lstlisting}[language=bash]
for s in 0 1 2 3 4 5 6 7 8 9; do
  python -u testbed_d.py --text ./small-enwik8.txt \
    --steps 50001 --batch 4 --d_model 512 --n_layer 6 --n_head 8 \
    --ctx 512 --lmin 16 --lmax 512 --warmup 250 --opt kbeta \
    --layer_bucket per-array --barrier_every 100 --eval_every 500 \
    --lr 1e-3 --seed $s --fixed_eval_seed 1234 --deterministic --compile \
    --wd 0.0 --lr_schedule "1:1e-3,30000:5e-4,40000:1e-4,60000:1e-5"
done
\end{lstlisting}

\paragraph{Notes}
(i) \texttt{--deterministic} and \texttt{--fixed\_eval\_seed} give per-seed reproducibility. 
(ii) BPC is reported on a fixed held-out batch every 500 steps; final numbers at step \texttt{50001}. 
(iii) Hardware: Apple Studio M2 Ultra (198\,GB).

\noindent\textbf{Environment and MLX pin.} All runs use \texttt{mlx==0.26.3} (see \S\ref{app:env}).  
\medskip

\noindent\textbf{Dataset checksums.} Enwik8: \texttt{2b49720e…024a8}; Small‑Enwik8: \texttt{e0152eee…298b7}.

\clearpage
\section{Training Schedules and Configurations}
\label{app:schedules}
\noindent\textbf{Learning‑rate schedule used in PINN--3D.} Identical across methods (Kourkoutas--$\beta$, Adam--0.95, Adam--0.999): cosine decay from $10^{-2}$ to $10^{-5}$ over the first $40{,}000$ steps, then constant at $10^{-5}$. 
\begin{lstlisting}[language=Python, caption={PINN--3D learning‑rate schedule (MLX).}, label={lst:lr-pinn3d}]
# ---------------- learning schedule ----------------
init_lr    = 1e-2
target_lr  = 1e-5
ramp_steps = 40_000

cosine_part   = optim.cosine_decay(init_lr, decay_steps=ramp_steps, end=target_lr)
constant_part = lambda _: target_lr

lr_schedule = optim.join_schedules([cosine_part, constant_part], [ramp_steps])
\end{lstlisting}


\clearpage

\section{$\beta_2$ and \texttt{sunspike} Heatmaps for PINN--3D}
\label{appendix:PINNheatmaps}
We use heatmaps and violins to visualize the dynamics of the bounded \texttt{sunspike} ratio and the induced $\beta_2$.
High \texttt{sunspike} ($\to 1$) indicates gradients much larger than their recent EMA and yields smaller $\beta_2$ (more agile updates); low \texttt{sunspike} ($\approx 0$) keeps $\beta_2 \approx \beta_{2,\max}$ (more smoothing). See \S5.2.5 for the discussion that these distributions concentrate around sunspike $\approx 0.3$--$0.6$ and $\beta_2 \approx 0.93$--$0.96$ with mild epoch‑wise drift.

\begin{figure*}[h]
  \centering
  \begin{subfigure}{0.48\textwidth}\includegraphics[width=\linewidth]{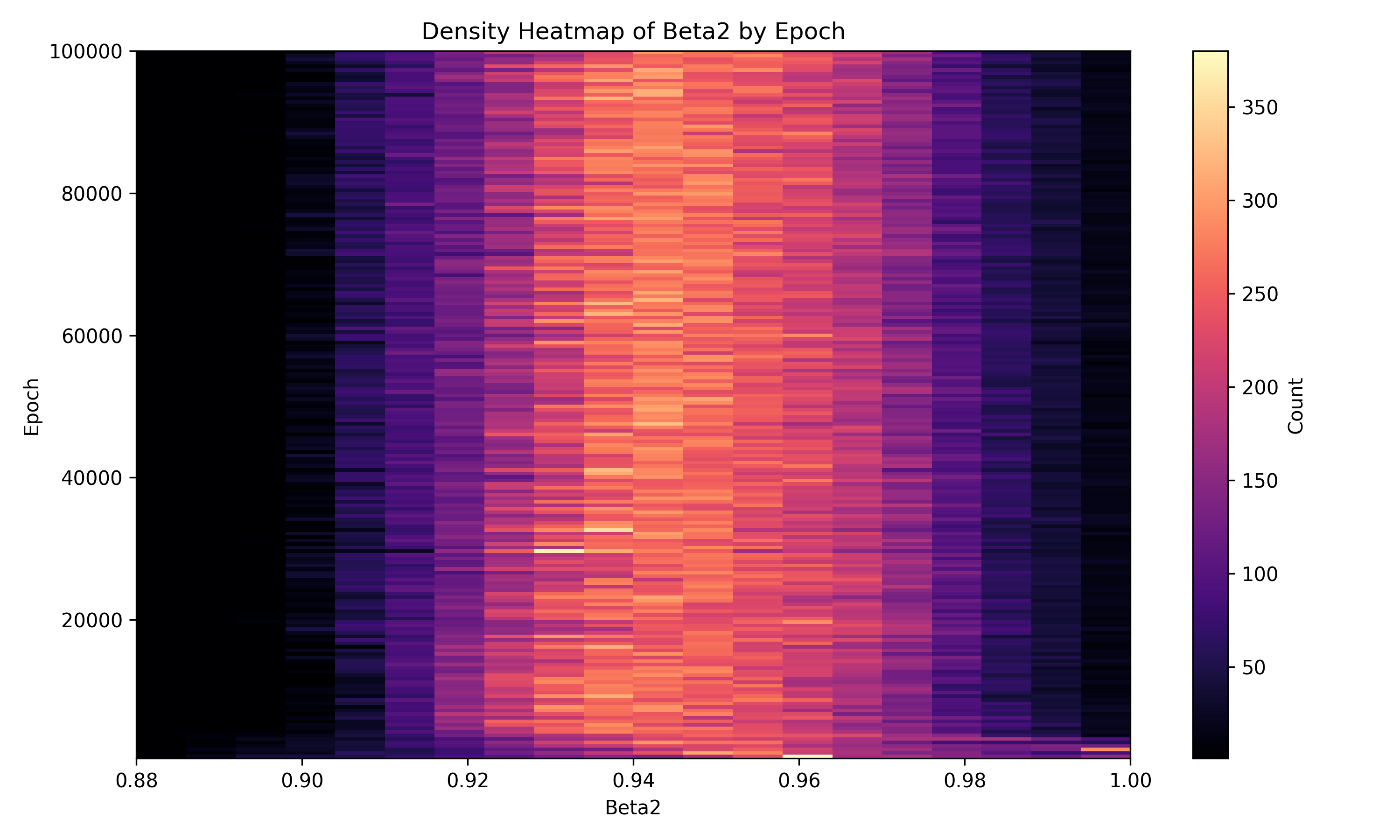}\caption{seed=0: $\beta_2$ density by epoch.}\end{subfigure}\hfill
  \begin{subfigure}{0.48\textwidth}\includegraphics[width=\linewidth]{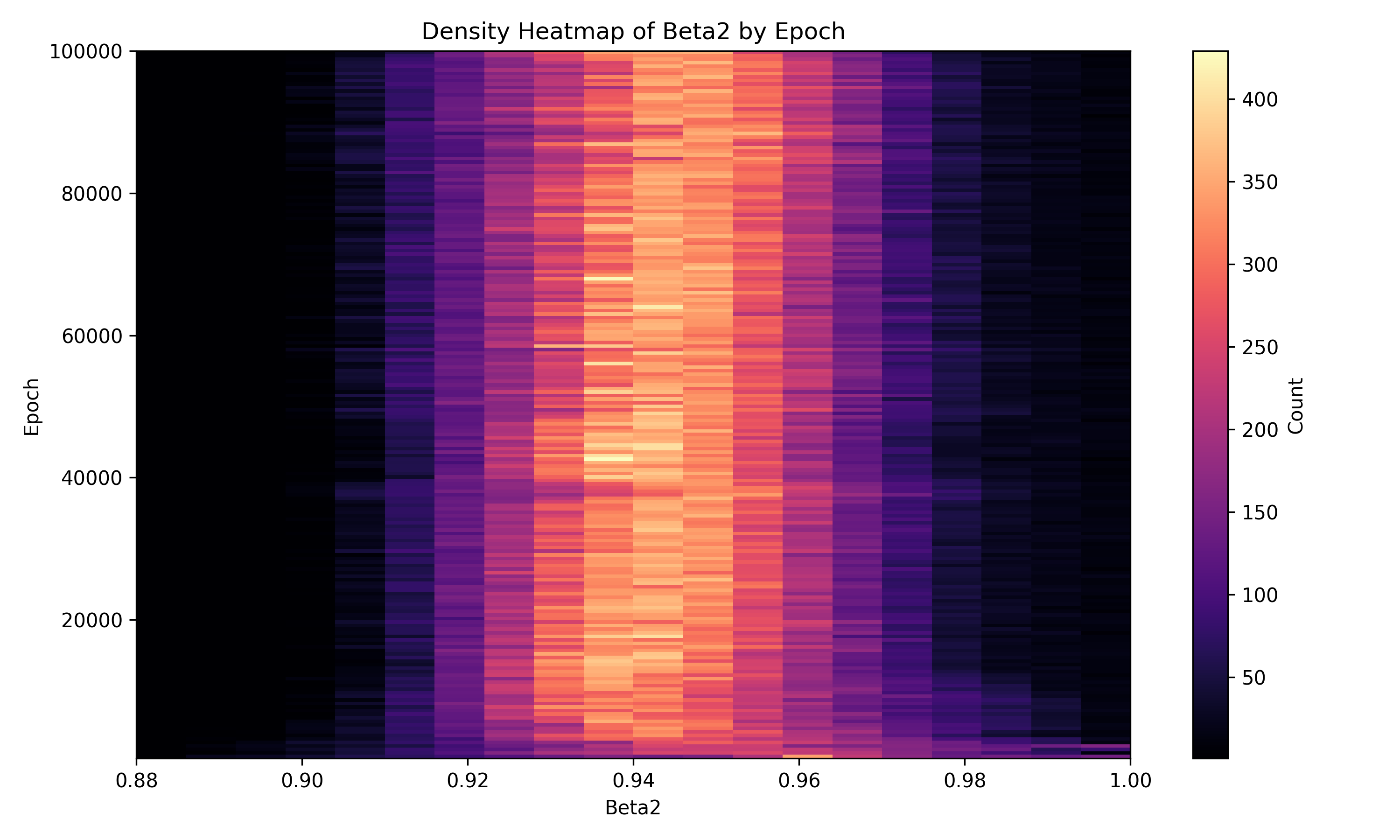}\caption{Seed 1: $\beta_2$ density by epoch.}\end{subfigure}

  \medskip

  \begin{subfigure}{0.48\textwidth}\includegraphics[width=\linewidth]{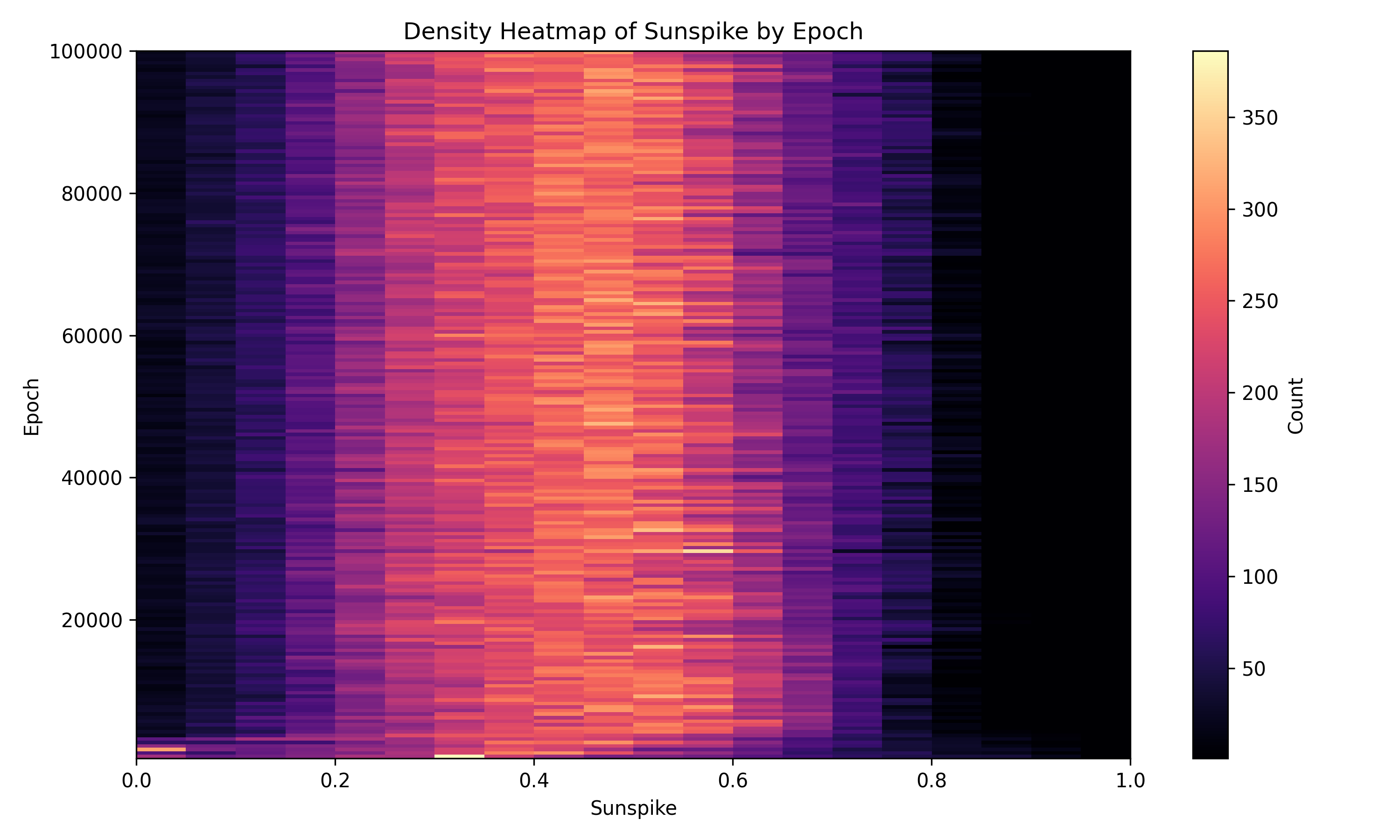}\caption{seed=0: \texttt{sunspike} density by epoch.}\end{subfigure}\hfill
  \begin{subfigure}{0.48\textwidth}\includegraphics[width=\linewidth]{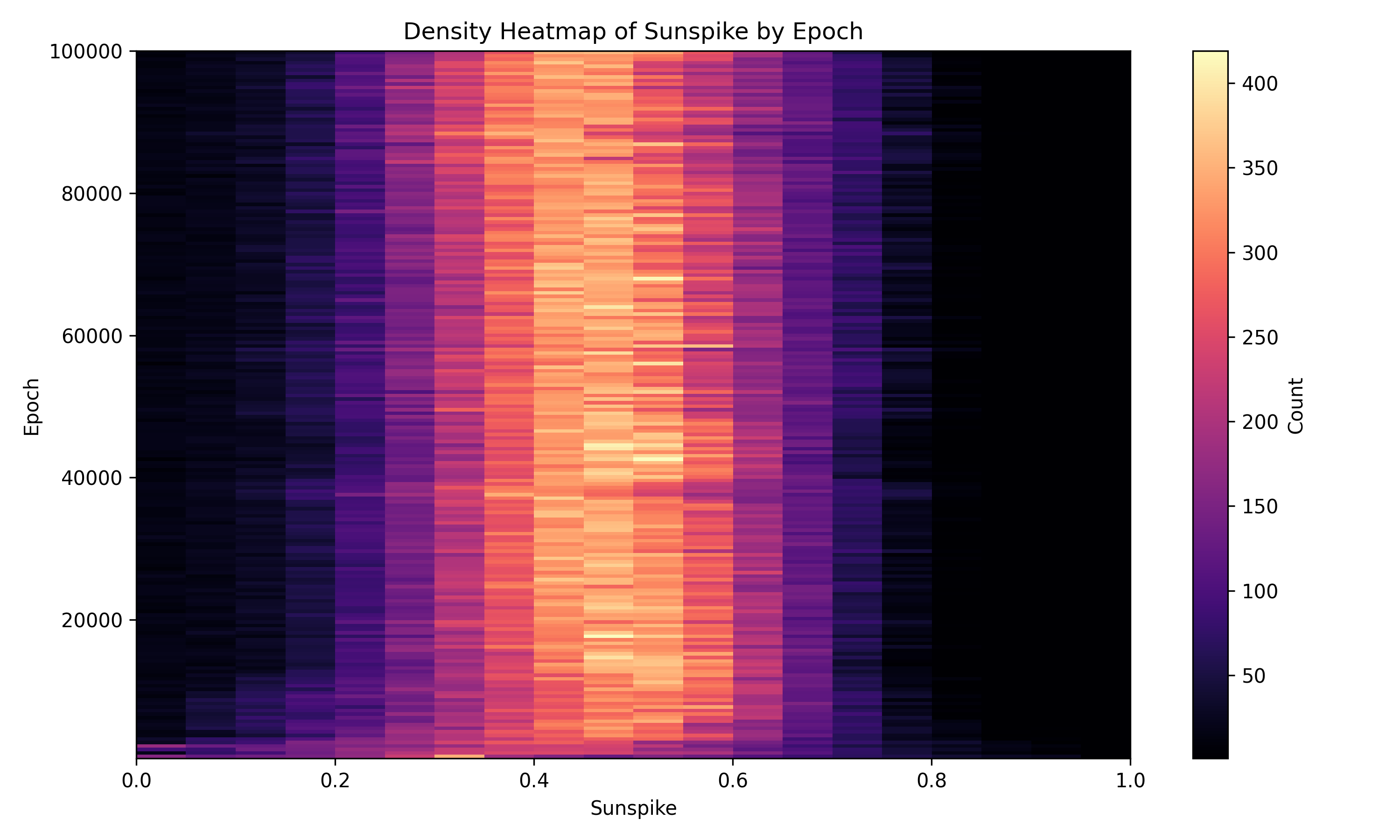}\caption{Seed 1: \texttt{sunspike} density by epoch.}\end{subfigure}

  \caption{Density heatmaps complementing the violins in \S5.2.5. High density concentrates near $\beta_2\!\approx\!0.94$--$0.96$ and \texttt{sunspike} $\approx 0.3$--$0.6$, with modest, schedule‑induced drift that is consistent across seeds.}
  \label{fig:kbeta-heatmaps}
\end{figure*}

\clearpage

\section{Reproducibility and Environment}
\label{app:env}

\noindent\textbf{Hardware.} Apple Studio M2 Ultra with 198\,GB unified memory. Wall‑clock timings are reported per epoch with diagnostics disabled and \emph{no} untimed warm‑up (averaged over the full run). 

\noindent\textbf{Key software pins.}
Python~3.11; \texttt{mlx} \textbf{v0.26.3} (all Adam baselines were run with this version due to minor version‑sensitivity in early trajectories), NumPy, Matplotlib, and any other libraries as in the provided environment file. Kourkoutas--$\beta$ is implemented in our codebase and was unaffected by MLX updates. 

\medskip
\noindent\textbf{Install (wheels).}
If you prefer prebuilt wheels (either from PyPI or the \texttt{wheels/} directory of the artifact bundle):
\begin{verbatim}
python3.11 -m venv .venv
source .venv/bin/activate
python -m pip install --upgrade pip
pip install "mlx==0.26.3"
# If/when published on PyPI:
pip install kbeta kbeta-transformer2d kbeta-pinn3d
# Or install local wheels shipped in the artifact bundle:
pip install wheels/*.whl
\end{verbatim}
\noindent This preserves the MLX pin used for all Adam baselines while allowing a one‑command install of the optimizer and testbeds.

\medskip
\noindent\textbf{Install (from source).}
To reproduce the exact development layout with the repository pins in \texttt{pyproject.toml}:
\begin{verbatim}
python3.11 -m venv .venv
source .venv/bin/activate
python -m pip install --upgrade pip
pip install "mlx==0.26.3"
# Editable installs for the optimizer and testbeds:
pip install -e ./kbeta
pip install -e ./kbeta-transformer2d
pip install -e ./kbeta-pinn3d
\end{verbatim}
\noindent Alternatively, to recreate the full environment exactly as used for tables/figures, use the shipped environment file (see artifact bundle):
\begin{verbatim}
pip install -r env/requirements.txt
\end{verbatim}

\medskip
\noindent\textbf{Artifacts.} We ship a full environment file and per‑repo \texttt{pyproject.toml} pins to reproduce tables and figures exactly. See the artifact bundle for the precise versions/hashes and training commands.

\clearpage
\section{Ablations}
\label{appendix:ablations}
The scripts run under Python~3.11 with \texttt{mlx} v0.25.0--0.28.0 and complete in a few seconds on an Apple M‑series GPU. 
\vspace{0.5em}

\subsection{Kourkoutas-$\beta$ configured as Adam with bias correction off}
\label{appendix:KbetaASAdam}
The observed per‑parameter FP32 max absolute difference satisfies $\,\le 1\times 10^{-6}\,$ across 1{,}000 steps.

\begin{lstlisting}[language=Python,caption={\texttt{ablation\_Kourkoutas\_asAdam.py}}]
# adam_equivalence_check.py
import mlx.core as mx
from mlx.optimizers import Adam as MLXAdam
from kbeta.optim import KourkoutasSoftmaxFlex as Kbeta

mx.set_default_device(mx.gpu)

mx.random.seed(0)
n, d = 128, 16
X = mx.random.normal((n, d))
y = mx.random.normal((n,))
w0 = mx.random.normal((d,)) * 0.01

def loss_fn(w):
    r = X @ w - y
    return 0.5 * mx.mean(r * r)

grad_w = mx.grad(loss_fn)

# MLX Adam (BC off by default)
adam = MLXAdam(learning_rate=1e-3, betas=[0.9, 0.999], eps=1e-8, bias_correction=False)
p_adam = {"w": w0 + 0}

# K-beta configured as Adam
kbeta_as_adam = Kbeta(
    learning_rate=1e-3, beta1=0.9,
    beta2_max=0.999, beta2_min=0.999,  # fixed beta2
    eps=1e-8, alpha=0.95,
    decay=None, adaptive_tiny=False, max_ratio=None,
    bias_correction="none",
    warmup_steps=0,
    layer_key_fn=lambda _: 0,
    diagnostics=False,
)
p_kb = {"w": w0 + 0}

for t in range(200):
    g_a = {"w": grad_w(p_adam["w"])}
    g_k = {"w": grad_w(p_kb["w"])}  # identical grads if params match

    # Step both
    adam.update(p_adam, g_a)    # MLX may mutate in place
    p_kb = kbeta_as_adam.apply_gradients(g_k, p_kb)

    mx.eval(p_adam["w"], p_kb["w"])
    assert mx.allclose(p_adam["w"], p_kb["w"], rtol=1e-7, atol=1e-6), f"mismatch at step {t}"
print("Per-step equivalence: OK")
\end{lstlisting}

\subsection{Kourkoutas-$\beta$ configured as Adam with bias correction on}
\begin{lstlisting}[language=Python,caption={\texttt{ablation\_Kourkoutas\_asAdam\_BC.py}}]
# adam_equivalence_check.py
import mlx.core as mx
from mlx.optimizers import Adam as MLXAdam
from kbeta.optim import KourkoutasSoftmaxFlex as Kbeta

mx.set_default_device(mx.gpu)

mx.random.seed(0)
n, d = 128, 16
X = mx.random.normal((n, d))
y = mx.random.normal((n,))
w0 = mx.random.normal((d,)) * 0.01

def loss_fn(w):
    r = X @ w - y
    return 0.5 * mx.mean(r * r)

grad_w = mx.grad(loss_fn)

# MLX Adam (BC off by default)
adam = MLXAdam(learning_rate=1e-3, betas=[0.9, 0.999], eps=1e-8, bias_correction=True)
p_adam = {"w": w0 + 0}

# K-beta configured as Adam
kbeta_as_adam = Kbeta(
    learning_rate=1e-3, beta1=0.9,
    beta2_max=0.999, beta2_min=0.999,  # fixed beta2
    eps=1e-8, alpha=0.95,
    decay=None, adaptive_tiny=False, max_ratio=None,
    bias_correction="beta2max",
    warmup_steps=0,
    layer_key_fn=lambda _: 0,
    diagnostics=False,
)
p_kb = {"w": w0 + 0}

for t in range(200):
    g_a = {"w": grad_w(p_adam["w"])}
    g_k = {"w": grad_w(p_kb["w"])}  # identical grads if params match

    # Step both
    adam.update(p_adam, g_a)    # MLX may mutate in place
    p_kb = kbeta_as_adam.apply_gradients(g_k, p_kb)

    mx.eval(p_adam["w"], p_kb["w"])
    assert mx.allclose(p_adam["w"], p_kb["w"], rtol=1e-7, atol=1e-6), f"mismatch at step {t}"
print("Per-step equivalence: OK")
\end{lstlisting}

\end{document}